\documentclass[twoside,11pt]{article}
\usepackage{lastpage}
\usepackage{jmlr2e}
\jmlrheading{23}{2022}{1-49}{9/20}{9/22}{20-1002}{Davoud Ataee Tarzanagh and George Michailidis}
\ShortHeadings{Double Core Tensor Factorization}{Tarzanagh and Michailidis}

\firstpageno{1}
\usepackage{url}
\usepackage{booktabs}       % professional-quality tables
\usepackage{amsfonts}       % blackboard math symbols
\usepackage{nicefrac}       % compact symbols for 1/2, etc.
\usepackage{microtype}      % microtypography
\usepackage{dsfont}
\usepackage{mathtools}
\usepackage{stmaryrd}
\usepackage{braket}
\usepackage{longtable}
\usepackage{lscape}
\usepackage{mathtools}
\usepackage{accents}
\usepackage{graphicx,url}
\usepackage{helvet}
\usepackage{tabularx}
\usepackage{color}
\usepackage{float}
\usepackage{lipsum}%
\usepackage[mathscr]{euscript} %% Tensor
\usepackage{enumitem}
\usepackage{graphicx}
\usepackage{longtable}
\usepackage{multirow} % span multiple rows/columns
\usepackage{wrapfig}
\usepackage{stmaryrd}
\usepackage{bbm}
\usepackage[table]{xcolor}
\usepackage{algpseudocode}
\usepackage{algorithm}
\usepackage{amsmath}
\usepackage{comment}
\usepackage{xcolor}
\usepackage{bbm}
\usepackage{stmaryrd}
\SetSymbolFont{stmry}{bold}{U}{stmry}{m}{n}
\newcommand{\minitab}[2][l]{\begin{tabular}{@{}#1}#2\end{tabular}}
\newcommand{\dom}[1]{\mathrm{dom}\,{#1}} 
\newcommand{\prox}{\mathrm{prox}} % prox map
\newcommand{\T}[1]{\boldsymbol{\mathcal{\MakeUppercase{#1}}}}

\newcommand{\bA}{\mathbf{A}}

\newcommand{\bC}{\mathbf{C}}
\newcommand{\bD}{\mathbf{D}}

\newcommand{\bG}{\mathbf{G}}
\newcommand{\bH}{\mathbf{H}}

\newcommand{\bM}{\mathbf{M}}

\newcommand{\bS}{\mathbf{S}}

\newcommand{\bU}{\mathbf{U}}
\newcommand{\bV}{\mathbf{V}}

\newcommand{\bX}{\mathbf{X}}

\newcommand{\bZ}{\mathbf{Z}}
\newcommand{\ba}{\mathbf{a}}

\newcommand{\bg}{\mathbf{g}}
\newcommand{\bh}{\mathbf{h}}

\newcommand{\br}{\mathbf{r}}

\newcommand{\bx}{\mathbf{x}}
\newcommand{\by}{\mathbf{y}}
\newcommand{\bu}{\mathbf{u}}

\newcommand{\N}{\mathbb{N}}
\newcommand{\R}{\mathbb{R}}
\newcommand{\Lc}{\mathcal{L}}
\newcommand{\rmvec}{\textrm{vec}}
\DeclareMathOperator*{\argmin}{argmin}

\newtheorem{assumption}{Assumption}

\begin{document}

\title{Regularized and Smooth Double Core Tensor Factorization for Heterogeneous Data}

\author{\name Davoud Ataee Tarzanagh$^{\star}$  \email tarzanagh@ufl.edu \\
    %   \addr
    %   %Department of Mathematics\\
    %   UF Informatics Institute \\
    %   University of Florida\\
    %   Department of Electrical Engineering and Computer Science\\
    %   Gainesville, FL 32611-8105, USA
    %   \AND
       \name George Michailidis$^{\dagger}$ \email gmichail@ufl.edu \\
       \addr
      $^{\star}$Department of Electrical Engineering and Computer Science, University of Michigan, Ann Arbor, MI 48109, USA\\
      $^{\dagger}$Department of Statistics, University of Florida, Gainesville, FL 32603, USA\\
      $^{\star,\dagger}$University of Florida Informatics Institute, Gainesville, FL 32611, USA}

\editor{Jie Peng}

\maketitle

\begin{abstract}
We introduce a general tensor model suitable for data analytic tasks for {\em heterogeneous} datasets, wherein there are joint low-rank structures within groups of observations, but also discriminative structures across different groups. To capture such complex structures, a double core tensor (DCOT) factorization model is introduced together with a family of smoothing loss functions. By leveraging the proposed smoothing function, the model accurately estimates the model factors, even in the presence of missing entries. A linearized ADMM method is employed to solve regularized versions of DCOT factorizations, that avoid large tensor operations and large memory storage requirements. Further, we establish theoretically its global convergence, together with consistency of the estimates of the model parameters. The effectiveness of the DCOT model is illustrated on several real-world examples including image completion, recommender systems, subspace clustering, and detecting modules in heterogeneous Omics multi-modal data, since it provides more insightful decompositions than conventional tensor methods.
\end{abstract} 
\vspace{.5cm}
\begin{keywords}Double core tensor factorization, heterogeneity, smoothing loss functions, regularization, ADMM 
\end{keywords}
\section{Introduction}
Tensor factorizations have received increasing attention over the last decade, due to new technical developments, as well as novel applications \citep{kolda2009tensor,cichocki2015tensor,bi2020tensors}. Popular decompositions include Tucker \citep{tucker64extension}, Canonical Polyadic (CP) \citep{carroll1970analysis}, higher-order SVD (HOSVD) \citep{de2000multilinear}, tensor train (TT) \citep{oseledets2011tensor}, and tensor SVD (t-SVD) \citep{kilmer2013third}. Nevertheless, there is limited knowledge about the properties of the Tucker and CP ranks; further,
computing such ranks has been shown to be NP-complete \citep{haastad1990tensor}. The ill-posedness of the best low-rank approximation of a tensor was investigated in \citet{de2008tensor}, while upper and lower bounds for tensor ranks have been studied in \citet{alexeev2011tensor}. In fact, determining or even bounding the rank of an arbitrary tensor is quite difficult in contrast to the matrix rank \citep{allman2013tensor}.

In many data analysis applications where tensor decompositions are extensively used, the multi-dimensional data exhibit (i) heterogeneity, (ii) missing values, and (iii) sparse representations. The literature to date has addressed the 2nd and 3rd issues, as briefly discussed next. Missing values are ubiquitous in link prediction, recommender systems, chemometrics, image and video analytics applications. To that end, tensor completion methods were developed to address this issue \citep{liu2012tensor,kressner2014low,zhao2015bayesian,zhang2014novel,song2017tensor,tarzanagh2018fast}. The standard assumption underpinning such methods is that entries are missing at random and that the data admit low-rank decompositions. However, on many occasions additional \textit{regularity} information is available, which may aid in improving the accuracy of tensor completion methods, especially in the presence of the large number of missing entries. For example, \citep{narita2012tensor} proposed two regularization methods called ``within-mode regularization" and ``cross-mode regularization", to incorporate auxiliary regularity information in the tensor completion problems. The key idea is to construct within-mode or cross-mode regularity matrices and incorporate them as smooth regularizers and then combine them with a Tucker decomposition to solve the tensor completion problem. A similar idea was also explored in \citet{bahadori2014fast,chen2014multilinear,ge2016taper}. 

Multi-way data often admit sparse representations. Due to the equivalence of the constrained Tucker model and the Kronecker representation of a tensor, the latter can be represented by separable sparse Kronecker dictionaries. A number of Kronecker-based dictionary learning methods have been proposed in literature  \citep{hawe2013separable,qi2018multi,shakeri2018minimax,bahri2018robust} and associated algorithms. Further, recent work \citep{shakeri2018minimax} shows that the sample complexity of dictionary learning for tensor data can be significantly lower than that for unstructured data, also supported by empirical evidence \citep{qi2018multi}.

However, in a number of applications, \textit{heterogeneity} is also present in the data. For example, in context-aware recommender systems that predict users' preferences, the user base exhibits heterogeneity due to different background and other characteristics. Note that such information can be a priori extracted from the available data. Similarly, image data exhibit heterogeneity due to differences in lighting and posing, which again can be extracted a priori from available metadata and utilized at analysis time. Analogous issues also are present in time varying data, wherein strong correlations can be seen across subsets of time points. A number of motivating examples are discussed in detail in Section~\ref{DCOT:examples}, and how this paper addresses heterogeneity by adding an additional core in the tensor decomposition and applying a new tensor smoothing function. Note that a standard low-rank factorization of the tensor data would not suffice, since the extracted factors would not accurately reflect the joint structure across modes. Hence, to address heterogeneity in tensor factorizations, we introduce a novel decomposition of the core tensor into  global homogeneous and local (subject--specific) heterogeneous cores. The latter encodes a priori available information on the \textit{presence} and \textit{structure} of heterogeneity on the datasets under consideration.

Hence, the \textit{key contributions} of this work are:
\begin{itemize}
\item[I.] The development of a novel \textit{supervised} tensor decomposition, coined \textit{Double Core Tensor Decomposition} (DCOT), wherein the core tensor comprises of the superposition of \textit{homogeneous} and \textit{heterogeneous} (subject specific) cores. This decomposition captures local structure present due to variations in similarities across subjects/objects in the data.
\item[II.] The DCOT model is enhanced with a new tensor smoothing loss function. Specifically, a similarity function is introduced to capture neighborhood information from the data tensor to improve the accuracy of the decompositions, as well as the convergence rate of the algorithm employed for obtaining the decomposition. We show both theoretical and computational advantages of the proposed smoothing technique in comparison to the generalized CP (GCP) models \citep{bi2018multilayer,hong2018generalized}. 
\item[III.] A new linearized ADMM method for the DCOT model is developed that can handle both non-convex constraints and objectives and its global convergence under the posited tensor smoothing loss function is established. To the best of our knowledge, despite the wide use and effectiveness of ADMM for tensor factorization tasks \citep{liu2012tensor,zhao2015bayesian,chen2013simultaneous,wang2015rubik,bahri2018robust} its global convergence does not seem to be available for tensor problems.
\end{itemize}
Finally, we illustrate the implementation of DCOT and the speed and robustness of the proposed linearized ADMM algorithm on a number of synthetic datasets, as well as analytic tasks involving large scale heterogeneous tensor data.

\subsection{Related Literature}
This work is related to a broad range of literature on tensor analysis. For example, tensor factorization approaches focus on the extraction of low-rank structures from noisy tensor observations \citep{zhang2001rank,richard2014statistical,anandkumar2014tensor}. Correspondingly, a number of methods have been proposed and analyzed under either deterministic or random Gaussian noise designs, such as maximum likelihood estimation \citep{richard2014statistical}, HOSVD \citep{de2000multilinear}, and higher-order orthogonal iteration (HOOI) \citep{de2000best}. Since non-Gaussian-valued tensor data also commonly appear in practice, \citet{chi2012tensors,hong2018generalized} considered the generalized tensor decomposition and introduced computational efficient algorithms. However, theoretical guarantees for many of these procedures and the statistical limits of the smooth tensor decomposition still remain open.

Our proposed framework includes the topic of tensor compression and dictionary learning. Various methods, such as convex regularization \citep{tomioka2013convex,raskutti2019convex}, alternating minimization \citep{zhou2013tensor,hawe2013separable,qi2018multi,han2020optimal,liu2012tensor,tarzanagh2018fast}, and  (adaptive) gradient methods \citep{han2020optimal,kolda2020stochastic,nazari2019dadam,nazari2020adaptive} were introduced and studied. In addition, tensor block models \citep{smilde2000multiway}, supervised tensor learning \citep{tao2005supervised,wu2013supervised,lock2018supervised}, multi-layer tensor factorization \citep{bi2018multilayer,tang2020individualized}, coupled matrix and tensor factorizations \citep{banerjee2007multi,yilmaz2011generalised,acar2011all} are important topics in tensor analysis and have attracted significant attention in recent years. Departing from the existing results, this paper, to the best of our knowledge, is the first to give a unified treatment for a broad range of smooth and heterogeneous tensor estimation problems with both statistical optimality and computational efficiency.

This work is also related to a substantial body of literature on structured matrix factorization, wherein the goal is to estimate a low-rank matrix based on a limited number of observations. Specific examples on this topic include group-specific matrix factorization \citep{lock2013joint,bi2017group}, local matrix factorization \citep{lee2013local}, and smooth matrix decomposition \citep{dai2019smooth}. Despite similarities of our consistency analysis to \cite{dai2019smooth}, their results cannot be directly generalized to tensor problems for many reasons. First, many basic matrix concepts or methods cannot be directly generalized to high-order ones \citep{hillar2013most}. Naive generalization of matrix concepts such as kernels, operator norm, and singular values are possible, but most often computationally NP-hard. Second, tensors have more complicated algebraic structure than matrices. As what we will illustrate later, one has to simultaneously handle all factors matrices and the core tensors with distinct dimensions in the consistency and global convergence analyses. To this end, we develop new technical tools for tensor algebra and Kronecker smoothing functions; see, e.g.,  Definition~\ref{def:kronsim}. Additional technical issues related to generalized tensor estimation and the connections of our consistency bounds with prior results in the literature are addressed in Section~\ref{sec:consist}.

The remainder of the paper is organized as follows: the DCOT formulation is presented in Section~\ref{DCOT:formulation} together with illustrative motivating examples. The linearized ADMM algorithm and its convergence properties are presented in Section~\ref{sec:optimization}. The numerical performance of the DCOT model together with applications are discussed in Section~\ref{sec:result}. Section~\ref{sec:conc} concludes the paper. Proofs and other technical results are delegated to the Appendix. 
 
\textbf{Notation.}  Any notation is defined when it is used, but for reference the reader may also find it
summarized in Table~\ref{table_notation2}. 
\section{A Double Core Tensor Factorization (DCOT)}\label{DCOT:formulation}
We start by introducing the DCOT model.  
\begin{definition}[DCOT]\label{def:dcot}
Given an $N$-way tensor $\T{Z}\in\R^{I_1\times I_2 \times \cdots \times I_N}$ with $M$ subjects (units) each containing $m_{\pi}$ subgroups for $m=1, \cdots, M$, its \textnormal{DCOT} decomposition is given by
\begin{align}\label{eq:dcot}
\nonumber
\T{Z} &= \sum\limits_{r_1 = 1}^{R_1} 
 \cdots \sum\limits_{r_N = 1}^{R_N}{ \left(g_{r_1 \cdots r_N}+h_{r_1  \cdots r_N}\right)\bu^{(1)}_{r_1} \circ \cdots \circ \bu^{(N)}_{r_N} } \\
  &= \left(\T{G}+  \T{H}\right) \times_1 \bU^{(1)}\times_2 \bU^{(2)} \cdots \times_N \bU^{(N)},
\end{align}
where $\bU^{(n)} =[\bu_1^{(n)},\bu_2^{(n)},\cdots,\bu_{R_n}^{(n)}] \in \R^{I_n\times R_n}$, is the $n$-th factor matrix consisting of latent components $\bu_r^{(n)}$; $\T{G} \in \R^{R_1 \times R_2 \times \cdots \times R_N}$ is a global core tensor reflecting the connections (or links) between the latent components and factor matrices; and $\T{H}$ is another core tensor reflecting the joint connections between the latent components in each subject. Specifically, for each subject $m \in [M]$, we have 
$$\T{H}_{m_1} =  \T{H}_{m_2}= \cdots = \T{H}_{m_\pi}.$$
\end{definition}

In Definition~\ref{def:dcot}, the $N$-tuple $(R_1, R_2, \dots , R_N)$ with $R_n = \text{rank}(\bZ_{(n)})$ is called the multi-linear rank of $\T{Z}$. For a core tensor of minimal size, $R_1$ is the column rank (the dimension of the subspace spanned by mode-1 fibers), $R_2$ is the row rank (the dimension of the subspace spanned by mode-2 fibers), and so on. An important difference from the matrix case is that the values of $R_1,R_2,\cdots,R_N$ can be different for $N\geq3$. Note that similar to the Tucker decomposition \citep{tucker64extension}, \textnormal{DCOT} factorization is said to be independent, if each of the factor matrices has full column rank; a \textnormal{DCOT} decomposition is said to be orthonormal if each of the factor matrices has orthonormal columns. We also note that decomposition \eqref{eq:dcot} can be expressed in a matrix form as:
\begin{eqnarray}\label{eq:dcotmode}
 \bZ_{(n)} & =&  \bU^{(n)} (\bG_{(n)}+\bH_{(n)}) (\bigotimes_{k \neq n} \bU^{(k)})^\top.
\end{eqnarray}

The DCOT model formulation provides a generic tensor decomposition that encompasses many other popular tensor decomposition models. Indeed, when $\T{H}=0$ and $\bU^{(n)}$ for $ n=1, 2, \cdots, N$ are orthogonal, \eqref{eq:dcot} corresponds to HOSVD. The CP decomposition can also be considered as a special case of the DCOT model with super-diagonal core tensors. 

In the DCOT model, we assume that subjects can be categorized into subgroups, where tensor components within the same subgroup share similar characteristics and are dependent on each other. For subgrouping, we can incorporate prior information. For example, in recommender system we may use users' demographic information, item categories and functionality, and contextual similarity. If this kind of information is not available, one can use the missing pattern of the tensor data, or the number of records
from each user and on each item \citep{salakhutdinov2007restricted}. In more general situations, clustering methods such as the $k$-means may be used to determine the subgroups \citep{wang2010consistent,fang2012selection}.

Next, we propose an estimation method associated with the DCOT model. Let $\T{X}$ be a data tensor that admits a DCOT decomposition and $F(s^h, \T{X};\T{Z})$ be a tensor smoothing loss function (introduced in Section~\ref{sec:smooth}) that depends on an unknown parameter $\T{Z}$ and regulated by a smoothing function $s^h$ (details discussed in Section~\ref{sec:smooth}). To estimate $\T{Z}$ from data, we propose the ``DCOT"
estimator given by 
\begin{equation}\label{eq:dcotestimation}
\hat{\T{Z}} := \big({\hat{\T{G}} + \hat{\T{H}}}\big) \times_1 \hat{\bU}^{(1)} \times_2 \hat{\bU}^{(2)}\cdots \times_N \hat{\bU}^{(N)}, 
\end{equation}
where $\hat{\T{Z}}$ is determined by solving the following penalized optimization problem: 
\begin{eqnarray}\label{eq:loss:general}
\nonumber
\min_{\T{T}} &&~ F(s^h,\T{X}; \T{Z}) + \lambda_1 J_1(\T{G}) +\lambda_2 J_2(\T{H})+  \sum_{n=1}^N \lambda_{3,n} J_{3,n}(\bU^{(n)}), \\
 \text{s.t.}&& ~ \T{Z} = \big({{\T{G}} + {\T{H}}}\big) \times_1 {\bU}^{(1)} \times_2 {\bU}^{(2)} \cdots \times_N {\bU}^{(N)}, \qquad \T{T}  \in   \Gamma(\T{T}).
\end{eqnarray}
In the formulation of the problem, $\T{T}=(\T{Z},\T{G},\T{H}, \bU^{(1)}, \bU^{(2)}, \cdots, \bU^{(N)})$ denotes the collection of optimization variables; $\{J_1(.), J_2(.),  J_{3,1}(.),   J_{3,2}(.), \cdots,  J_{3,N}(.)\}$ are penalty functions; $\{\lambda_1, \lambda_2, \lambda_{3,1},\lambda_{3,2}, \cdots, \lambda_{3,N} \}$ are penalty tuning parameters; and $\Gamma(\T{T})$ is the parameter space for $\T{T}$. 

Throughout, we impose the following set of assumptions for Problem~\eqref{eq:loss:general}.
\begin{assumption} \label{AssumptionsA} 
\rm{(i)} $J_{1} : \R^{R_{1} \times R_{2} \cdots  \times R_{N} } \rightarrow \left(-\infty , \infty\right]$, $J_{2} : \R^{R_{1} \times R_{2} \cdots  \times R_{N} } \rightarrow \left(-\infty , \infty\right]$, and
$J_{3,n} : \R^{N_{n} \times R_{n}} \rightarrow \left(-\infty , \infty\right]$ 
 are proper and lower semi-continuous such that $\inf_{\R^{R_{1} \times R_{2} \cdots  \times R_{N} }}J_{1} > -\infty$, $\inf_{\R^{R_{1} \times R_{2} \cdots  \times R_{N} }}J_{2} > -\infty$, and 
$\inf_{\R^{I_{n} \times R_{n}}} J_{3,n} > -\infty$ for $n=1,2 , \cdots, N$. 
 \\
\qquad $\rm{(ii)}$ $F(s^h,\T{X}; \T{Z}): \R^{I_{1} \times I_{2} \times \cdots \times I_{N}} \rightarrow \R$ is differentiable 
            		and $\inf_{\R^{I_{1} \times I_{2} \times, \cdots, \times I_{N}}} F > -\infty$.
\\
\qquad $\rm{(iii)}$ The gradients $\nabla F(s^h,\T{X}; \T{Z})$ is Lipschitz continuous with moduli $L_F$, i.e.,
                \begin{equation*}
\|\nabla F(s^h,\T{X}; \T{Z}^1) - \nabla F(s^h,\T{X}; \T{Z}^2)\|_F^2 \leq L_F\| \T{Z}^1 - \T{Z}^2 \|_F^2, \quad \textnormal{for all} \quad \T{Z}^{1}, \T{Z}^{2}.
  \end{equation*} 
\end{assumption}

We need to recall the fundamental proximal map which is at the heart of the DCOT algorithm. 
Given a proper and lower semicontinuous function $J: \R^{d} \rightarrow \left(-\infty , \infty\right]$, the proximal mapping associated 
	with $J$ is defined by
	\begin{equation} \label{D:ProximalMap}
        \prox_{t}^{J}\left(p\right) := \argmin \left\{ J\left(q\right) + \frac{t}{2} \|q - p\|^{2} : 
        \; q \in \R^{d} \right\}, \quad \left(t > 0\right).
    \end{equation}
The following result can be found  in \citet{rockafellar2009variational,bolte2014proximal}.
\begin{proposition} \label{P:WellProximal} Let Assumption~\ref{AssumptionsA}(i) hold. Then, for every $t \in \left(0 ,\infty\right)$, the set $\text{prox}_{t}^{J}\left(u\right)$ is nonempty and compact.
\end{proposition}

We note that $\text{prox}_{t}^{J}$ is a set-valued map. When $J$ is the indicator function of a nonempty and closed set $\Omega$, the proximal map reduces to the projection operator onto $\Omega$. It is also worth to mention that Assumptions (i)-(iii) make \eqref{eq:loss:general} have a solution and make the proposed linearized ADMM being well defined. Besides these assumptions, many practical tensor factorization functions including ones provided in Subsection~\ref{DCOT:examples} satisfy the Kurdyka--{\L}ojasiewicz property (see, Definition~\ref{eq:loss:general}) which is required to obtain a globally convergent linearized ADMM.

\begin{remark} Note that separate identification of $\T{G}$ and $\T{H}$ is not required; the DCOT estimator is designed to automatically recover the combination $\hat{\T{G}} + \hat{\T{H}}$ that leads to optimal prediction of $\T{G}+\T{H}$. Nevertheless, such identification can be beneficial from both a convergence and interpretation viewpoint. To that end, we show in Section~\ref{sec:optimization} that Assumption \ref{AssumptionsA} provides sufficient conditions to achieve identifiable cores based on a linearized multi-block ADMM approach.
\end{remark}

\subsection{Smoothing Loss Functions for DCOT Factorization}\label{sec:smooth}

In this section, we introduce a new class of loss functions for DCOT factorization that in addition to the unit/subject information reflected in the core $\T{H}$, it incorporates more nuanced information in the form of similarities between tensor fibers for each unit/subject under consideration. 
To that end, following common approaches in non-parametric statistics \citep{wand1994kernel,tibshirani1987local,fan2018local,lee2013local,dai2019smooth}, we define the smoothing function used in this work. Our proposed smoothing approach is different from these studies, since we rely on joint label information and multi-dimensional kernels.

\subsection{A Smoothing Function Based on Tensor Similarity and Tensor Labels}\label{sec:sm:loss}

The proposed loss function incorporates both tensor similarity information, as well as subgroup or label information. For example, in dictionary learning problems, in addition to using similarity information between data points, we also associate label information (0-1 label) with each dictionary item  to enforce discriminability in sparse codes during the dictionary learning process \citep{jiang2013label}. In recommender systems, the proposed loss function is constructed based on the closeness between continuous covariates in addition to a user-item specific label tensor \citep{frolov2017tensor,dai2019smooth}. In many imaging applications, additional variables of interest are available for multiway data objects. For instance, \citet{kumar2009attribute} provide several attributes for the images in the faces in the Wild database, which describe the individual (e.g., gender and race) or their expression (e.g., smiling/not smiling). It is shown in \citet{lock2018supervised} that incorporating such additional variables can improve both the accuracy and interpretation of the results.  

\begin{definition}[Kronecker Similarity]\label{def:kronsim}
Given an $N$-way data tensor $\T{X}$, assume there is additional information on each subject in the data, encoded by an $N$-way tensor $\T{Y}$. Let $s_{i_n,j_n}^h$ denote pairwise similarities between fibers $i_n$ and $j_n$ of $\T{Y}$. Each $s_{i_n,j_n}^h$ indicates how well fibers of $\T{Y}$ represent fibers of $\T{X}$, i.e., the smaller the value of $s_{i_n,j_n}^h$ is, the better $\T{Y}$ represents $\T{X}$. Under this setting, we define the Kronecker-product similarity as
\begin{eqnarray}\label{eq:hypergraph-miss}
s_{i_1 \cdots i_N,j_1\cdots j_N}^h= s_{i_1,j_1}^h c_{i_1,j_1} \cdots s_{i_N,j_N}^h c_{i_N,j_N}.
\end{eqnarray}
Here, $h > 0$ is the window size; each $s_{i_n,j_n}^h$ measures the distance between fibers $(i_n,j_n)$ for $n=1, \cdots, N$; and $c_{i_n,j_n}$ is a label consistent which is set to a value close to $1$ if fibers $i_n$ and $j_n$ share the same labels or belong to the same subgroups and close to $0$, otherwise. We note that a large value of $h$ implies that $s_{i_1 \cdots i_N,j_1\cdots j_N}^h$ has a wide range, while a small $h$ corresponds to a narrow range for $s_{i_1 \cdots i_N,j_1\cdots j_N}^h$. 
\end{definition}

When appropriate vector-space representations of fibers of $\T{Y}$ are given, we can compute similarities using a predefined function. Such functions are the encoding error -$s_{i_n,j_n}^h = K(h^{-1} \|\bA \by_{i_n} -\by_{j_n}\|_2)$ for an appropriate $\bA$-, the Euclidean distance -$s_{i_n,j_n}^h = K(h^{-1} \|\by_{i_n} - \by_{j_n} \|_2)$-, or a truncated quadratic -$s_{i_n,j_n}^h = \min \{ \xi , K(h^{-1} \|\by_{i_n} - \by_{j_n} \|_2)\}$-, where $\xi$ is some constant and $K$ denotes a kernel function \citep{tibshirani1987local}. However, we may be given or can compute similarities without having access to vector-space representations; such instances include edges in a social network graph, subjective pairwise comparisons between images, or similarities between sentences computed via a string kernel. Finally, we may learn similarities by using metric learning methods \citep{xing2003distance,davis2007information,elhamifar2015dissimilarity}.

\subsubsection{Generalized smoothing tensor loss functions}

Next, we define smoothing tensor loss functions by looking at the statistical likelihood of a model for a given data tensor. Assume that we have a parameterized probability density function or probability mass function that gives the likelihood of each entry, i.e.,
\begin{displaymath}
  x_{i_1 \cdots i_N} \sim p(x_{i_1 \cdots i_N}|\theta_{i_1 \cdots i_N}), ~~\text{where} ~~ \ell(\theta_{i_1 \cdots i_N}) = z_{i_1 \cdots i_N}.
\end{displaymath}
Here, $x_{i_1 \cdots i_N}$ is an observation of a random variable, and $\ell(\cdot)$ is an invertible \emph{link function} that connects the model parameters $z_{i_1 \cdots i_N}$ and the corresponding \emph{natural parameters} of the distribution, $\theta_{i_1 \cdots i_N}$. 

Our goal is to obtain the maximum likelihood estimate $\T{Z}$.  Let  $\Omega$ be an index set of observed tensor components. Assuming that the samples are independent and identically distributed, we can obtain $\T{Z}$ by solving
\begin{equation}
  \label{eq:joint-prob}
  \max_{\T{Z}} \;
  L(\T{X};\T{Z}) \equiv
\prod_{(i_1, \cdots, i_N) \in \Omega} p(x_{i_1 \cdots i_N}|\theta_{i_1 \cdots i_N}), ~~\text{where} ~~ \ell(\theta_{i_1 \cdots i_N}) = z_{i_1 \cdots i_N}.
\end{equation}
Working with the log-likelihood, one can easily obtain the following minimization problem
\begin{equation*}%\label{eq:non:loss:smooth}
\min_{\T{Z}} \Big\{ \widehat{F}(\T{X};\T{Z}) =-\frac{1}{\Omega}\sum_{(i_1, \cdots, i_N) \in \Omega}  \widehat{f}(x_{i_1 \cdots i_N}; z_{i_1 \cdots i_N})  \Big\},
\end{equation*}
where 
\begin{equation}\label{eq:loss:non}
 \widehat{f}(x_{i_1 \cdots i_N}; z_{i_1 \cdots i_N}) = \log \big(p(x_{i_1 \cdots i_N}|\ell^{-1}(z_{i_1 \cdots i_N})\big).
\end{equation}

In this paper, we propose a novel approach based on the idea of a tensor similarity and tensor labels to improve the prediction performance. Specifically, using the similarity function $s^h$, we consider the following cost function  
\begin{equation}\label{eq:loss:smooth}
\min_{\T{Z}} \Big\{ F(s^h,\T{X};\T{Z}) = - \frac{1}{\prod_{n \in [N]} I_n}\sum_{i_1=1}^{I_1}\cdots\sum_{i_N=1}^{I_N}  f(s^h , \T{X}; z_{i_1 \cdots i_N})  \Big\},
\end{equation}   
where 
\begin{equation}\label{eq:loss}
f(s^h , \T{X}; z_{i_1 \cdots i_N})  = \sum_{(j_1,\dots, j_N)\in \Omega}   s_{i_1\cdots i_N, j_1 \cdots j_N}^h   \widehat{f} (x_{j_1 \cdots j_N}; z_{i_1 \cdots i_N})
\end{equation}
is a smoothing probability density function. 

One key strategy of this smoothing function is to pool information across each $ x_{i_1\cdots i_N}$ through the weights $ s_{i_1\cdots i_N, j_1 \cdots j_N}^h$ to increase effective sample size and improve prediction accuracy. Next, we present various loss functions corresponding to different types of data; e.g., numerical, binary, and count.

\subsubsection{Numerical data}
We are concerned with the situation where we have the data tensor $\T{X}$ corrupted by \textit{white noise}. Specifically, we assume that
\begin{equation}\label{eq:white_noise}
x_{i_1\cdots i_N} = z_{i_1 \cdots i_N}+ \epsilon_{i_1\cdots i_N} \quad \text{with} \quad \epsilon_{i_1\cdots i_N} \sim \mathcal{N}(0,\sigma^2)
 ~~ \text{for all} \quad {(i_1, \cdots ,i_N)} \in \Omega.
\end{equation}
Here, $\mathcal{N}(\mu,\sigma^2)$ denotes the normal or Gaussian distribution with mean $\mu$ and variance $\sigma^2$ \footnote{We assume $\sigma$ is \emph{constant across all entries}.}. It follows from  \eqref{eq:white_noise} that
\begin{displaymath}
x_{i_1\cdots i_N}  \sim \mathcal{N}(\mu_{i_1\cdots i_N} , \sigma^2) ~~\text{with} ~~ \mu_{i_1\cdots i_N} = z_{i_1 \cdots i_N}  ~~ \text{for all} \quad {(i_1, \cdots ,i_N)} \in \Omega.
\end{displaymath}
In this case, the link function between $ \mu_{i_1 \cdots i_N} $ and $z_{i_1\cdots i_N}$ is the identity, i.e.,  $\ell(\mu_{i_1\cdots i_N})=\mu_{i_1\cdots i_N}$. Plugging this link function into \eqref{eq:loss:non} yields $  \widehat{f} (x_{i_1 \cdots i_N}; z_{i_1 \cdots i_N}) = (z_{i_1 \cdots i_N}- x_{i_1 \cdots i_N})^2$. Now, using \eqref{eq:loss}, we obtain
\begin{equation}\label{eq:quad:loss}
f(s^h , \T{X}; z_{i_1 \cdots i_N})  = \sum_{(j_1,\dots, j_N)\in \Omega}   s_{i_1 \cdots i_N, j_1\cdots j_N}^h (z_{i_1 \cdots i_N}- x_{j_1 \cdots j_N})^2.
\end{equation}

\subsubsection{Binary data}\label{sec:bernoulli-odds-link}
The standard assumption of a data generating mechanism for such data is the Bernoulli distribution; specifically,
a binary random variable $x \in \{0,1\}$ is Bernoulli distributed with parameter $\rho \in [0,1]$ if $\rho$ is the probability of obtaining a value of $1$ and $(1-\rho)$ is the probability for obtaining a value of $0$. The probability mass function is given by
\begin{equation}\label{eq:bernoulli-pdf}
  p( x |\rho) = \rho^{ x} (1-\rho)^{(1- x)}, \quad  x \in \{0,1\}, \qquad    x  \sim \text{Bernoulli}(\rho).
\end{equation}
A reasonable model for a binary data tensor $\T{X}$ is
\begin{equation}\label{eq:bernoulli}
 x_{i_1\cdots i_N}  \sim \text{Bernoulli}(\rho_{i_1\cdots i_N}), ~~\text{where}~~ \ell(\rho_{i_1\cdots i_N}) =z_{i_1 \cdots i_N}.
\end{equation}
A common option for the link function $\ell(\rho)$ is to work with the log-odds, i.e.,
\begin{equation}\label{eq:log-odds-link}
  \ell(\rho) = \log (\frac{\rho}{1-\rho}).
\end{equation} 
Substituting the link function \eqref{eq:log-odds-link} into \eqref{eq:loss:non}, gives 
$ \widehat{f} (x_{i_1 \cdots i_N}; z_{i_1 \cdots i_N}) = \log(1+e^{z_{i_1 \cdots i_N}}) - x_{i_1 \cdots i_N}  z_{i_1 \cdots i_N}$. Now, using  \eqref{eq:loss}, we get the following smoothing tenor function
\begin{displaymath}
f(s^h , \T{X}; z_{i_1 \cdots i_N})   = \sum_{(j_1,\dots, j_N)\in \Omega}   s_{i_1 \cdots i_N, j_1\cdots j_N}^h \big(\log(1+e^{z_{i_1 \cdots i_N}}) - x_{j_1 \cdots j_N}  z_{i_1 \cdots i_N}\big),
\end{displaymath}
where  $x_{j_1 \cdots j_N}\in \{0,1\}$ and the associated probability is $\rho = e^{z_{j_1 \cdots j_N}}/(1+e^{z_{j_1 \cdots j_N}})$. 

%
%A form of logistic tensor decomposition for a different type of decomposition called DEDICOM was proposed by Nickel and Tresp \citet{nickel2013logistic}.

\subsubsection{Count data}
\label{sec:poiss-ident-link}

For count data, it is common to model them as a Poisson distribution. The probability mass function for a Poisson distribution with mean $\lambda$ is given by
\begin{equation}\label{eq:poisson-pdf}
  p(x|\lambda) = \frac{e^{-\lambda} \lambda^x}{x!} ~~ \text{for} ~~ x \in \N.
\end{equation}
If we use the identity link function, i.e., $\ell(\lambda) = \lambda$  and substitute \eqref{eq:poisson-pdf} into \eqref{eq:loss:non}, we obtain 
$ \widehat{f} (x_{i_1 \cdots i_N}; z_{i_1 \cdots i_N})  = z_{i_1 \cdots i_N} - x_{i_1 \cdots i_N} \log z_{i_1 \cdots i_N}$. Now, it follows from \eqref{eq:loss} 
that 
\begin{equation}\label{eq:poisson}
f(s^h , \T{X}; z_{i_1 \cdots i_N})  = \sum_{(j_1,\dots, j_N)\in \Omega}  
  s_{i_1 \cdots i_N, j_1\cdots j_N}^h \big(z_{i_1 \cdots i_N} - x_{j_1 \cdots j_N} \log z_{i_1 \cdots i_N}\big),
\end{equation}
where $ x_{j_1 \cdots j_N} \in \N$ and $z_{i_1 \cdots i_N} \geq 0$.
%
%This is a smooth variant of the loss function proposed by \citet{welling2001positive,chi2012tensors} and \citet{lee1999learning} in the context of tensor and matrix decomposition, respectively. 

\subsubsection{Positive continuous data}
\label{sec:rayleigh-mean-link}

There are several distributions for handling nonnegative continuous data: Gamma, Rayleigh, and even Gaussian with nonnegativity constraints. Next, we consider the Gamma distribution which is appropriate for strictly positive data. For $x>0$, the probability density function is given by
\begin{equation}\label{eq:gamma-pdf}
p(x|t,\theta) =  \frac{ x^{t-1} e^{-x/\theta} }{ \Gamma(t)\, \theta^t } ,
\end{equation}
where the parameters $t$ and $\theta$ are positive real quantities as is the variable $x$ and $\Gamma(\cdot)$ is the Gamma function.

A common choice for the link function is $\ell(t,\theta) = \theta/t$ which induces a positivity constraint on $z_{j_1 \cdots j_N}$. Assume $t$ is constant across all entries. Plugging the functions $p$ and $\ell$ into \eqref{eq:loss:non} and removing the constant terms yields $ \widehat{f} (x_{i_1 \cdots i_N}; z_{i_1 \cdots i_N}) = \log(z_{i_1 \cdots i_N}) + x_{i_1 \cdots i_N}/z_{i_1 \cdots i_N}$.  Hence,  the smoothing loss function is defined by
\begin{equation}\label{eq:gamma-loss}
f(s^h , \T{X}; z_{i_1 \cdots i_N})  = \sum_{(j_1,\dots, j_N)\in \Omega}   s_{i_1 \cdots i_N, j_1\cdots j_N}^h \big(\log(z_{i_1 \cdots i_N}) + x_{j_1 \cdots j_N}/z_{i_1 \cdots i_N}\big),
\end{equation}
where $x_{j_1 \cdots j_N}$ and $z_{i_1 \cdots i_N}$ are both positive. In practice, we use  $z_{i_1 \cdots i_N} \geq 0$ and replace $z_{i_1 \cdots i_N}$ with $z_{i_1 \cdots i_N}+\epsilon$ (with small $\epsilon$) in the loss function \eqref{eq:gamma-loss}.

\subsection{Motivating Examples and Applications of the DCOT Model}\label{DCOT:examples}

Next, we discuss a number of motivating examples for DCOT and the associated smoothing loss function.

\subsubsection{Context-Aware Recommender Systems}\label{sec:dcot:recomm} Recommender systems predict users' preferences across a set of items based on large past usage data, while also leveraging information from similar users. In multilayer recommender systems, a tensor based analysis is beneficial due to its flexibility to accommodate contextual information from data, and is also regarded as effective in developing context-aware recommender systems (CARS) \citep{adomavicius2011context,frolov2017tensor,bi2018multilayer,bi2020tensors,zhang2020dynamic}.   
Besides user and item information available in traditional recommender systems \citep{lang1995newsweeder,verbert2012context,bi2017group}, multilinear recommender systems also use additional contextual variables, including geolocation data, time stamps, store information, etc. Although CARS are capable of utilizing such additional information and thus furnishing more accurate recommendations, they are also hampered by the so-called ``cold-start" problem, wherein not sufficient information is available on new users, items or contexts.
To address these issues, we propose a new tensor model which incorporates smoothing loss functions and can accommodate heterogeneity across observation groups. More specifically, we consider the objective function
\begin{equation}\label{dcot:recomm}
  F(s^h,\T{X};\T{Z}) +\lambda_{1} \|\T{G}\|_F^2 + \lambda_{2}\|\T{H}\|_F^2+ \lambda_{3} \|\bU\|_F^2,
\end{equation} 
where $\lambda_{i}$ for $i=1,2, 3 $ denote regularization parameters and  $\bU =\bigotimes_{n \in [N]} \bU^{(n)}$. Other regularization methods include, but are not limited to, the $\ell_0$- and $\ell_1$-penalty for sparse low-rank pursuit.

The function \eqref{dcot:recomm} enables pooling information from neighboring $(i_1,\cdots,i_N)$ points, through similarity function $s^h$. Further, it addresses satisfactorily the ``cold start" problem, by leveraging information from similar users in similar contextual settings. Finally, the issue of missing data in a non-ignorable fashion can be easily addressed through appropriately constructed neighborhoods and similarities~\eqref{eq:hypergraph-miss}. 

\subsubsection{Discriminative and Separable Dictionary Learning}\label{sec:tendict} 

We aim to leverage the supervised information (i.e. subjects) of input signals to learn a discriminative and separable dictionary. Assume the available data are organized in a $K$-order tensor $\T{X}_n \in \mathbb{R}^{I_1\times I_2 \times \dots \times I_K}$. According to the separable dictionary models \citep{hawe2013separable}, given \textit{coordinate dictionaries} $\bU^{(k)} \in \mathbb{R}^{I_k \times R_k}$, \textit{coefficient tensors} $\T{G}, \T{H}  \in \mathbb{R}^{R_1\times R_2 \times \dots \times R_K}$, and a \textit{noise tensor} $\T{E}_n$, we can express $\bx_n = \rmvec(\T{X}_n)$ as
\begin{equation}\label{eq:dictionmodel}
\bx_n= \Big( \bigotimes_{n \in [N]} \bU^{(n)} \Big) (\bg_n+\bh_n) + {\varepsilon}_n,
\end{equation}
where $\bg_n = \rmvec(\T{G}_n), ~\bh_n = \rmvec(\T{H}_n)$ and ${\varepsilon}_n = \rmvec(\T{E}_n)$.

Let
\begin{equation*}
	I = \prod_{k \in [K]} I_k \quad \text{and} \quad
		R = \prod_{k \in [K]} R_k.
	\end{equation*}		
By concatenating $N$ noisy observations $\{\bx_n\}_{n=1}^N$ that are realizations from the data generating process posited in \eqref{eq:dictionmodel} into $\bX\in \mathbb{R}^{I\times N}$, we obtain the following discriminative dictionary learning model
\begin{equation}\label{eq:dcot:dict}
 F(s^h, \bX; \bZ)+ \lambda_1 \|\bG\|_1+ \lambda_2 \|\bH\|_1 +\lambda_3 \sum_{m=1}^{M} \sum_{\kappa=1}^{\pi} \|\bH_{m_\kappa}\|_{F},
\end{equation}
where $\bZ=\bU(\bG+\bH)$; $\bU =\bigotimes_{k \in [K]} \bU^{(k)}$ is the basis matrix; $\bG $ and $\bH$ are coefficient matrices; and $\lambda_{1}$--$\lambda_{3}$
are regularization parameters.

Note that in \eqref{eq:dcot:dict}, we consider a \textit{sparse group lasso} penalty for the structured core $\T{H}$. This penalty yields solutions that are sparse at both the group and individual feature levels for all subjects $m=1,2, \cdots, M$.

\subsubsection{Image Analytics}\label{sec:dcot:image} 

On many occasions, the available image dataset contains multiple shots of the same subject, as is the case in the CMU faces database \citep{sim2002cmu}. As an illustration, using 30 subjects from the data base and extracting 11 poses under 21 lighting conditions, we end up with a tensor comprising of  6930 = 30 $\times$ 11 $\times$ 21 images, of $32\times 32$ dimension each. Since each subject remains the same under different illuminations for the same pose, and there are also a number of other subjects with the same pose, we consider the following partition of the core tensor $\T{H}$ $$\T{H}_{i} =\T{H}(i, :,:,: ), \quad  \text{for} \quad i=1,2,\ldots,30.$$

Further, if the resulting tensor is missing certain illuminations or poses for selective subjects, a completion task needs to be undertaken. Existing methods use either factorization or completion schemes to recover the missing components. However, as the number of missing entries increases, factorization schemes may overfit the model because of incorrectly predefined ranks, while completion schemes may fail to obtain easy to interpret model factors \citep{chen2013simultaneous}. To this end, we propose a model that combines a rank minimization technique \citep{chen2013simultaneous} with the DCOT model decomposition. Moreover, as the model structure is implicitly included in the DCOT model, we use the similarity function $s^h$ to borrow neighborhood information from image data over an image-subject specific network.

The proposed method leverages the two schemes previously discussed and accurately estimates the model factors and missing entries via the following objective function
\begin{equation}\label{dcot:face}
 F(s^h,\T{X};\T{Z}) +\lambda_{1} \|\T{G}\|_F^2 + \lambda_{2}\|\T{H}\|_F^2+ \sum_{n=1}^{N} \lambda_{3,n} \|\bU^{(n)}\|_*,
\end{equation} 
where $\bU^{(n)} \in \mathbb{R}^{I_n\times I_n}$, $n=1,\ldots, N$ are factor matrices; $\T{G}$ and $\T{H} \in \mathbb{R}^{I_1 \times I_2 \times \cdots \times I_N}$ are core tensors; and $\|\cdot\|_{*}$ denotes the trace norm.  As an example, formulation \eqref{dcot:face} leverages similarity between members of $I_3$ in the CMU dataset, i.e, $\T{H}_{m_\kappa} =(1, :, \kappa,: )$ for $\kappa=1,2, \cdots, \pi$,  and accross subjects $I_1$. The results are briefly depicted in Figure~\ref{fig:example:completion:CMU}.
\begin{figure}[t]
	\begin{center}
		\vspace{1em}% Space between image A and B
		\includegraphics[scale=0.37]{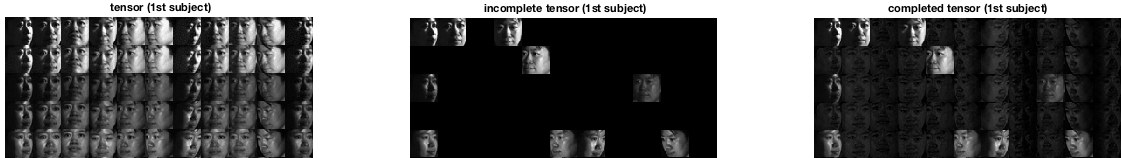}
		\vspace{.3em}
		\includegraphics[scale=0.37]{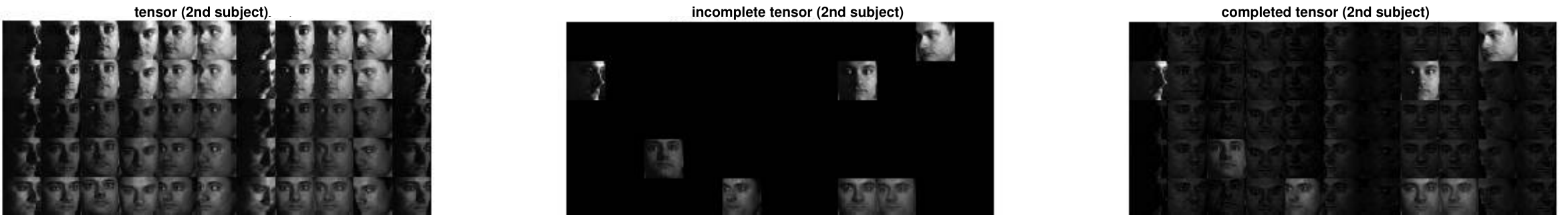}
\caption{Illustration of the tensor completion with DCOT (third iteration step) on CMU face dataset of size $30 \times 11 \times 21 \times 1024$. The last column shows that DCOT encourages joint structures within subgroups of each subject and discriminative structures across members of different subjects.}\label{fig:example:completion:CMU}			
	\end{center}
\end{figure}
%
%
%\begin{figure}[t]
%	\centering
%		\begin{tabular}{c}
%			\includegraphics[scale=0.37]{dcotexample}
%			%a)\\
%			\\
%			\includegraphics[scale=0.37]{dcotexample2png}
%			%(b)\\			
%		\end{tabular}
%\caption{Illustration of the tensor completion with DCOT (third iteration step) on CMU face dataset of size $30 \times 11 \times 21 \times 1024$. The last column shows that DCOT encourages joint structures within subgroups of each subject and discriminative structures across members of different subjects.}\label{fig:example:completion:CMU}
%\end{figure}

\subsubsection{Integrative Tensor Factorization for Omics Multi-Modal Data} 

A major challenge for integrative analysis of multi-modal Omics data is the heterogeneity present across samples, as well as across different Omics data sources, which makes it difficult to identify the coordinated signal of interest from source-specific noise or extraneous effects. Tensor factorization methods are broadly used across multiple domains to analyze genomic datasets \citep{hore2016tensor,kim2017discriminative,lee2018gift,taguchi2017identification,wang2015rubik}. In contrast to these methods, DCOT provides an approach for jointly decomposing the data matrices as slices of the data tensor. Formally, for non-negative observationally-linked datasets $\bX_1 , \dots, \bX_{I_3}$, we form a 3-way tensor $\T{X}\in \mathbb{R}^{I_1\times I_2 \times I_3}$. Then, based on a non-negative DCOT factorization, the objective function becomes
\begin{eqnarray}\label{eq:dcotomic}
F(s^h,\T{X};\T{Z}) &+& \lambda_1 \mathbbm{1}_{\T{G} \geq 0} (\T{G}) + \lambda_2 \mathbbm{1}_{\T{H} \geq 0} (\T{H})+ \sum_{n=1}^{N} \lambda_{3,n} \mathbbm{1}_{\bU^{(n)} \geq 0} (\bU^{(n)}).
\end{eqnarray}
Here, $\mathbbm{1}_{A} (.)$ is the indicator function of set $A$; $\bU^{(n)} \in \mathbb{R}^{I_n\times R_n}$ is the $n$--th nonegative factor matrix for $n=1,2$; $\T{G} \in \mathbb{R}^{R_1 \times R_2 \times R_3}$ is a core tensor reflecting the connections (or links) between the latent components and is able to capture the homogeneous part across sources; and $\T{H}$ is defined as $ \T{H}(:, :,i)=\T{H}(:, :,j)$ for all $i,j\in [R_3]$ in order to detect coordinated activity (heterogeneous part) across multiple genomic variables in the form of multi-dimensional modules. 

\section{A Linearized ADMM Method for Penalized DCOT Decomposition}\label{sec:optimization}

We develop a {\em linearized} ADMM to solve the regularized DCOT decomposition problem posited in \eqref{eq:loss:general}.  Let  $\T{T}= (\bU^{(1)}, \dots, \bU^{(N)}, \T{G}, \T{H}, \T{Z}, \T{Y})$. To obtain the updates in the standard ADMM, we first formulate  \eqref{eq:loss:general} as follows:
\begin{eqnarray}\label{main:constprob}
\nonumber
 &\underset{\T{T} }{\text{minimize}} & ~  F(s^h,\T{X};\T{Z}) + \lambda_1 J_1(\T{G})+ \lambda_2 J_2(\T{H})+ \sum_{n=1}^N \lambda_{3,n} J_{3,n}(\bU^{(n)}), \\
 &\text{s.t.}&~  \big({{\T{G}} + {\T{H}}}\big) \times_1 {\bU}^{(1)} \cdots \times_N {\bU}^{(N)} -\T{Z}=0. %\quad \T{Z}  \in   \Gamma(\T{Z}).
\end{eqnarray}
By introducing the dual variable $ \T{Y}$ and parameter $\gamma >0$, the standard ADMM is constructed for an augmented Lagrangian function defined by
\begin{eqnarray}\label{eq:aug}
\nonumber
\Lc (\T{T})
   &=& F(s^h,\T{X};\T{Z}) + \lambda_1 J_1(\T{G})+ \lambda_2 J_2(\T{H})+ \sum_{n=1}^N \lambda_{3,n} J_{3,n}(\bU^{(n)})\\
   \nonumber
 & -& \langle \T{Y},  ({{\T{G}} + {\T{H}}}) \times_1 {\bU}^{(1)} \cdots \times_N {\bU}^{(N)} -\T{Z}\rangle \\  
&+& \frac{\gamma}{2} \| ({{\T{G}} + {\T{H}}}) \times_1 {\bU}^{(1)} \cdots \times_N {\bU}^{(N)} -\T{Z}\|^2_{F}. %\qquad \T{Z}  \in   \Gamma(\T{Z}).
\end{eqnarray}
In a typical iteration of the ADMM for solving \eqref{main:constprob}, the following updates are implemented:
\begin{eqnarray} \label{eq:admm:up}
\nonumber
\bU^{(1)}_{k+1} &=& \argmin_{\bU^{(1)}} \quad  \Lc\big(\underbrace{\bU^{(1)}, \bU^{(2)}_k, \dots, \bU^{(N)}_k, \T{G}_k, \T{H}_k,  \T{Z}_k, \T{Y}_k }_{=:\T{T}^{\bU^{(1)}}_k}\big), \label{eq:admm:upu1}\\
\nonumber
\bU^{(n)}_{k+1} &=& \argmin_{\bU^{(n)}} \quad \Lc\big(\underbrace{\bU^{(1)}_{k+1}, \dots,\bU^{(n-1)}_{k+1}, \bU^{(n)}, \bU^{(n+1)}_{k}, \dots, \bU^{(N)}_k, \T{G}_k, \T{H}_k,  \T{Z}_k, \T{Y}_k)}_{=:\T{T}^{\bU^{(n)}}_k }\big),  \label{eq:admm:upun2N}\\
\nonumber
\bU^{(N)}_{k+1} &=& \argmin_{\bU^{(N)}} \quad  \Lc\big(\underbrace{\bU^{(1)}_{k+1}, \bU^{(2)}_{k+1}, \dots, \bU^{(N-1)}_{k+1},\bU^{(N)}, \T{G}_k, \T{H}_k,  \T{Z}_k, \T{Y}_k }_{=:\T{T}^{\bU^{(N)}}_k}\big), \label{eq:admm:upuN}\\
\nonumber
\T{G}_{k+1} &=& \argmin_{\T{G}}  \quad \Lc \big(\underbrace{\bU^{(1)}_{k+1}, \dots\bU^{(N)}_{k+1}, \T{G}, \T{H}_{k},  \T{Z}_k, \T{Y}_k }_{=:\T{T}^{\T{G}}_k}\big),  \label{eq:admm:upg}\\
\nonumber
\T{H}_{k+1} &=& \argmin_{\T{H}} \quad   \Lc \big(\underbrace{\bU^{(1)}_{k+1}, \dots\bU^{(N)}_{k+1}, \T{G}_{k+1}, \T{H},  \T{Z}_k, \T{Y}_k }_{=:\T{T}^{\T{H}}_k}\big),  \label{eq:admm:uph}\\
\nonumber
\T{Z}_{k+1} &=& \argmin_{\T{Z}} \quad  \Lc \big(\underbrace{\bU^{(1)}_{k+1}, \dots\bU^{(N)}_{k+1}, \T{G}_{k+1}, \T{H}_{k+1},  \T{Z}, \T{Y}_k }_{=:\T{T}^{\T{Z}}_k}\big),  \label{eq:admm:upz} \\
\T{Y}_{k+1} &=& {\T{Y}}_k - \gamma \big(({{\T{G}}_{k+1} + {\T{H}}_{k+1}}\big) \times_1 {\bU}^{(1)}_{k+1} \cdots \times_N {\bU}^{(N)}_{k+1} -\T{Z}_{k+1} \big),\label{eq:admm:upy} 
\end{eqnarray}
where $n \in \{2, 3,\cdots, N-1\}$.

Note that problem \eqref{main:constprob} is non-convex; hence, the global convergence of ADMM is a priori not guaranteed. Recent work \citep{hong2016convergence,wang2019global,lin2016iteration,tarzanagh2018estimation} studied the convergence of ADMM for non-convex and non-smooth problems under linear constraints. However, the constraints in the tensor factorization problem are nonlinear. To avoid introducing auxiliary variables and still solving \eqref{main:constprob} efficiently, we propose to approximate each sub-problem in \eqref{eq:admm:up} by linearizing the smooth terms with respect to the factor matrices and core tensors. With this linearization, the resulting approximation to \eqref{eq:admm:up} is then simple enough to have a closed form solution, and we are able to provide the global convergence under mild conditions. 

To do so, we regularize each subproblem in \eqref{eq:admm:up} and consider the following updates: 
\begin{subequations}
\begin{eqnarray} \label{eq:reg3.3}
\bU^{(n)}_{k+1} &=& \argmin_{\bU^{(n)}} \quad  \Lc(\T{T}^{\bU^{(n)}}_k)+ \frac{ \varrho^n}{2 }\| \bU^{(n)}-\bU^{(n)}_k\|_F^2,  \quad  \quad n=1, \cdots N, \label{eq:reg3.3a01}\\
\T{G}_{k+1} &=& \argmin_{\T{G}} \quad   \Lc(\T{T}^{\T{G}}_k)+ \frac{ \varrho^g}{2 }\| \T{G}-\T{G}_k\|_F^2, \label{eq:reg3.3a02}\\
\T{H}_{k+1} &=& \argmin_{\T{H}}  \quad  \Lc(\T{T}^{\T{H}}_k)+ \frac{ \varrho^h}{2 }\| \T{H}-\T{H}_k\|_F^2,  \label{eq:reg3.3a03}\\
\T{Z}_{k+1} &=& \argmin_{\T{Z}} \quad   \Lc(\T{T}^{\T{Z}}_k), %\qquad \qquad  \T{Z}  \in   \Gamma(\T{Z}),
\label{eq:finalup4}\\
\T{Y}_{k+1} &=& {\T{Y}}_k - \gamma \big((\T{G}_{k+1} + \T{H}_{k+1}\big) \times_1 {\bU}^{(1)}_{k+1} \cdots \times_N {\bU}^{(N)}_{k+1} -\T{Z}_{k+1}\big),  \label{eq:finalup5}
\end{eqnarray}
\end{subequations}
where positive constants $ \varrho^g$, $\varrho^h$, and $\{\varrho^n\}_{n=1}^{N}$ correspond to the  
regularization parameters. 

It follows from \eqref{eq:aug} that 
\begin{eqnarray}\label{eq:aug:equi}
\Lc (\T{T}) = F(s^h,\T{X};\T{Z}) + \lambda_1 J_1(\T{G})+ \lambda_2 J_2(\T{H})+ \sum_{n=1}^N \lambda_{3,n} J_{3,n}(\bU^{(n)})+ \bar{\Lc} (\T{T}),
\end{eqnarray}
where
\begin{equation}\label{eq:quadprob}
\bar{\Lc} (\T{T}) :=  \frac{\gamma}{2} \|({{\T{G}} + {\T{H}}}) \times_1 {\bU}^{(1)} \cdots \times_N {\bU}^{(N)} -\T{Z} -\frac{1}{\gamma} \T{Y} \|_F^2.
\end{equation}

Now, using \eqref{eq:aug:equi}, we approximate \eqref{eq:reg3.3a01}-\eqref{eq:reg3.3a03} by linearizing the function $\bar{\Lc} (\T{T})$ with respect to $\bU^{(1)} , \dots, \bU^{(N)},\T{G}$,  and $\T{H}$ as follows:
\begin{subequations}
\begin{align} %\label{eq:finalup}
\bU^{(n)}_{k+1} =& \argmin_{\bU^{(n)}}  \langle \nabla_{{\bU^{(n)}}} \bar{\Lc}(\T{T}^{\bU^{(n)}_k}_k), \bU^{(n)} -\bU^{(n)}_k \rangle + \lambda_{3,n} J_3(\bU^{(n)})
+ \frac{ \varrho^n}{2 }\| \bU^{(n)}-\bU^{(n)}_k\|_F^2,   \label{eq:finalup1}\\
\T{G}_{k+1}  = & \argmin_{\T{G}}    \langle \nabla_{{\T{G}}} \bar{\Lc} (\T{T}^{\T{G}_k}_k), \T{G}-\T{G}_k \rangle + 
\lambda_1 J_1(\T{G})+ \frac{ \varrho^g}{2 }\| \T{G}-\T{G}_k\|_F^2,    \label{eq:finalup2}\\
\T{H}_{k+1} =& \argmin_{\T{H}}   \langle \nabla_{\T{H}} \bar{\Lc} (\T{T}^{\T{H}_k}_k), \T{H}-\T{H}_k \rangle+ \lambda_2 J_2(\T{H}) + \frac{ \varrho^h}{2 }\| \T{H}-\T{H}_k\|_F^2.  \label{eq:finalup3}
%\\\T{Z}_{k+1} =& \argmin_{\T{Z}} \quad  \langle \nabla_{\T{Z}} \bar{\Lc}(\T{T}^{\T{Z}}_k), %\T{Z}-\T{Z}_k \rangle + \frac{ \varrho_z}{2 }\| \T{Z}-\T{Z}_k\|_F^2. \label{eq:finalup4}
\end{align}
\end{subequations}
Here, $\nabla_{{\bU^{(n)}}} \bar{\Lc}$, $\nabla_{\T{G}} \bar{\Lc}$, and $\nabla_{\T{H}} \bar{\Lc}$ denote the gradients of \eqref{eq:quadprob} w.r.t. ${\bU^{(n)}}$, $\T{G}$ and $\T{H}$, respectively. 

The following lemma gives the partial gradients of $\bar{\Lc}(\T{T})$ w.r.t. $\bU^{(n)}$,  $\T{G}$, and $\T{H}$.
\begin{lemma} \label{lem:pgrad} The partial gradients of $\bar{\Lc}(\T{T})$ are 
\begin{align*}%\label{eq-four-part-gradient}
	& \nabla_{\bU^{(n)}} \bar{\Lc}(\T{T}) =\bM_{(t)}(\bigotimes_{t \neq n} \bU^{(t)}) (\bH_{(t)}+\bG_{(t)})^\top, \qquad  n=1, \ldots, N,\\
	& \nabla_{\T{H}} \bar{\Lc}(\T{T}) = \nabla_{\T{G}} \bar{\Lc}(\T{T})=  \T{M} \times_1 \bU^{(1)} \cdots \times_N \bU^{(N)},
\end{align*}
where  $\bM_{(t)}$ denotes mode-``t" matricization of $\T{M}$, and 
\begin{equation}\label{eqn:m}
\T{M} = \gamma \Big(({{\T{G}} + {\T{H}}}) \times_1 {\bU}^{(1)} \cdots \times_N {\bU}^{(N)} -\T{Z} -\frac{1}{\gamma} \T{Y} \Big).
\end{equation}
\end{lemma}

A schematic description of the proposed ADMM is given in Algorithm~\ref{alg:lin:admm}. 

\begin{algorithm}[t]
\begin{small}
\caption{Regularized and Smooth DCOT Factorization via Linearized ADMM}\label{alg:lin:admm}
\begin{algorithmic}
   \State \textbf{Input:} $\T{X} \in \mathbb{R}^{I_1 \times I_2 \times \cdots \times I_N}$, positive constants $\lambda_1, \lambda_2, \lambda_{3,i}, i=1, \ldots, N$, factor matrices $\bU^{(n)}_0 \in \mathbb{R}^{I_n\times R_n}, n=1, \ldots,N$, dual variable $\T{Y}_0 \in \mathbb{R}^{I_1\times I_2 \times \cdots \times I_N}$, smoothing function $s^h$, and a dual step size $ \gamma > 2L_F$ where $L_F$ is a Lipschitz constant for the gradient of $F$.
\State \textbf{Initialize:}  $k=0$,  $\bU^{(n)}_k =\bU^{(n)}_0$ for $n=1, \ldots, N $, $\T{Z}_k = \T{X}$, $\T{G}_k=\T{H}_k =  \T{X} \times_1 \bU^{(1)}_k\cdots \times_N \bU^{(N)}_k$, and $\T{Y}_k= \T{Y}_0$.
\\
\textbf{For} $k=1,2, \ldots $
\begin{itemize}
\item For $n=1, \ldots, N$, update the factor matrix $\bU^{(n)}$: 
\begin{eqnarray*}
\bU^{(n)}_{k+1} &=& \prox^{J_{3,n}}_{\varrho^n} \left(\bU^{(n)}_{k} -\frac{1}{\varrho^n}  \nabla_{\bU^{(n)}} \bar{\Lc}(\T{T}^{\bU^{(n)}_k}), \frac{\lambda_{3,n}}{\varrho^n} \right).
\end{eqnarray*}  
\item Update the homogeneous core $\T{G}$: 
\begin{eqnarray*}   
\T{G}_{k+1} &=&\prox^{J_1}_{\varrho^g}\left(\T{G}_k- \frac{1}{\varrho^g}  \nabla_{\T{G}} \bar{\Lc} (\T{T}^{\T{G}_k}_k), \frac{\lambda_1}{\varrho^g}\right).
\end{eqnarray*}
\item For $m=1, \ldots, M$, update the heterogeneous core $\T{H}_{m_\pi}$: 
\begin{eqnarray*}  
\T{H}_{m_\pi, k+1} & = &\prox^{J_2}_{\varrho^h}\left(\T{H}_{m_\pi,k}- \frac{1}{\varrho^h} \nabla_{\T{H}_{m_\pi}} \bar{\Lc} (\T{T}^{\T{H}_{m_\pi,k}}_{k}), \frac{\lambda_2}{\varrho^h} \right), 
\end{eqnarray*}
and set $ \T{H}_{m_1, k+1} =  \T{H}_{m_2, k+1}= \cdots = \T{H}_{m_\pi, k+1} $. 
 \item Update the model parameter $\T{Z}$: 
\begin{eqnarray*}
\T{Z}_{k+1} &=& \argmin_{\T{Z}} \left\{F(s^h,\T{X};\T{Z}) +\bar{\Lc}(\T{T}^{\T{Z}}_k) \right\}.
\end{eqnarray*}
 \item Update the dual variable $\T{Y}$: 
 \begin{eqnarray*}
 \T{Y}_{k+1} &=& \T{Y}_{k} - \gamma\left( (\T{G}_{k+1} + \T{H}_{k+1}) \times_1 \bU^{(1)}_{k+1} \cdots \times_N \bU^{(N)}_{k+1}-\T{Z}_{k+1} \right). 
\end{eqnarray*}    
\end{itemize}
\textbf{End}
\end{algorithmic}
\end{small}
\end{algorithm}

\subsection{Global Convergence}\label{sec:globconv}

Before establishing the global convergence result of our algorithm for DCOT, we provide the necessary definitions used in the proofs. Most of the concepts that we use in this paper can be found in~\cite{rockafellar2009variational,bauschke2011convex}. 

For any proper, lower semi-continuous function $g : H \rightarrow (-\infty, \infty]$, we let $\partial_L g : H \rightarrow 2^H$ denote the \textit{limiting subdifferential} of $g$; see~\cite[Definition 8.3]{rockafellar2009variational}.

For any $\eta \in (0, \infty)$, we let $F_\eta$ denote the class of concave continuous functions $\varphi : [0, \eta) \rightarrow \mathbb{R}_+$ for which $\varphi(0) = 0$; $\varphi$ is $C^1$ on $(0, \eta)$ and continuous at $0$; and for all $s \in (0, \eta)$, we have $\varphi'(s) > 0$.

\begin{definition}[Kurdyka--{\L}ojasiewicz Property]\label{defn:kl} A function $g : H \rightarrow (- \infty, \infty]$ has the \textit{Kurdyka-{\L}ojasiewicz} (KL) property at $\overline{u} \in \dom(\partial_L g)$ provided that there exists $\eta \in (0, \infty)$, a neighborhood $U$ of $\overline{u}$, and a function $\varphi \in F_\eta$ such that 
\begin{align*}
\left(\forall u \in U \cap \{ u' \mid g(\overline{u}) < g(u') < g(\overline{u}) + \eta\}\right), \qquad \varphi'(g(u) - g(\overline{u})) \text{dist}(0, \partial_L g(u)) \geq 1.  
\end{align*}
The function $g$ is said to be a \textit{KL function} provided it has the KL property at each point $u \in \dom(g)$.
\end{definition}

In the following Theorem~\ref{thm:main}, we establish the global convergence of the standard multi-block ADMM for solving the DCOT decomposition problem, by using the KL property of the objective function in \eqref{eq:aug}. 
 
\begin{theorem}[Global Convergence]\label{thm:main}
Suppose Assumption~\ref{AssumptionsA} holds and  the augmented Lagrangian $\Lc (\T{T})$ is a KL function. Then, the sequence $\T{T}_k= (\bU^{(1)}_k, \dots, \bU^{(N)}_k, \T{G}_k, \T{H}_k, \T{Z}_k, \T{Y}_k)$ generated by Algorithm~\ref{alg:lin:admm} from any starting point converges to a stationary point of Problem~\eqref{eq:aug}.
\end{theorem}

Semi-Algebraic functions are an important class of objectives for which Algorithm~\ref{alg:lin:admm} converges:
\begin{definition}[Semi-Algebraic Functions]\label{def:semialge}
A function $\Psi : H \rightarrow (0, \infty]$ is \emph{semi-algebraic} provided that the graph $ \mathbb{G}(\Psi) = \{(x, \Psi(x)) \mid  x \in H\}$ is a semi-algebraic set, which in turn means that there exists a finite number of real polynomials $g_{ij}, h_{ij} : H \times \mathbb{R}\rightarrow \mathbb{R}$ such that 
\begin{align*}
      \mathbb{G}(\Psi) := \bigcup_{j=1}^p \bigcap_{i=1}^q \{ u \in H \mid g_{ij}(u) = 0 \text{ and } h_{ij}(u) < 0\}.
\end{align*}
\end{definition}

\begin{definition}[Sub-Analytic Functions]\label{def:analytic}
A function $\Psi : H \rightarrow (0, \infty]$ is \emph{sub-analytic} provided that the graph $ \mathbb{G}(\Psi) = \{(x, \Psi(x)) \mid  x \in H\}$ is a sub-analytic set, which in turn means that there exists a finite number of real analytic functions $g_{ij}, h_{ij} : H \times \mathbb{R}\rightarrow \mathbb{R}$ such that 
\begin{align*}
      \mathbb{G}(\Psi) := \bigcup_{j=1}^p \bigcap_{i=1}^q \{ u \in H \mid g_{ij}(u) = 0 \text{ and } h_{ij}(u) < 0\}.
\end{align*}
\end{definition}
It can be easily seen that both real analytic and semi-algebraic functions are sub-analytic. In general, the sum of two sub-analytic functions is not necessarily sub-analytic. However, it is easy to show that for two sub-analytic functions, if at least one function maps bounded sets to bounded sets, then their sum is also sub-analytic \citep{bolte2014proximal}.

The KL property has been shown to hold for a large class of functions including sub-analytic and semi-algebraic functions such as indicator functions of semi-algebraic sets, vector (semi)-norms $\|\cdot \|_p$ with $p \geq 0$ be any rational number, and matrix (semi)-norms (e.g., operator, trace, and Frobenious norm). These function classes cover most of smooth and nonconvex objective functions encountered in practical applications; see \citet{bolte2014proximal} for a comprehensive list.

\begin{remark}\label{rem22}
Each penalty function $J_i$ in \eqref{eq:aug} is a semi-algebraic function, while the loss function $F$ is sub-analytic. Hence, the augmented Lagrangian function
\begin{eqnarray*}
\Lc (\T{T})
   &=& F(s^h,\T{X};\T{Z}) + \lambda_1 J_1( \T{G})+ \lambda_2 J_2(\T{H})+ \sum_{n=1}^N \lambda_{3,n} J_{3,n}(\bU^{(n)})\\
   \nonumber
 & -& \langle \T{Y},  ({{\T{G}} + {\T{H}}}) \times_1 {\bU}^{(1)} \cdots \times_N {\bU}^{(N)} -\T{Z}\rangle \\  
&+& \frac{\gamma}{2} \| ({{\T{G}} + {\T{H}}}) \times_1 {\bU}^{(1)} \cdots \times_N {\bU}^{(N)} -\T{Z}\|^2_{F}, 
\end{eqnarray*}
which is the summation of semi-algebraic functions is itself semi-algebraic. Thus, the augmented Lagrangian function $\Lc (\T{T})$ satisfies the KL property.
\end{remark}

\section{Consistency of the DCOT Factorization}\label{sec:consist}
In this section, we derive asymptotic properties for the proposed DCOT factorization using the $\ell_2$-smoothing loss function defined in \eqref{eq:quad:loss}. In particular, we focus on the Gaussian case where $x_{i_1, \cdots, i_N}$ satisfies \eqref{eq:white_noise}. Under this setting, we provide the estimation error rate as a function of the sample size $\Omega$, the maximum rank $R_{\max}=\max\{R_1,R_2, \cdots, R_N \}$, and the tuning parameter $\lambda$ and show the necessity of the smoothing function $s^h$ for providing a faster convergence rate and a small prediction error.

Let $\widehat{\T{Z}} \in \Gamma(\T{Z})$ denote an estimator of $\T{Z}^*$. The prediction accuracy of $\hat{\T{Z}}$ is defined by the root mean square error (RMSE): 
\begin{equation}\label{K}
\rho ( \widehat{\T{Z}}, \T{Z}^*) = \left( \frac{1}{\Omega} \sum_{(i_1,\cdots, i_N )\in \Omega} (\hat{z}_{i_1 \cdots i_N} - z_{i_1 \cdots i_N}^*)^2 \right)^\frac{1}{2}.
 \end{equation}
In order to provide the asymptotic behavior of the penalized DCOT, we require the following technical assumptions:

\begin{assumption} \label{assu:smooth} Let $\{c_{i_n,j_n}\}_{n=1}^N$ be the label constraints defined in \eqref{eq:hypergraph-miss}. Then, there exist constants $a_1 \ge 0$ and $\alpha >0$, such that for any $N$-tuples $(i_1, i_2, \cdots, i_N)$ and $(j_1,j_2, \cdots,j_N)$ 
\begin{equation*}
\big|z^*_{i_1\cdots i_N}- z^*_{j_1  \cdots j_N} \big|\leq a_1 R_{\max} \max\big\{\sum_{n=1}^{N} \big(d(\by_{i_n},\by_{j_n})\big)^\alpha , \mathbbm{1}_{\prod_{n=1}^{N} c_{i_n,j_n}=0} \big\},
\end{equation*}
where $d (\by_{i_n},\by_{j_n})$ denotes the distance between $\by_{i_n}$ and $\by_{j_n}$, and $R_{\max}=\max\{R_1,R_2, \cdots, R_N \}$. 
%The corresponding expression in the maximum operator is set as 0 if $(\by_{i_n},\by_{j_n})$ or $c_{i_n,j_n}$ is absent for $n \in [N]$ .
\end{assumption}

Assumption~\ref{assu:smooth} describes the smoothness of $z^*_{i_1\cdots i_N}$ in terms of the side information $\T{Y}$. We that if $d(\by_{i_n},\by_{j_n})=0$, and $c_{i_n,j_n}=1$ for all $n \in [N]$, Assumption~\ref{assu:smooth} degenerates to $z^*_{i_1\cdots i_N}= z^*_{j_1 j_2  \cdots j_N}$. This assumption is mild when for example all fibers $\by_{i_n}$ of $\T{Y}$ are available, and is relatively more restrictive when they are absent. In the case when $N=2$, Assumption~\ref{assu:smooth} reduces to a variant of the regularity condition used in \citet{vieu1991nonparametric,wasserman2006all,stone1984asymptotically,marron1987asymptotically,dai2019smooth}. 

\begin{assumption} \label{assu:subgaus}
The tensor $\T{Y}$ has bounded support $\Psi \subset \mathbb{R}^{I_1\times I_2 \times \cdots \times I_N}$ and the error term $\epsilon_{i_1\cdots i_N}$ defined in \eqref{eq:white_noise} has a sub-Gaussian distribution with variance $\sigma^2$.
\end{assumption}
This assumption is the \textit{regularity condition} for the underlying probability distribution, and similar assumptions are widely used in literature to provide the asymptotic behavior of the matrix factorization methods~\citep{bi2017group,dai2019smooth}. 

The next result provides a general upper bound of the root mean square error $\rho(\widehat{\T{Z}},\T{Z}^*)$, which may vary by the window size $h$, the maximum rank $R_{\max}=\max\{R_1,R_2, \cdots, R_N \}$, and the number of observed variables $\Omega$.

\begin{theorem} \label{l2convergence}% \label{generalpollard}
Suppose Assumptions~\ref{assu:smooth} and \ref{assu:subgaus} hold. Let 
\begin{align*}
\phi_1 &:= \max_{(i_1,\cdots i_N)}  \sum_{(j_1,\cdots ,j_N) \in \Omega} s^h_{i_1\cdots i_N, j_1\cdots j_N} \sum_{n=1}^{N} \left(d(\by_{i_n},\by_{j_n})\right)^\alpha, \quad \textnormal{and} \\
\phi_2&:=\max_{(i_1,\cdots i_N)} \sum_{(j_1,\cdots ,j_N ) \in \Omega} \left( s^h_{i_1\cdots i_N, j_1\cdots j_N} \right)^2.
\end{align*}
Then, for some positive constant $a_2$, we have   
$$ P\left(\rho(\widehat{\T{Z}},\T{Z}^*) \ge \eta\right) \leq \exp\left(- \frac{a_2 \eta^2}{\phi_2} +  \sum_{n \in [N]} \log I_n\right),
$$
provided that
\begin{align*}
& \eta \geq \max \left\{\sqrt{R_{\max}} \phi_1, \sqrt{\phi_2}\right\} \sum_{n \in [N]} \log I_n, \quad  \textnormal{and} \\
& \lambda_1 J_1(\T{G}^*) +\lambda_2 J_2(\T{H}^*)+  \sum_{n=1}^N \lambda_{3,n} J_{3,n}(\bU^{(n)*}) \leq \eta^2.
\end{align*}
\end{theorem}

Theorem~\ref{l2convergence} is quite general in terms of the rates of  $I_1, \cdots, I_N$. If $\phi_1$ and $\phi_2$ tend to zero and can be computed for some specific smoothing parameters, the convergence rate then becomes  
$$ \rho(\widehat{\T{Z}},\T{Z}^*) \leq  \max \left\{\sqrt{R_{\max}} \phi_1, \sqrt{\phi_2}\right\} \sum_{n \in [N]} \log I_n. 
$$ 
The result of Theorem~\ref{l2convergence}, i.e., the upper bound of $ \rho(\widehat{\T{Z}},\T{Z}^*)$ may vary by the choice of parameters $\phi_1$ and $\phi_2$. Next, we provide an explicit convergence rate under some additional assumptions.  

\begin{assumption} \label{assu:kernels}
Let $s_{i_n,j_n}^h = K(h^{-1} \|\by_{i_n} - \by_{j_n} \|_2)$ and assume that the kernel function $K(.)$ satisfies 
\begin{equation}
\max \left\{\int_{0}^\infty K^2(u) du, \int_{0}^\infty K(u) u^\alpha du\right\} \leq a_3
\end{equation}
for $\alpha$ defined in Assumption~\ref{assu:smooth} and some finite $a_3>0$.
\end{assumption}
Assumption~\ref{assu:kernels} is widely used in literature for smoothing kernels~\citep{bi2017group,dai2019smooth}. Kernels with an exponential decay rate, such as the  RBF and Gaussian kernels always satisfy Assumption~\ref{assu:kernels}.
  
For any $(i_1, \cdots, i_N)$ and $(j_1, \cdots, j_N)$, let  
\begin{equation}\label{eq:uu}
U_{i_1 \cdots i_N,j_1\cdots j_N}:= \sum_{n=1}^{N}  \|\by_{i_n}-\by_{j_n}\|_2,
\end{equation}
and 
  \begin{equation*}
    \Delta_{i_1 \cdots i_n} :=
    \begin{cases*}
      1 & if $x_{i_1 \cdots i_N} \in \Omega$ \\
      0       & otherwise
    \end{cases*}
  \end{equation*}
Assume that $(y_{i_1 \cdots i_N} ,   \Delta_{i_1 \cdots i_N})$ are independent and identically distributed, but the distribution of $\Delta_{i_1 \cdots i_N}$ may depend on $y_{i_1 \cdots i_N}$.
\begin{assumption}\label{assum:prob}
For any $(i_1, \cdots, i_N)$ and $(j_1, \cdots, j_N)$, $P(c_{i_1 \cdots i_N,j_1\cdots j_N} = 1|\Delta_{j_1	 \cdots j_N}= 1)$ is bounded away from zero, and the conditional density $$f_{U_{i_1 \cdots i_N,j_1\cdots j_N}|c_{i_1 \cdots i_N,j_1\cdots j_N}=1, \Delta=1}$$ is continuous and bounded away from zero, where $c_{i_1 \cdots i_N,j_1\cdots j_N}=  \prod_{n=1}^{N} c_{i_n,j_n}$. 
\end{assumption}
Assumption~\ref{assum:prob} ensures that for any pair $(i_1,  \cdots, i_N)$, the probability of $\Delta_{i_1 \cdots i_N}$ may depend on $y_{i_1 \cdots i_N} $ and $s_{i_1 \cdots i_N,j_1\cdots j_N}^h$ and that the corresponding neighboring pairs are observed with positive probability. % Indeed, it suffices to assume that $P(c_{i_1 \cdots i_N,j_1\cdots j_N} = 1|\Delta_{i_1 \cdots i_n}= 1)$ is bounded away from zero with certain order, and similar results can be obtained with more involved derivation.

The following corollary provides an explicit value of $\phi_1$ and $\phi_2$, and the convergence rate for DCOT factorization using the smoothing loss function defined in \eqref{eq:loss}.

\begin{corollary}[Convergence Rate]\label{rate} Suppose Assumptions~\ref{assu:smooth}--\ref{assum:prob} hold. Then, we have $\phi_1= h^\alpha$, $\phi_2 = (|\Omega| h)^{-1}$, and
\begin{eqnarray}\label{rate:l2simple}
\rho ( \widehat{\T{Z}}, \T{Z}^*)  = O\left(\frac{\log \left( \prod_{n \in [N]} I_n\right)}{{|\Omega|}^\frac{\alpha}{2\alpha+1}}\right)
\end{eqnarray}
provided that $R_{\max}=O(1)$.
\end{corollary}

\begin{remark}\label{rem:rate}
Next we discuss the connections of our bound \eqref{rate:l2simple} with prior results in the literature. For $\alpha> 1/2$, since $|\Omega| \leq \prod_{n \in [N]} I_n < (\sum_{n \in [N]} I_n)^2,$ smoothing DCOT can achieve significantly better rate than 
$O \big( (\frac{\sum_{n \in [N]} I_n}{|\Omega|} \log (\frac{\sqrt{\prod_{n \in [N]} I_n}}{|\Omega|}))^{\frac{1}{2}} \big) $ established in \citet{bi2017group,bi2018multilayer} for matrix and tensors, respectively. In addition, Corollary~\ref{rate} reveals an interesting theoretical property of the smooth matrix factorization proposed by \cite{dai2019smooth} and suggests the convergence rate of the smooth tensor factorization for tensor (structured) data can be significantly better
than that for matrix (unstructured) data \citep{dai2019smooth}. Indeed, when $N=2$, it shows that the estimation error is bounded by $O \big(\frac{(I_1+I_2)^{1/(2\alpha +1)}}{|\Omega|^{1/2}} \log(I_1 I_2)\big)$ which is similar to the one provided in \citet[Corollary 1]{dai2019smooth}. This indicates a disadvantage of matricizing (unfolding) a data tensor for completion tasks such as recommender systems. More specifically, for unstructured data the bound scales linearly as $\alpha \rightarrow 0^{+}$ with the product of the factor matrices dimensions, whereas for tensor-structured data the bound scales linearly with the sum of the factor matrices dimensions. 
\end{remark}

\section{Experimental Results} \label{sec:result}

We test the performance of DCOT and its smoothing version (called S-DCOT) on a number of data analytics tasks, including subspace clustering, imaging tensor completion and denoising, recommender systems, dictionary learning, and multi-platform cancer analysis in terms of accuracy and scalability.

Algorithm \ref{alg:lin:admm} requires a good initializer to achieve good performance, which is also the case for the Tucker decomposition. To that end, we use DCOT with HOSVD \citep{de2000multilinear} and random initialization, called DCOT(H) and DCOT(R), respectively. In the first setting, given a tensor $\T{X}_0$, we construct the mode-``$n$" matricization $\T{X}_{0,(n)}$. Then, we compute the singular value decomposition $\T{X}_{0,(n)}= \bU^{(n)}_r \bD^{(n)}_r \bV^{(n)}_r$, and store the left singular vectors $\bU^{(n)}$. In both cases, the core tensor $\T{G}$ is the projection of $\T{X}$ onto the tensor basis formed by the factor matrices $\{\bU^{(n)}\}_{n=1}^N$, i.e., $\T{G} =\T{X} \times_{n=1}^N {\bU^{(n)}}^\top$. The initial heterogeneous core $\T{H}$ is set equal to $\T{G}$.

% In the implementation of S-DCOT, we use Gaussian kernels with various radii (window sizes) $h$ to define $s_{i_n,j_n}^h$. 
To select tuning parameters $\{\lambda_1, \lambda_2, \lambda_{3,1},\lambda_{3,2}, \cdots, \lambda_{3,N} \}$,    $\{R_i\}_{i=1}^N$, we search over a set of grid points aiming to minimize the RMSE defined in \eqref{K} or the average detection accuracy of clustering on the validation set.  Specifically, we used the following grids of values for the parameter search:
\begin{itemize}
    \item Regualrization paramaters  $\{\lambda_1, \lambda_2, \lambda_{3,1},\lambda_{3,2}, \cdots, \lambda_{3,N} \}$ are selected in
  \begin{subequations}
 \begin{align}\label{eqn:grid:lambda} 
 \nonumber
\lambda_1, \lambda_2 &\in \left\{\frac{1}{\|\T{X}\|_F}\cdot 10^{0.1(\nu-21)};~ \nu=1, \cdots, 41\right\},\\
\lambda_{3,1}, \ldots, \lambda_{3,N}  &\in \left\{ 10^{0.1(\nu-21)};~ \nu=1, \cdots, 41\right\},
\end{align}
\item Tensor ranks  $\{R_i\}_{i=1}^N$ ranging from
\begin{equation}\label{eqn:grid:rank}
\left\{5,10,15,20,25,30,50\right\}.
\end{equation}
% \item Number of subjects $M$  is  selected in 
% \begin{equation}\label{eqn:grid:cluster}
% \{1,5,7,10,15,20,25,30\}.
% \end{equation}
% \item Window size $h$ is selected in  
% \begin{equation}\label{eqn:grid:cluster}
% \{2^{(\nu-6)};~ \nu=1, \cdots, 9\}.
% \end{equation}
%which is the grid search suggested by the LibSVM package. 
\end{subequations}
\end{itemize}

 We note that scaling  $\{\lambda_1, \lambda_2, \lambda_{3,1},\lambda_{3,2}, \cdots, \lambda_{3,N} \}$  with tensor norms is motivated by the STDC model~\citep[3.5.2]{chen2013simultaneous}, and  provides an adaptive way to balance the impacts of the factor matrices and tensor cores for tensor factorization with smooth and gradient Lipschitz losses. 
In the implementation of S-DCOT, we use the average of $10$ multiple Gaussian kernels with $h$ selected in $\{0.75,1,1.25, \cdots, 3\}$ to define $s_{i_n,j_n}^h$. The choice of a Gaussian kernel is due to the better empirical performance obtained, compared to other possibilities.  
We set $c_{i_n, j_n}$ to $0.8$ if fibers belong to same clusters (subjects) and $0.2$ otherwise. The Kronecker-product similarity function is defined as in \eqref{eq:hypergraph-miss}. The smoothing functions are normalized such that $\sum_{j_1\cdots j_N} s^h_{i_1\cdots i_N,j_1\cdots j_N}=1$. We also set $\gamma=1/\sigma_1(\bX_{(1)}{\bX_{(1)}}^\top)$ as suggested by \citep{chen2013simultaneous}. 
% Regarding the selection of the number of subjects $M$, we note that a rather small $M$ may not be adequately ``powered" to distinguish between the proposed method and the Tucker method. In practice, if subjects (clusters) are based on categorical variables, then we can use existing categories, and hence $M$ is known. However, if clustering is based on a continuous variable, we select the number of clusters through a grid search over \eqref{eqn:grid:cluster}; see \citep{wang2010consistent} for a consistent selection of the number of clusters in more general settings. 

Regarding the selection of the number of subjects $M$, we note that a rather small $M$ may not be adequately ``powered" to distinguish between the proposed method and the Tucker method. In practice, if subjects (clusters) are based on categorical variables, then we can use existing categories, and hence $M$ is known. However, if clustering is based on a continuous variable, we can apply the quantiles of the continuous variable to determine $M$ and ``quantize" the dataset accordingly; see, \citep{wang2010consistent}.
% . We then select the number of clusters through a grid search over \eqref{eqn:grid:cluster}; see \citep{wang2010consistent} for a consistent selection of the number of clusters in more general settings. 

%Note that 
%In the implementation of S-DCOT, we use the average of $10$ multiple Gaussian kernels with various hyper-parameters to define $s_{i_n,j_n}^h$. 
% The choice of a Gaussian kernel is due to the better empirical performance obtained, compared to other possibilities.  We set $c_{i_n, j_n}$ to $0.8$ if fibers belong to same clusters (subjects) and $0.2$ otherwise. The Kronecker-product similarity function is defined as in \eqref{eq:hypergraph-miss}. The smoothing functions are normalized such that $\sum_{j_1\cdots j_N} s^h_{i_1\cdots i_N, j_1\cdots j_N} =1$.
 
\subsection{DCOT for Tensor Completion Problems}\label{sec:completion}

Next, we examine the performance of DCOT~\footnote{https://github.com/Tarzanagh/DCOT} factorization for different tensor completion tasks.

\subsubsection{Image Completion and Denoising Problems}\label{sec:imgres}
We use S-DCOT for image completion and compare it with the following tensor factorization methods for image processing: fully Bayesian CP factorization using mixture prior (FBCP-MP) \citep{zhao2015bayesian}, simultaneous tensor decomposition and completion (STDC) using factor prior \citep{chen2013simultaneous}, high accuracy low rank tensor completion (HaLRTC) \citep{liu2013tensor}, exact tensor completion using TSVD \citep{zhang2016exact}, and Low-rank Tensor Completion by Parallel Matrix Factorization (TMAC) \citep{xu2013parallel} \footnote{The codes can be obtained from \url{https://github.com/qbzhao/BCPF}, \url{https://sites.google.com/site/fallcolor/projects/stdc}, \url{http://www.cs.rochester.edu/u/jliu/}, and \url{https://xu-yangyang.github.io/software.html}, respectively.}.
 
We applied our tensor completion method proposed in Subsection~\ref{sec:dcot:image} to the 4D CMU faces database \citep{sim2002cmu} and the Cine Cardiac dataset \citep{lingala2011accelerated}. The CMU dataset \citep{sim2002cmu} comprises of 65 subjects with 11 poses, and 21 types of illumination. All face images are aligned by their eye coordinates and then cropped and resized into $32\times 32$ images. Images are vectorized, and the dataset is arranged as a fourth-order tensor. Thus, the size of the CMU data is $65 \times 11 \times 21 \times 1024$. Since each facial image is similar under different illuminations, but for similar faces, poses are not necessarily similar, we consider the following partitions
\begin{equation*}
\T{H}_{m_{\kappa}}=\T{H}(m,:,\kappa,:),~~~\kappa=1, \dots, 21,~~~ m=1, \dots, 65,
\end{equation*}
which enforces the similarity across members of $I_3$.

Dynamic cardiac imaging is performed either in cine or real-time mode. Cine MRI, the clinical gold-standard for measuring cardiac function/volumetrics~\citep{bogaert2012clinical}, produces a movie of roughly 20 cardiac phases over a single cardiac cycle (heart beat). However, by exploiting the semi-periodic nature of cardiac motion, it is actually formed over many heart beats. Cine sampling is gated to a patient's heart beat, and as each data measurement is captured it is associated with a particular cardiac phase. This process continues until enough data has been collected such that all image frames are complete. Typically, an entire 2D cine cardiac MRI series is acquired within a single breath hold (less than 30 secs). In  our experiment, we consider a bSSFP long-axis cine sequence ($n=192$, $t=19$) acquired at 1.5 T (Tesla) magnets, using an phased-array cardiac receiver coil ($l=8$ channels) \citep{candes2013unbiased}. Hence, the size of the Cine Cardiac data is $192 \times 192 \times 8 \times 19$. We use the following partitions
\begin{equation*}
\T{H}_{m_{\kappa}}=\T{H}(:,:,m,\kappa),~~~\kappa=1,\dots,19,~~~~m=1, \cdots, 8,
\end{equation*}
which enforces the similarity across time dimension.
 
\begin{table*}
\rowcolors{2}{white}{black!05!white}
\renewcommand{\arraystretch}{1.3}
\centering
\resizebox{.9\textwidth}{!}
{
\begin{tabular}{l  c c c c c c c c c c c c c}
\hline
&& \multicolumn{2}{c}{S-DCOT(H)} & \multicolumn{2}{c}{FBCP-MP} & \multicolumn{2}{c}{STDC}  & \multicolumn{2}{c}{HaLRTC} & \multicolumn{2}{c}{TSVD} & \multicolumn{2}{c}{TMAC} \\
Data & $\rho$ &  RMSE & Time  & RMSE & Time  & RMSE & Time & RMSE & Time  & RMSE & time  & RMSE & Time\\
\hline
\hline
CMU 
& $0.8$ &  {15.04e-2} &  {47.22} &   30.28e-2& 125.19    & 21.77e-2 & 79.48  & 31.89e-2 & 225.15  & 40.11e-2 & 112.48    & 29.19e-2 & 40.69  \\
& $0.85$ &  {20.17e-2}&  {42.31}&     38.49e-15& 108.23   & 28.28e-2 & 57.19  & 35.14e-2 & 189.05 & 43.12e-2 & 99.82     & 24.15e-2 & 39.18  \\
& $0.9$ &   {26.13e-2} & {29.27} &   44.15e-2& 97.37      & 43.21e-2 & 35.72  & 49.17e-2 & 155.28  & 56.71e-2 & 104.33  & 44.26e-2 & 27.04  \\
& $0.95$ &  {54.84e-2} & {19.79} &     76.41e-2& 88.19   & 81.85e-2 & 13.11 & 67.77e-2 & 114.36  & 88.79e-2 & 111.49  & 85.14e-2 & 95.6  \\
\hline
\hline
Cine  
& $0.8$ &   {18.34e-2} & { 41.39} &     42.31e-2&112.39   & 33.71e-2 & 88.39 & 41.11e-2 & 178.38  & 49.71e-2 & 110.13   & 33.12e-2 & 55.01 \\
& $0.85$ &  {23.54e-2} & { 39.29} &   48.25e-2&110.18   & 44.15e-2 & 76.32 & 55.10e-2 & 144.82  & 55.19e-2 & 114.23   & 35.11e-2 & 47.33 \\
& $0.9$ &  {29.13e-2} & { 34.45}&  30.24e-2&108.34   & 55.38e-2 &88.72 & 55.19e-2 & 107.71  & 59.13e-2 & 101.41   & 46.13e-2 & 41.55 \\
& $0.95$ & { 51.33e-2} &  {27.88} &  77.25e-2&76.85   &78.31e-2 &23.05 & 78.08e-2 & 99.12  & 76.44e-2 & 78.45   & 57.14e-2 & 39.01 \\
\hline
\end{tabular}}
\caption{RMSE and runtime (seconds) of tensor completion methods on imaging datasets.}
\label{tab:compl}
\end{table*}

Table \ref{tab:compl} shows the completion results on the imaging datasets. The running time and RMSE corresponds to regularization parameters  $\lambda_1= 
\frac{10^{1.1}}{\|\T{X}\|_F},~\lambda_2=\frac{0.1}{\|\T{X}\|_F}, ~\lambda_{3,i}=10^{0.7}~\textnormal{for all}~i=1, \ldots,N$ and ranks $R_1=R_4=10$ and $R_2=R_3=5$ which are determined by an exhaustive grid search over \eqref{eqn:grid:lambda} and \eqref{eqn:grid:lambda}.  The reason to use an exhaustive grid search is that other approaches such as a train-test procedure may not be suitable, when the size of the data set available is small and the estimated performance could be overly optimistic or overly pessimistic \citep{brownlee2020trian}.

For each given ratio $\rho$, 5 test runs were conducted and RMSE is used to evaluate the performance. Table \ref{tab:compl} and Figure~\ref{fig:cardi} show that S-DCOT significantly outperforms its competitors.
\begin{figure}
	\centering \makebox[0in]{
		\begin{tabular}{c}
			\includegraphics[scale=0.35]{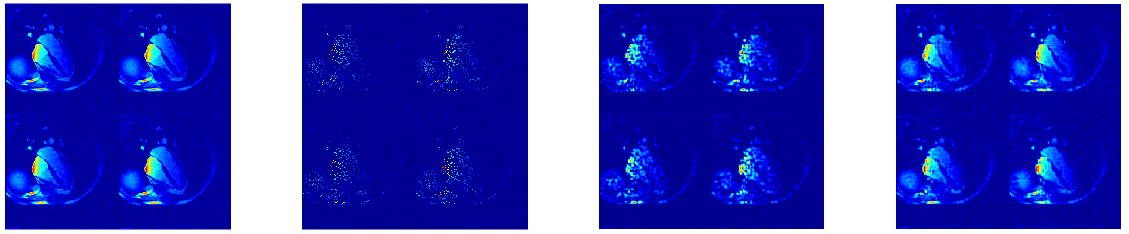}
			%a)\\
			\\
			\includegraphics[scale=0.35]{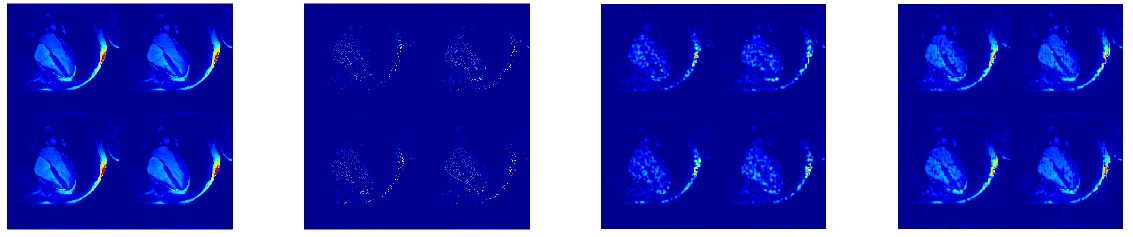}
			%(b)\\			
		\end{tabular}}
\caption{ { Tucker (third column) and DCOT (fourth column) completion results for a cine cardiac MRI series ($192 \times 192 \times 8 \times 19$) with 85\% missing rate. The effect of using the supervised core is demonstrated (four images per channel). The first and second rows show some images from the channels one and two, respectively. The last column illustrates that S-DCOT outperforms the Tucker-based completion.}}\label{fig:cardi}
\end{figure}

\subsubsection{Rainfall in India}\label{sec:rainfall-india}

We consider monthly rainfall data for different regions in India for the period 1901--2015, available from Kaggle\footnote{\url{https://www.kaggle.com/rajanand/rainfall-in-india}}. For each of 36 regions, 12 months and 115 years, we have the total rainfall in millimeters. Since the monthly rainfall is similar within the time periods Jan-Mar, Apr-Jun, Jul-Sep and Oct-Dec, we consider the following partitions
\begin{eqnarray*}
\T{H}(m,1,:) &=&\T{H}(m,2,:), \\
\T{H}(m,3,:) &=&\T{H}(m,4,:)=\T{H}(m,5,:),\\
\T{H}(m,6,:) &=&\T{H}(m,7,:)=\T{H}(m,8,:), \\
\T{H}(m,9,:) &=&\T{H}(m,10,:)=\T{H}(m,11,:)=\T{H}(m,12,:),  \quad m=1,2, \cdots, 36,
\end{eqnarray*}
which enforces similarity across members of $I_2=12$. 

\begin{figure}[t]
	\centering \makebox[0in]{
		\begin{tabular}{cc}
			\includegraphics[scale=0.12]{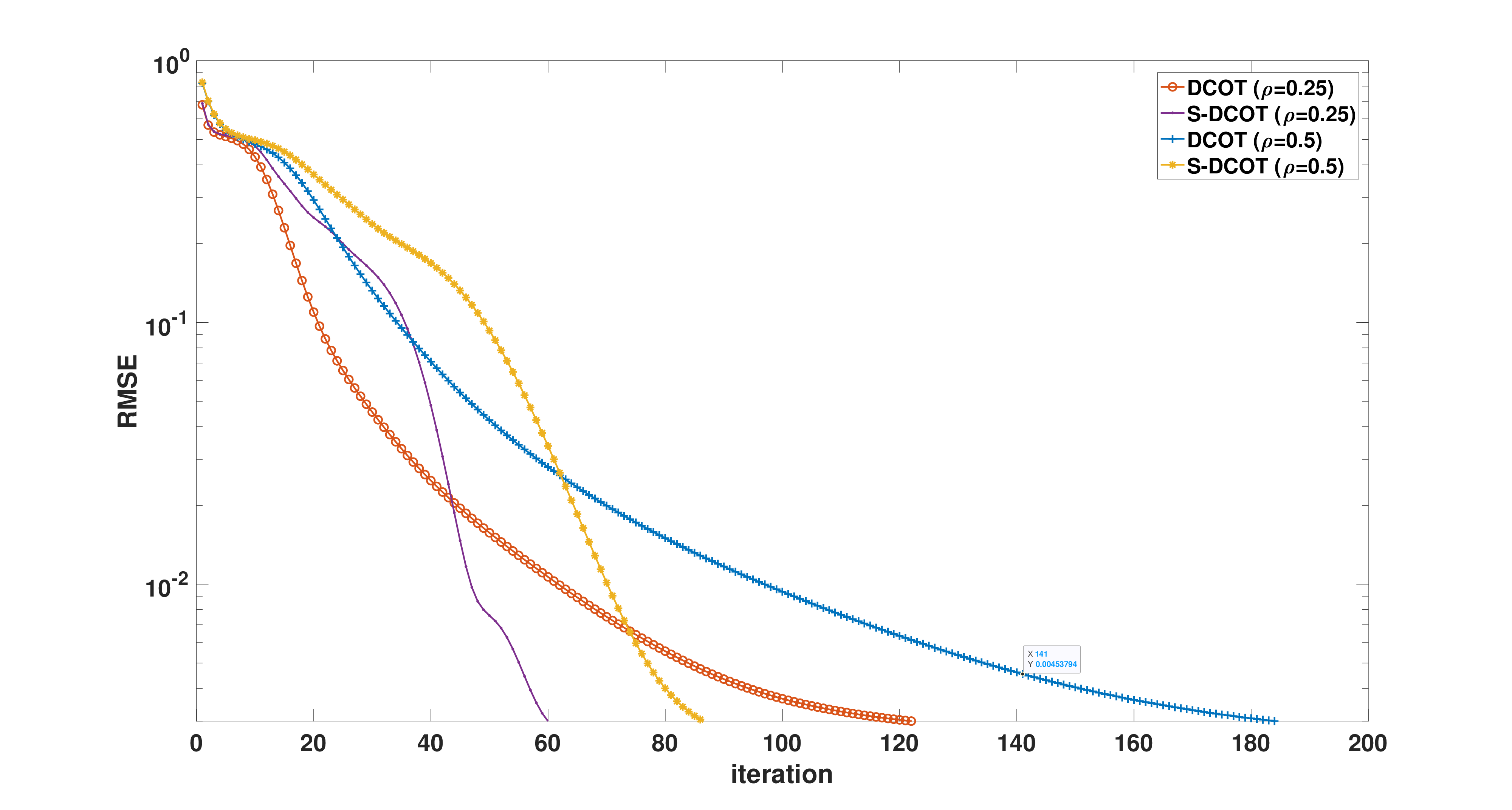}
			\includegraphics[scale=0.12]{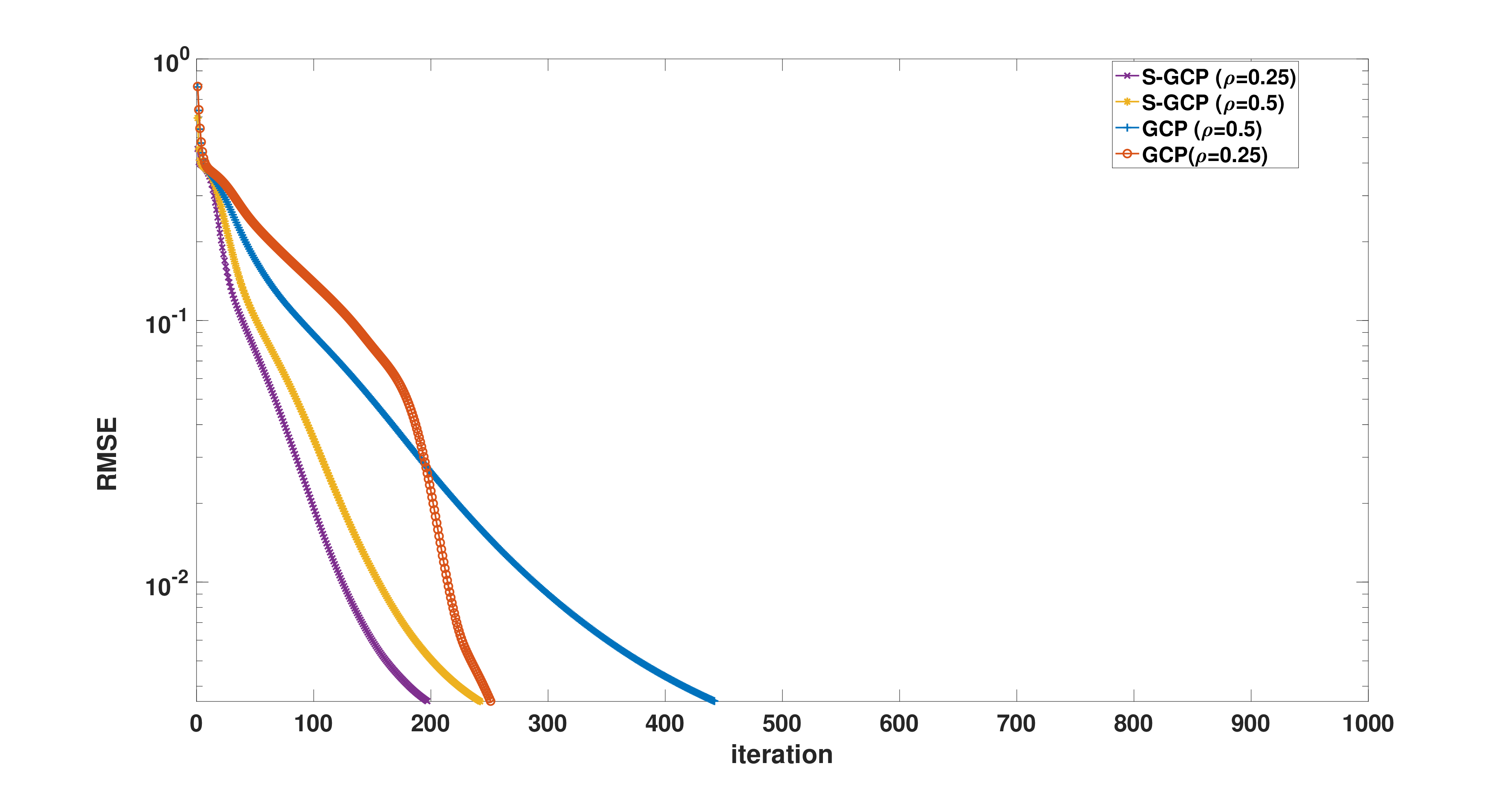}\\
		\end{tabular}}
\caption{RMSE and the number of iterations of generalized tensor methods on rainfall dataset. \textbf{Left}: DCOT vs S-DCOT \textbf{Right}: GCP vs S-GCP.}\label{fig:rain}
\end{figure}

There are several distributions for Rainfall data. As mentioned previously, one option is to assume a Gaussian distribution but impose a nonnegativity constraint. Recently, \citet{hong2018generalized} showed that the Gamma distribution is potentially a reasonable model for this dataset.  Hence, we investigate the performance of DCOT with smoothing loss \eqref{eq:gamma-loss} applied to Rainfall dataset and compare it with the  GCP \citep{hong2018generalized}. %We split the data into a 50\% training set, a 25\% validation set and a 25\% testing set.
 For all completion algorithms,  the regularization parameters and tensor ranks are determined  by a grid search over \eqref{eqn:grid:lambda} and \eqref{eqn:grid:rank} aiming to minimize the RMSE.  We run each method with 5 different random starting points and report the average RMSE. Figure~\ref{fig:rain} indicates that the proposed S-DCOT has the best performance in terms of both RMSE and number of iterations. 
%\begin{table*}
%\rowcolors{2}{white}{black!05!white}
%\renewcommand{\arraystretch}{1.3}
%\caption{The root mean square error (RMSE) and runtime in seconds of generalized tensor methods on rainfall dataset.}
%\centering
%\resizebox{.9\textwidth}{!}
%{
%\begin{tabular}{l c  c c c c c c c c c c c c}
%\hline
%&&\multicolumn{2}{c}{\textsc{S-DCOT(H)}} & \multicolumn{2}{c}{DCOT} & \multicolumn{2}{c}{S-GCP}  & \multicolumn{2}{c}{GCP}\\
%Data & $\rho$ &  RMSE & time(s)  & RMSE & time(s)  & RMSE & time (s)&   RMSE & time(s)\\
%\hline
%\hline
%Rainfall in India  
%& 0.25 & 21.45e-2 & 98.36 &  29.36e-2& 75.14   & 22.38e-2 & 26.33  & 84.73e-2 & 25.17  \\
%& 0.5 & 57.01e-2&  145.41 &  79.19e-2& 178.14   & 82.98e-2 & 57.19  & 83.26e-2 & 81.05  \\
%\hline
%\end{tabular}}
%\label{tab:rainfall}
%\end{table*}  

\subsubsection{DCOT Applied to Sparse Count Crime Data}\label{sec:count crime data}
Next, we examine the performance of smooth DCOT factorization for completion and factorization of count datasets. To do so, we consider a real-world crime statistics dataset containing more than 15 years of crime data from the city of Chicago. The data\footnote{\url{www.cityofchicago.org}} is organized as a 4-way tensor and obtained from FROSTT \footnote{\url{http://frostt.io/}}. The tensor modes correspond to 6,186 days from 2001 to 2017, 24 hours per day, 77 communities, and 32 types of crimes. Each $\T{X}(i_1, i_2, i_3, :)$ is the number of times that a crime occurred in neighborhood $i_3$ during hour $i_2$ on day $i_1$.  To enforce similarity within each community, we consider the following partitions
\begin{equation*}
\T{H}_{m_{\kappa}}=\T{H}(:,m,\kappa,:),~~~\kappa=1,\dots,R_3,~~~~m=1,\dots,R_1.
\end{equation*}
We use the DCOT model with the nonnegativity constraints and the proposed smoothing Poisson loss function defined in \eqref{eq:poisson}. %We split the data into a 50\% training set, a 25\% validation set and a 25\% testing set. 
We compare DCOT with GCP \citep{hong2018generalized}. For both  DCOT and GCP and their smoothing variants, regularization parameters and ranks are determined by a grid search over \eqref{eqn:grid:lambda} and \eqref{eqn:grid:rank}. We run each method with 5 different random starting points and report the average RMSE. 

The results are provided in Figure~\ref{fig:crime}. The GCP and S-GCP methods descend much more quickly, but do not reduce the loss quite as much, though this failure to achieve the same final minimum is likely an artifact of the function estimation. On the other hand, Figure~\ref{fig:crime} indicates that the proposed S-DCOT has the best performance in terms of RMSE. The RMSE of the proposed method is less than that of S-GCP, illustrating that S-DCOT has better performance among the competing tensor factorization methods.
\begin{figure}[t]
	\centering \makebox[0in]{
		\begin{tabular}{cc}
		\includegraphics[scale=0.12]{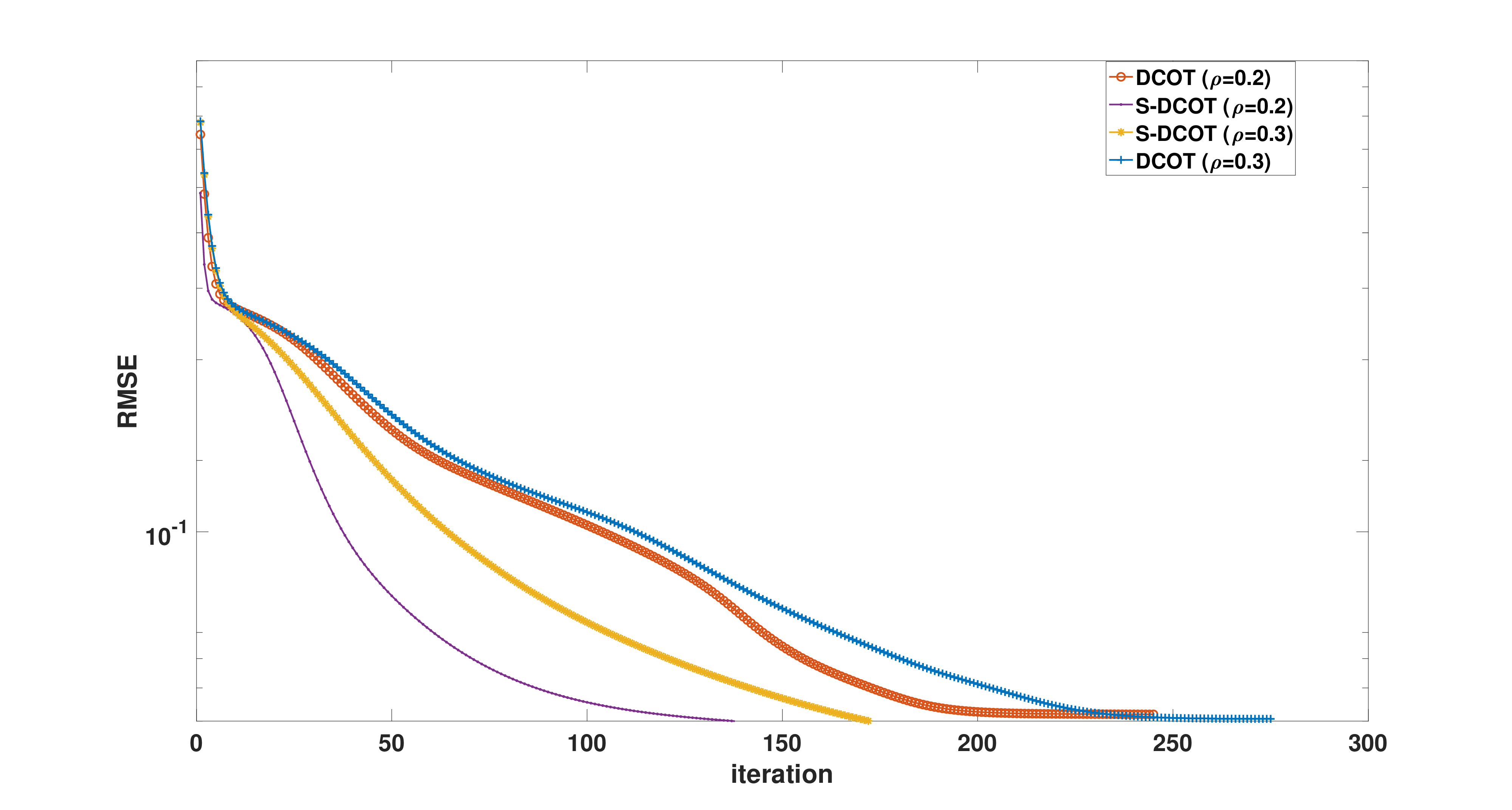}
		\includegraphics[scale=0.12]{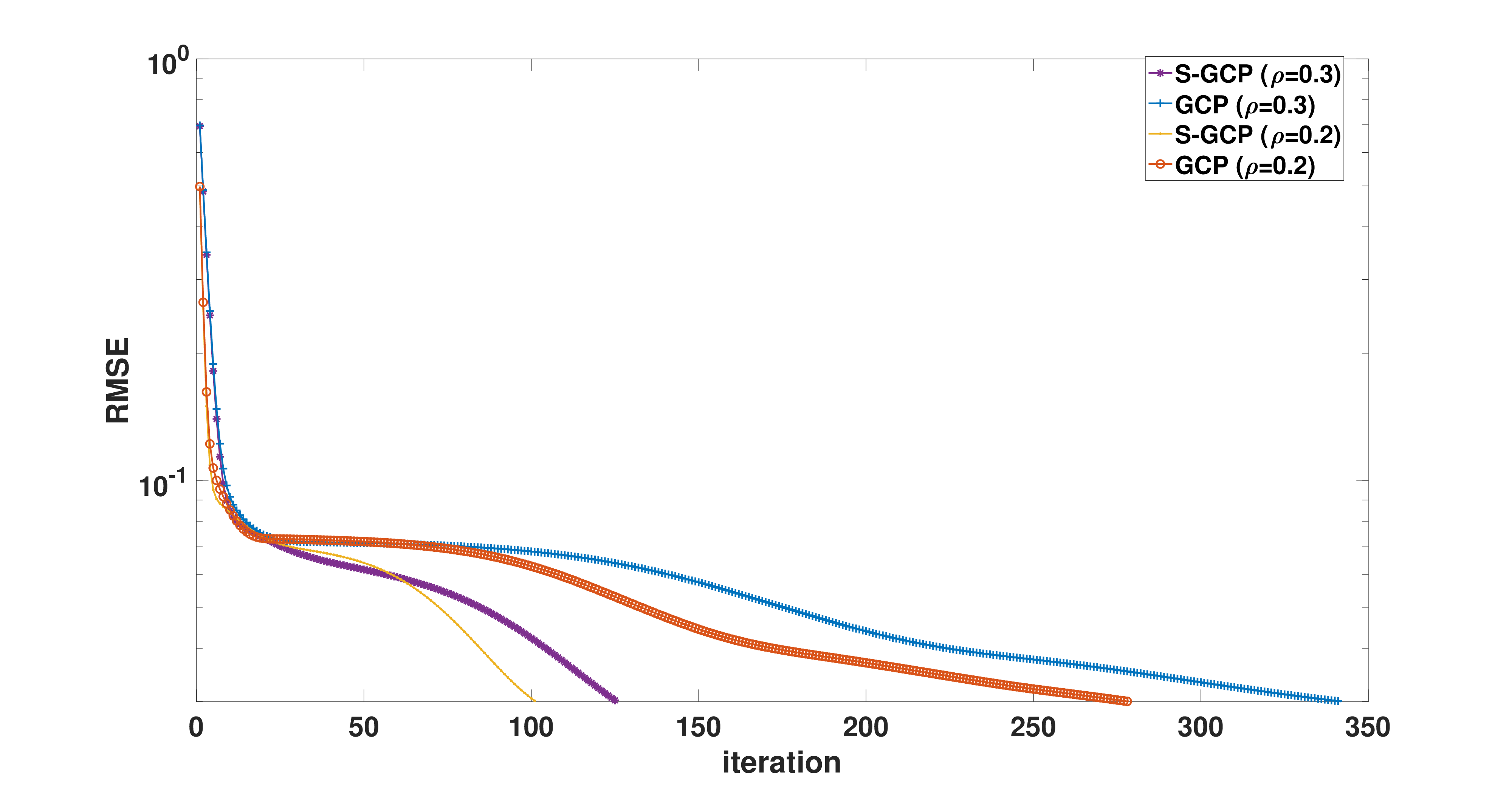}\\
		\end{tabular}}
\caption{ RMSE and the number of iterations of generalized tensor methods on count crime dataset. \textbf{Left}: DCOT vs S-DCOT \textbf{Right}: GCP vs S-GCP.}\label{fig:crime}
\end{figure}

\subsection{DCOT Applied to Multi-Platform Genomic Data}\label{sec:genom} 

DCOT and S-DCOT models have been applied to understand latent relationships between patients and genes for multi-platform genomic data.

We use the PanCan12 dataset \citep{weinstein2013cancer} and the Hallmark gene sets collections from MSigDB \citep{liberzon2015molecular} for obtaining the input tensor and label functions $c_{i_n, j_n}$, respectively. The PanCan12 contains multi-platform data with mapped clinical information of patient groups into cohorts of twelve cancer types including glioblastoma multiform, lymphoblastic acute myeloid leukemia, head and neck squamous carcinoma, lung adenocarcinoma, lung squamous carcinoma, breast carcinoma, kidney renal clear cell carcinoma, ovarian carcinoma, bladder carcinoma, colon adenocarcinoma, uterine cervical and endometrial carcinoma, and rectal adenocarcinoma. They are selected based on data maturity, adequate sample size, and publication or submission for publication of the primary analyses. The five Omics platforms used are miRNA-seq, methylation, somatic mutation, gene expression, and copy number variation.

The PanCan12 dataset was downloaded from the Sage Bionetworks repository by Synapse \citep{omberg2013enabling} and was transformed to a 3rd-order tensor ($4555 \times 15351 \times 5$), containing 4555 samples, 14351 genes, and 5 Omics platforms. The data for each platform was min-max normalized and was further normalized such that the Frobenius norm became one. In order to efficiently fuse the date into the interpretable latent factors, we consider the following partitions
\begin{equation*}
\T{H}(:,:,1)= \T{H}(:,:,1)=\dots= \T{H}(:,:,5),
\end{equation*}
which enforces the similarity across third dimension, i.e., platform.

\begin{table*}
\rowcolors{2}{white}{black!05!white}
\renewcommand{\arraystretch}{1.3}
\centering
\resizebox{.9\textwidth}{!}
{
\begin{tabular}{l  c  c c c c c c c c c c c c}
\hline
Dataset & & \multicolumn{2}{c}{\textsc{S-DCOT(R)}} & \multicolumn{2}{c}{DCOT(R)} & \multicolumn{2}{c}{S-Tucker(R)}  & \multicolumn{2}{c}{Tucker(R)}\\
% & $[R_1,R_2,R_3]$ &  RMSE   &  RMSE   & RMSE & time   &  RMSE & time  \\
\hline
%\hline
PanCan12  
&   & 23.79e-2  (0.009) &&   25.04e-2(0.008)&  &  75.42e-2(0.007) &  & 89.2e-4(0.005) & \\
%\\
%& $[40,60,5]$ &  {11.43e-2 (0.008)}&  &   {15.46e-2(0.008)}& & { 33.80e-2(0.007)} & & { 48.16e-2} & \\
\hline
\end{tabular}}
\caption{RMSE of tensor methods applied to the PanCan12 dataset.} \label{tab:pancan}
\end{table*}  

For the gene subgroups, we chose the Hallmark gene sets collection from MSigDB \citep{liberzon2015molecular} and set $c_{i_n,j_n} =1 $ if genes belong to same subgroups. A gene smoothing function $s_{i_n,j_n}^h$ is generated in the form of gene-gene interaction within each subgroup. Test RMSE is used to measure the accuracy of tensor methods on this dataset. We split the data into a 50\% training set, a 25\%
validation set and a 25\% testing set, randomly.  The regularization parameters and tensor ranks are determined  by a grid search over \eqref{eqn:grid:lambda} and \eqref{eqn:grid:rank} aiming to minimize the RMSE on the validation set.
Table~\ref{tab:pancan} indicates that the proposed DCOT has the best performance in terms of RMSE. The test RMSE of Tucker is $5\times$ higher than that of DCOT. Furthermore, test RMSE of S-DCOT is slightly higher or even better than that of Tucker and S-Tucker. 

\subsection{DCOT Applied to Recommender Systems}\label{sec:dcotresult:recom}

Next, we consider S-DCOT for recommender systems and compare it with five competing factorization methods. Three methods correspond to existing ones, namely, Bayesian probabilistic tensor factorization (BPTF) \citep{xiong2010temporal}, the factorization machine (libFM) \citep{rendle2012factorization}, and the Gaussian process factorization machine (GPFM)~\citep{nguyen2014gaussian}.\footnote{The codes can be obtained from \url{https://www.cs.cmu.edu/~lxiong/bptf/bptf.html}, \url{http://www.libfm.org/}, and \url{http://trungngv.github.io/gpfm/}, respectively.}
In addition, we also investigate the performance of the structured matrix factorization (MF) \citep{bi2017group}, and  smooth neighborhood matrix factorization (S-MF)\citep{dai2019smooth} with the proposed linearized ADMM.

We apply the proposed method to MovieLens 1M data collected by GroupLens Research\footnote{\url{http://grouplens.org/datasets/movielens}}. This dataset contains 1,000,209 ratings of 3883 movies by 6040 users, and rating scores range from 1 to 5. Also, the MovieLens 1M dataset provides demographic information for the users (age, gender, occupation, zipcode), genres, and release dates of the movies. 

We define day as a context for the DCOT recommender model detailed in Subsection~\ref{sec:dcot:recomm}. Having the length of the context determined, we need to create \textit{time bands} for days. Time bands specify the time resolution of a day, which are also data dependent. We can create time bands with equal or different length. For this dataset, we used time bands of 1 hours. Events are assigned to time bands according to their time stamp. Thus, we can create the \texttt{[user, item, day, time bands]} tensor. We factorize this tensor using the DCOT model and we get feature vectors for each user, for each item, for each day, and for each time bands. 

Since we expect that at the same time offset in different days, the aggregated behavior of the users will be similar, we consider the following partitions
\begin{equation*}
\T{H}_{m_{\kappa}}=\T{H}(:,:,m,\kappa),~~~\kappa=1,\dots,R_4,~~~m=1,\dots, R_3,
\end{equation*}
which enforces the similarity across members of $R_4$. For this application, in addition to the heterogeneous core $\T{H}$, we focus on employing a user-item smoothing function to solve the cold-start issue. We classify users based on the quantiles of the number of their ratings and set $c_{i_n, j_n} =1$ if users belong to same clusters. On the other hand, the items are classified based on their release dates and $c_{i_n, j_n} =1$ if they belong to same clusters. 

\begin{table*}
\rowcolors{2}{white}{black!05!white}
\renewcommand{\arraystretch}{1.3}
\centering
\resizebox{.9\textwidth}{!}
{
\begin{tabular}{l  c  c c c c c c c c c c c c}
\hline
Dataset & & \multicolumn{2}{c}{\textsc{S-DCOT(R)}} & \multicolumn{2}{c}{libFM} & \multicolumn{2}{c}{GPFM}  & \multicolumn{2}{c}{BPTF} & \multicolumn{2}{c}{S-MF} & \multicolumn{2}{c}{MF} \\
%Dataset & features &  RMSE & time(hrs)  & RMSE & time(hrs)  & RMSE & time  (hrs) & RMSE& time(hrs)  & RMSE & time(hrs) & RMSE & time(hrs) \\
%\hline
MovieLens 1M 
&  &   {0.970(0.004)} &  & { 0.989(0.006)}& & {1.071(0.005)} &   & {1.027(0.007)} &  &{0.982(0.006)} &  &  {1.056 (0.007)}& \\
%& $40$ &   {0.955(0.006)} &  &  0.981(0.006)& & 0.979(0.004) &  & 0.986(0.009) & &0.957(0.004) & &  1.014 (0.005)&  \\
\hline
\end{tabular}}
\caption{ RMSE of completion methods on recommender systems dataset.}
\label{tab:recom}
\end{table*}  

We split the data into a 60\% training set, a 15\%
validation set and a 25\% testing set, randomly. The regularization parameters and tensor ranks are determined  by a grid search over \eqref{eqn:grid:lambda} and \eqref{eqn:grid:rank} aiming to minimize the RMSE on the validation set. Table~\ref{tab:recom} indicates that the proposed S-DCOT has the best performance in terms of RMSE. The RMSE of the proposed method is less than that of BPTF and libFM, illustrating that S-DCOT has better performance among the competing tensor factorization methods.

\begin{table*}
\rowcolors{2}{white}{black!05!white}
\renewcommand{\arraystretch}{1.3}
\centering
\resizebox{.9\textwidth}{!}
{
\begin{tabular}{l  c  c c c c c c c c c c c}
\hline 
&&&&&& Misspecification rate&&&&&&\\
Dataset && \multicolumn{2}{c}{\textsc{5 \%}} &  \multicolumn{2}{c}{\textsc{10 \%}} & \multicolumn{2}{c}{15 \%} & \multicolumn{2}{c}{20 \%}  & \multicolumn{2}{c}{30 \%} \\
%Dataset & features &  RMSE & time(hrs)  & RMSE & time(hrs)  & RMSE & time  (hrs) & RMSE& time(hrs)  & RMSE & time(hrs) & RMSE & time(hrs) \\
%\hline
MovieLens 1M 
& & { 0.976(0.003)} &  &   {0.981(0.005)}& & {0.989(0.004)} &   & {0.983(0.005)} &  & {1.137 (0.005)}& & \\
%& $40$ &  0.951(0.004) &  &  0.955(0.006)& & 0.953(0.004) &    & 0.961(0.005)   &  & 1.103 (0.005)&&    \\
\hline
\end{tabular}}
\caption{ {RMSE of S-DCOT method on recommender systems dataset under $5 \%$, $10 \%$, $15 \%$, $20 \%$,  and $30\%$ subject (cluster) misspecification rate.}}
\label{tab:recom:missclas}
\end{table*}  

Next, we test the robustness of the proposed
method when the clusters are misspecified. Specifically, we misassign users and items to adjacent clusters with $5 \%$, $10 \%$, $15 \%$, $20 \%$,  and $30\%$ chance and then construct the smoothing loss function and DCOT factorization. The results are summarized in Table~\ref{tab:recom:missclas} which shows that S-DCOT is robust against the misspecification of clusters. Indeed, in comparison with Table~\ref{tab:recom}, S-DCOT method performs better than the other methods except when 30\% of the cluster members are misclassified.

\subsection{DCOT Applied to Subspace Clustering and Dictionary Learning}\label{sec:clus}
 
Previous studies show that HOSVD is very powerful for clustering, especially in multiway data clustering tasks \citep{lu2011survey}. It can achieve similar or better performance than most of the state-of-the-art clustering algorithms for mutliway data. Next, we evaluate the DCOT decomposition on a clustering problem. We compare DCOT with HOSVD and also four classical dimensionality reduction methods, including Principle Component Analysis (PCA) \citep{turk1991face}, Linear Discriminant Analysis (LDA) \citep{belhumeur1997eigenfaces}, Locality Preserving Projections (LPP) \citep{he2005face}, and Marginal Fisher Analysis (MFA) \citep{yan2007graph}.

The DCOT and competing methods are evaluated on the CMU and CASIA databases. The CASIA gait~B database \citep{yu2006framework} comprises of indoor walking sequences from 124 subjects with 11 camera views and 10 clothing styles. We represent each walking sequence by the Gait Energy Image \citep{man2006individual}, which is resized into size $128 \times 88$. All the images are vectorized, and the dataset is arranged as a fourth-order tensor, three for the latent factors and one for the feature dimension. Thus, the size of dataset is ($124 \times 11 \times 10 \times 11264$ ). 

To leverage the supervised information (i.e. subjects), we consider the following partitions
\begin{eqnarray*}
%&&P= \{I_3,I_2\}\mid \{I_1\} , \qquad \text{AR database}\\
\T{H}_{m_{\kappa}} &=& \T{H}(m,:,\kappa,:),~~\kappa=1,\dots,21,~~m=1,\dots,65,~~\text{CMU database},\\
 \T{H}_{m_{\kappa}} &=& \T{H}(:,m,\kappa,:),~~\kappa=1,\dots,11,~~m=1,\cdots,10,~~\text{CASIA database}.
\end{eqnarray*}

We randomly select $\varpi$ subjects, where $\varpi= 10, 20,30$, with 5 selected poses or illuminations in the CMU-PIE dataset and with 4 selected views or clothing styles in the CASIA dataset, respectively. The remaining samples in each database are used for testing. We employ a nearest neighbor classifier and repeat the procedure 5 times and average the results. To avoid the singularity of this problem, we use the first $P$ principal coefficients determined by 95\% energy for all the methods. Note that the MGE, LPP and MFA are manifold-based methods and need to determine the $k$ nearest neighbors in their graphs.

The regularization parameters and tensor ranks are determined  by a grid search over \eqref{eqn:grid:lambda} and \eqref{eqn:grid:rank}. The overall performance is given in Table~\ref{tab:cluster}. We consider the following cases: ``untrained pose (UP)", ``untrained illumination (UI)" that refers to a subset of testing data whose corresponding factors (pose or illumination) are not available during training.
 
 DCOT achieves a high detection rate on the samples even when other complex factors are unobserved. DCOT significantly improves over multilinear-based methods and better interprets the cross-factor variation hidden in multi-factor data, even when the factor variation is not given in the training stage. We believe the benefits mainly
come from the discriminative core tensor and the similarity function which exploit all of the latent factors to embed factor-dependent data pairs in a unified way. 

\begin{table*}
%\rowcolors{1}{}{lightgray}
\rowcolors{2}{white}{black!05!white}
\renewcommand{\arraystretch}{1.3}
\centering
\resizebox{.85\textwidth}{!}
{
\begin{tabular}{l | c | c c c c c c c c c c c ccc}
\hline
Data&& \multicolumn{2}{c}{S-DCOT(H)} & \multicolumn{2}{c}{DCOT(H)} & \multicolumn{2}{c}{HOSVD}  & \multicolumn{2}{c}{PCA} & \multicolumn{2}{c}{LPP} & \multicolumn{2}{c}{MFA}& \multicolumn{2}{c}{LDA} \\
%Data &  &  &  &  & & & & && & & &\\
\hline
\hline
CMU 
& $\text{UP}$ &  {36.79}&  &  {38.44}&  & 29.10 &  & 36.16 &   & 32.81&   &34.24& &33.19\\
& $\text{UI}$ &  {96.41}&  &  {93.57}&  & 90.39 &  & 75.39 &   & 83.90&   &92.14& &89.15\\
\hline
\hline
CASIA  
& $\text{UP}$ &  {70.24}&  &  {69.38}&  & 66.29 &  & 55.34 &   & 58.18&    &50.29& &62.17\\
& $\text{UI}$ &  {89.73}&  & {85.33}&  & 83.19 &  & 80.33 &   &84.34&    &82.29&  &83.21\\
\hline
\hline
%AR  
%& $\text{UP}$ &  66.13&  &  67.14&  & 66.29 &  & 55.34 &   & 58.11&    &50.29& &62.13\\
%& $\text{UI}$ &  88.31&  & 89.11&  & 85.19 &  & 80.33 &   &84.34&    &82.29&  &83.11\\
%& $\text{BU}$ &  49.50&  & 51.33&  & 44.44 &  & 38.11 &   &40.19&    &49.16&  &50.19\\
%& $\text{AU}$ &  64.22&  & 66.33&  &67.12 &  & 39.44 &   &50.31&    &55.17&  &60.15\\
%\hline
%\hline
%YaleB & &  97.11&  &  95.04&  &93.04 &  & 90.07 &   &82.14&    &92.19& &82.11\\
\hline
\end{tabular}}
\caption{ {Average Detection Accuracy  of  Clustering (\%) methods  on the Face Databases.}}
\label{tab:cluster}
\end{table*} 

\section{Conclusion}\label{sec:conc}
A new tensor model was introduced for data analytic tasks for  heterogeneous datasets, wherein there are joint low-rank structures within groups of observations, but also discriminative structures across different groups. The proposed model uses a double core tensor (DCOT) factorization together with a family of smoothing loss functions. By leveraging the proposed smoothing function, the model accurately estimates the model factors, even in the presence of missing entries. A linearized ADMM method was developed to solve regularized versions of DCOT factorizations, that avoid large tensor operations and large memory storage requirements. Further, we established theoretically its global convergence, together with consistency of the estimates of the model parameters. The effectiveness of the DCOT model was illustrated on several heterogeneous tensor data.

\section*{Acknowledgements}
The authors would like to acknowledge constructive and useful comments by two anonymous referees and the Action Editor. The  work of Davoud Ataee Tarzanagh was supported by ARO YIP award W911NF1910027, AFOSR YIP award FA9550-19-1-0026, NSF BIGDATA award IIS-1838179 and a Fellowship from the University of Florida Informatics Institute. The work of George Michailidis was supported in part by NSF grants DMS 1854476 and 2210358 and NIH grant 1U01CA235487-01. 

\appendix
\section*{Appendices}

\begin{table*}
\caption{Basic notation and product}
\centering
    \begin{tabular}{ll}%{@{}l*{10}{C}l@{}}
\hline
{$\T{A}, \; \bA, \; \ba, \; a$}& tensor, matrix, vector, scalar \\
$\bA = [\ba_1,\ba_2,\cdots,\ba_R]$ & {matrix  $\bA$ with column vectors $\ba_r$}   \\
$a_{i_1\cdots i_N}$ & \minitab[p{.6\linewidth}]{entry of tensor $\T{A} \in \mathbb{R}^{I_1 \times \cdots \times I_N}$}\\
$\ba(:,i_2,\cdots,i_N)$ & \minitab[p{.6\linewidth}]{fiber of tensor $\T{A}$ obtained by fixing all but one index}\\
$\bA(:,:,i_3,\cdots,i_N)$ & \minitab[p{.6\linewidth}]{matrix slice of tensor $\T{A}$ obtained by fixing all but two indices}\\
%$\T{A}(:,:,:,i_4,\cdots,i_N)$ & \minitab[p{.6\linewidth}]{tensor slice of $\T{A}$ obtained by fixing some indices}\\
%%CesarAddedOct17
$\T{A}(\mathcal{I}_1,\mathcal{I}_2,\cdots,\mathcal{I}_N)$ &  \minitab[p{.6\linewidth}] {subtensor of $\T{A}$ obtained by restricting indices to belong to subsets $\mathcal{I}_n \subseteq [I_n]\equiv \{ 1,2,\dots, I_n \}$}\\
$\bA_{(n)} \in \mathbb{R}^{I_n \times I_1 I_2 \cdots I_{n-1} I_{n+1} \cdots I_N}$ &  \minitab[p{.6\linewidth}]{mode-$n$ matricization of tensor $\T{A} \in \mathbb{R}^{I_1 \times I_2 \times \cdots \times I_N}$ whose entry at row $i_n$ and column $(i_1 -1) I_2 \cdots I_{n-1} I_{n+1} \cdots I_N +  \cdots + (i_{N-1} -1)I_N + i_N$ is equal to $a_{i_1\cdots i_N}$}
\\
$\rmvec{(\T{A})}\in \mathbb{R}^{I_N I_{N-1} \cdots I_1}$ &  \minitab[p{.6\linewidth}]{vectorization of tensor $\T{A} \in \mathbb{R}^{I_1 \times I_2 \times \cdots \times I_N}$ with the entry at position $i_1 +\sum_{k=2}^N[(i_k-1)I_1 I_2 \cdots I_{k-1}]$ equal to $a_{i_1\cdots i_N}$}\\
%$(i_N -1) I_{N-1} I_{N-2} \cdots I_1 +  \cdots + (i_2 -1)I_1 + i_1$ equal to $a_{i_1\cdots i_N}$
%with an entry at position $i_n$ and column $(i_N -1) I_{N} \cdots I_1 +  \cdots + (i_2 -1)I_1 + i_1$ equal to $a_{i_1\cdots i_N}$}
%$\Circ{\T{A}}$ & \minitab[p{.6\linewidth}]{ yields the circulant matrix of $A_{(n)}$ }\\
%$\Fold{\bA_{(n)}}$ & \minitab[p{.6\linewidth}]{ yields $\T{A} \in \mathbb{R}^{I_1 \times I_2 \times \cdots \times I_N}$
%}\\
$\langle \T{X},\T{Y}\rangle := \rmvec(\T{X})^T\rmvec(\T{Y})$ & \minitab[p{.6\linewidth}] {inner product of two tensors $\T{X},\T{Y} \in \R^{I_1\times I_2 \times \dots \times I_N}$}\\
$\bD = \mbox {diag} (a_1,a_2,\cdots,a_R)$ & \minitab[p{.6\linewidth}] {diagonal matrix with $d_{rr}=a_r$}  \\
${\bA}^{\top}$,   ${\bA}^{-1}$, ${\bA}^{\dag}$ &  \minitab[p{.6\linewidth}]{transpose, inverse, and Moore-Penrose pseudo-inverse} \\
$\|\bX\|_F$,   $\|\bX\|_*$, $\|\bX\|_1$ &  \minitab[p{.6\linewidth}] {Frobenius norm,  the trace norm or trace norm as the sum of singular values of $\bX$,  and  the $\ell_1$ norm.} \\
  $\|\bX\|_{2}$, $\|\bX\|_{b,a}$,  & \minitab[p{.6\linewidth}] {The spectral norm as the largest singular value of matrix,  the $\ell_a/\ell_b$ norm as the $\ell_a$ norm of
the vector formed by the $\ell_b$ norm of every row.} \\
%$\text{Shrink}(\cdot, \cdot)$ & \minitab[p{.6\linewidth}]{the soft-thresholds operator, applied element-wise to a vector $\bar{\bs}= \rmvec{(\bar{\T{S}})} \in  \mathbb{R}^{I_N I_{N-1} \cdots I_1}$:
%$\text{Shrink}(\bar{\bs}(i),b):=\sign(\bar{\bs}(i))\max \big(|\bar{\bs}(i)|-b,0\big)$.}\\
$\T{C} = \T{A} \times_n \bU$ &  \minitab[p{.6\linewidth}]{ mode-$n$ product of $\T{A} \in \mathbb{R}^{I_1 \times I_2 \times \cdots \times I_N}$ and $\bU \in \mathbb{R}^{j_N \times I_n}$ yields $\T{C} \in \mathbb{R}^{I_1 \times \cdots \times I_{n-1} \times j_N \times I_{n+1} \times \cdots \times I_N}$ with entries $c_{i_1 \cdots i_{n-1} \, j_N \, i_{n+1} \cdots i_N} = \sum_{i_n=1}^{I_n} a_{i_1 \cdots i_{n-1} \, i_n \, i_{n+1} \cdots i_N} b_{j_N \, i_n}$ and matrix representation $\bC_{(n)} = \bU \, \bA_{(n)}$} \\
{$\T{C} = [ \T{A}; \bU^{(1)}, \bU^{(2)}, \cdots, \bU^{(N)}]$ } & \minitab[p{.6\linewidth}] {full multilinear product, $\T{C} =   {\T{A}} \times_1 \bU^{(1)} \times_2 \bU^{(2)}
\cdots \times_N \bU^{(N)}$ } \\
$\T{C} = \T{A} \circ \T{B}$ & \minitab[p{.6\linewidth}] { tensor or outer product of $\T{A} \in \mathbb{R}^{I_1 \times I_2 \times \cdots \times I_N}$ and $\T{B} \in \mathbb{R}^{J_1 \times J_2 \times \cdots \times j_N}$ yields $\T{C} \in \mathbb{R}^{I_1 \times I_2 \times \cdots \times I_N \times J_1 \times J_2 \times \cdots \times j_N}$ with entries $c_{i_1\cdots i_N \,j_1\cdots j_N} = a_{i_1\cdots i_N} b_{j_1\cdots j_N}$} \\
$\T{X} = \ba^{(1)} \circ \ba^{(2)} \circ \cdots \circ \ba^{(N)}$ & \minitab[p{.6\linewidth}]{ tensor or outer product of vectors $\ba^{(n)} \in \mathbb{R}^{I_n}\;$ $(n=1,\cdots,N)$ yields a rank-1 tensor $\T{X} \in \mathbb{R}^{I_1 \times I_2 \times \cdots \times I_N}$ with entries% $x_{i_1\cdots i_N} = a^{(1)}_{i_1} a^{(2)}_{i_2} \cdots a^{(N)}_{i_N} $
}
\\
{$\bC = \bA\otimes\bU$}& \minitab[p{.6\linewidth}]{Kronecker product of $\bA \in \mathbb{R}^{I_1 \times I_2}$ and $\bU \in \mathbb{R}^{J_1 \times J_2}$ yields $\bC \in \mathbb{R}^{I_1 J_1 \times I_2 J_2}$ with entries  $c_{(i_1-1) J_1 +j_1,(i_2-1) J_2 + j_2} = a_{i_1 i_2} \: b_{j_1 j_2}$
}
\\
\\
 \hline
    \end{tabular}
%    \end{center}
\label{table_notation2}
\end{table*}

\section{Updating Parameters in Algorithm \ref{alg:lin:admm}} \label{sec:appendixb}

This section provides detailed implementation of the linearized ADMM method for solving problem~\eqref{eq:loss:general}.

\subsection*{Proof of Lemma \ref{lem:pgrad}}\label{sec:proof-lemma1}
\begin{proof}
For simplicity, we  assume that $N=3$ and $\T{S}=\T{G}+\T{H}$. We calculate the partial gradient of $\bar{\Lc}( [\T{S}; \bU^{(1)}, \bU^{(2)}, \bU^{(3)} ] )$ with respect to $\bU^{(1)}$ by chain rule:
\begin{equation*}
    \begin{split}
        \frac{\partial \bar{\Lc}}{\partial \bU_{ij}^{(1)}} & = \sum_{a=1}^{I_1} \sum_{b=1}^{I_2} \sum_{c=1}^{I_3} \frac{\partial \bar{\Lc}}{\partial \T{T}_{abc}} \cdot \frac{\partial \T{T}_{abc}}{\partial \bU_{ij}^{(1)}} \\
        & = \sum_{a=1}^{I_1} \sum_{b=1}^{I_2} \sum_{c=1}^{I_3} \frac{\partial \bar{\Lc}}{\partial \T{T}_{abc}} \cdot \left( 1_{\{a=i\}} \sum_{r_2=1}^{R_2}\sum_{r_3=1}^{R_3} \T{S}_{jr_2r_3}\bU_{br_2}^{(2)}\bU_{cr_3}^{(3)} \right)\\
        & = \sum_{b=1}^{I_2} \sum_{c=1}^{I_3} \frac{\partial \bar{\Lc}}{\partial \T{T}_{ibc}} \left(\sum_{r_2=1}^{R_2}\sum_{r_3=1}^{R_3} \T{S}_{jr_2r_3}\bU_{br_2}^{(2)}\bU_{c r_3}^{(3)}\right),
    \end{split}
\end{equation*}
where the second identity comes from the following fact:
$$\T{T}_{abc} = \sum_{r_1=1}^{R_1}\sum_{r_2=1}^{R_2}\sum_{r_3=1}^{R_3} \T{S}_{r_1r_2r_3}\bU_{ar_1}^{(1)}\bU_{br_2}^{(2)}\bU_{cr_3}^{(3)}.$$ Let $\T{M}=\nabla \bar{\Lc}(\T{T})$. One can verify that
\begin{equation*}
    \left(\T{M}_{(1)}(\bU^{(3)}\otimes \bU^{(2)})\bS_{(1)}^\top\right)_{ij} = \sum_{k_1=1}^{I_2}\sum_{k_2=1}^{I_3}\sum_{k_3=1}^{R_2} \sum_{k_4=1}^{R_3} \frac{\partial \bar{\Lc}}{\partial \T{T}_{ik_1k_2}}\cdot \bU_{k_1k_3}^{(2)}\bU_{ k_2k_4}^{(3)} \T{S}_{jk_3k_4}.
\end{equation*}
Here, $\T{M}_{(1)}$ and $\bS_{(1)}$ are mode-$1$ matricization of $\T{M}$ and $\T{S}$, respectively.

%which is exactly what we calculated for $\frac{\partial L}{\partial \bU_{ij}^{(1)}}$ (by changing the order of summation)
%
The partial gradient for $\bU^{(2)}$ and $\bU^{(3)}$ can be similarly calculated. For core tensor $\T{S}$, we have
\begin{equation*}
    \begin{split}
        \frac{\partial \bar{\Lc}}{\partial \T{S}_{ijk}} & = \sum_{a=1}^{I_1} \sum_{b=1}^{I_2} \sum_{c=1}^{I_3} \frac{\partial\bar{\Lc}}{\partial \T{T}_{abc}} \cdot \frac{\partial \T{T}_{abc}}{\partial \T{S}_{ijk}} \\
        & = \sum_{a=1}^{I_1} \sum_{b=1}^{I_2} \sum_{c=1}^{I_3} \frac{\partial \bar{\Lc}}{\partial \T{T}_{abc}} \cdot \bU_{ai}^{(1)} \bU_{bj}^{(2)} \bU_{ck}^{(3)}\\
        & = \left( \nabla \bar{\Lc}(\T{T}); \bU^{(1)}, \bU^{(2)}, \bU^{(3)} \rrbracket\right)_{ijk}\\
        &  = \left( [ \T{M}; \bU^{(1)}, \bU^{(2)}, \bU^{(3)} ]\right)_{ijk},
    \end{split}
\end{equation*}
which has finished the proof of this lemma. 
\end{proof}

\subsection*{{Updating} $\bU^{(n)}$}\label{sec:up:u}
We need the gradient of the function $\bar{\Lc}(\T{T})$ in \eqref{eq:quadprob} with respect to the factor matrices. It follows from Lemma~\ref{lem:pgrad} that 
\begin{eqnarray}\label{eq:gradu}
\nabla_{\bU^{(n)}} \bar{\Lc}(\T{T}_k^{\bU^{(n)}_k})= \nabla \bar{\Lc}(\T{T}_k^{\bU^{(n)}_k}) (\bigotimes_{t \neq n} \bU^{(t)}_k)(\bH_{k,(t)}+\bG_{k,(t)})^\top.
\end{eqnarray}
Substituting \eqref{eq:gradu} into \eqref{eq:finalup1}, we obtain
\begin{eqnarray}\label{eq:quad:u}
 \nonumber
 \bU^{(n)}_{k+1} &=& \argmin_{\bU^{(n)}} \quad \lambda_{3,n} J_{3,n}(\bU^{(n)})+ \langle \nabla_{\bU^{(n)}} \bar{\Lc}(\T{T}_k^{\bU^{(n)}_k}), \bU^{(n)}-\bU^{(n)}_k \rangle\\
 & + & \frac{ \varrho^n}{2 }\| \bU^{(n)}-\bU^{(n)}_k\|_F^2,
\end{eqnarray}
where  $\varrho^n$ is a constant equal or greater than the Lipschitz of the gradient $\nabla \bar{\Lc}(\T{T}^{\bU^{(n)}})$.

It follows from \eqref{D:ProximalMap} that \eqref{eq:quad:u} has the following solution
\begin{equation}\label{eq:update:u}
   \bU^{(n)}_{k+1} = \prox^{J_{3,n}}_{\varrho^n} \Big(\bU^{(n)}_{k} -\frac{1}{\varrho^n}  \nabla_{\bU^{(n)}} \bar{\Lc}(\T{T}^{\bU^{(n)}_k}_k) , \frac{\lambda_{3,n}}{\varrho^n} \Big).
\end{equation}
It is worth mentioning that we can use \eqref{eq:update:u} for different choices of penalty functions. As discussed in \eqref{DCOT:examples}, typical examples of the function $J_{3,n}(\bU^{(n)})$ include $\|\bU^{(n)}\|_1$, $\|\bU^{(n)}\|_*$ or the indicator of a closed convex convex set. For example, if $ J_{3,n}(\bU^{(n)}) =\|\bU^{(n)}\|_1$, then we can apply the $\ell_1$ proxiaml operator \citep{parikh2014proximal}. 

%We also note that the Lipschitz constant of $\nabla \bar{\Lc}(\T{T}^{\bU^{(n)}})$ is defined by 
%$\sigma_{1}\big(\gamma {\bZ}_{(n)}^t {\bZ^t}_{(n)}^\top \big),
%$where $\sigma_{1}$ denotes the largest singular value of a matrix.

\subsection*{\text{Updating $\T{G}$}}\label{sec:up:g}

From \eqref{eq:finalup2}, we have
\begin{eqnarray*}
%  \min_{\T{G}}  \lambda_1 \|\T{G}\|_1+  \frac{\gamma}{2} \|\T{X}^{\T{H}}_{k}-\T{P}_k -[\T{G},\bU^{(1)}_{k},\dots,\bU^{(N)}_{k}]\|_F^2,
   \T{G}_{k+1} &=& \argmin_{\T{G}} \quad   \langle \nabla \bar{\Lc} (\T{T}^{\T{G}}_k), \T{G}-\T{G}_k \rangle + 
 \lambda_1 J_1(\T{G}) + \frac{ \varrho^g}{2 }\| \T{G}-\T{G}_k\|_F^2 
\end{eqnarray*}
It follows from Lemma~\ref{lem:pgrad} that
\begin{align*}%\label{eq-four-part-gradient}
\nabla_{\T{G}} \bar{\Lc} (\T{T}^{\T{G}}_k)= \T{M}(\T{T}^{\T{G}}_k) \times_1 \bU^{(1)}_{k+1} \cdots \times_N \bU^{(N)}_{k+1},
\end{align*}
where 
\begin{equation}%\label{eqn:m}
\T{M}(\T{T}^{\T{G}}_k) = \gamma \Big(({{\T{G}}_k + {\T{H}}_k}) \times_1 {\bU}^{(1)}_{k+1} \cdots \times_N {\bU}^{(N)}_{k+1} -\T{Z}_{k} -\frac{1}{\gamma} \T{Y}_k \Big).
\end{equation}
Now, from \eqref{D:ProximalMap}, we obtain 
\begin{equation}\label{eq:update:g}
\T{g}_{k+1}=\prox^{J_1}_{\varrho^g}\big(\T{g}_k- \frac{1}{\varrho^g} \nabla_{\T{G}} \bar{\Lc} (\T{T}^{\T{G}_k}_{k}), \frac{\lambda_1}{\varrho^g}\big).
\end{equation}
%where $\varrho^g $  is equal or greater than the Lipschitz constant of the gradient $\nabla \bar{\Lc} (\T{T}^{\T{G}}_k)$ which is $ J_1(\bU^\top_k\bU_k)= \prod_{n \in [N]}J_1({\bU^{n}_{k}}^\top \bU^{n}_{k})$.

\subsection*{\text{Updating $\T{H}$}}\label{sec:up:h}

To minimize the sub-Lagrangian function w.r.t. $\T{H}$, using \eqref{eq:finalup3}, we have
\begin{eqnarray*}
%  \min_{\T{G}}  \lambda_1 \|\T{G}\|_1+  \frac{\gamma}{2} \|\T{X}^{\T{H}}_{k}-\T{P}_k -[\T{G},\bU^{(1)}_{k},\dots,\bU^{(N)}_{k}]\|_F^2,
   \T{H}_{k+1} &=& \argmin_{\T{H}} \quad   \langle \nabla \bar{\Lc} (\T{T}^{\T{H}_k}_k), \T{H}-\T{H}_k \rangle + 
 \lambda_2 J_2(\T{H}) + \frac{ \varrho^h}{2 }\| \T{H}-\T{H}_k\|_F^2.
\end{eqnarray*}
This problem is separable w.r.t. $\{\T{H}_{m_\pi}\}_{m=1}^M$ and we have that
\begin{eqnarray*}
\T{H}_{m,k+1} = \argmin_{\T{H}_{m_\pi}} && \langle \nabla \bar{\Lc} (\T{T}^{\T{H}_{m_\pi}}_k), \T{H}_{m_\pi}-\T{H}_{{m_\pi},k} \rangle \\
&+&  \lambda_2 J_2(\T{H}_{m_\pi}) + \frac{ \varrho^h}{2 }\| \T{H}_{m_\pi}-\T{H}_{{m_\pi},k}\|_F^2.
\end{eqnarray*}  
Let
\begin{eqnarray}\label{eq:resi}
\T{R}^{\T{G}}_k &=& \T{Z}_k-\gamma^{-1} \T{Y}_k - [\T{G}_{k+1}; \bU^{(1)}_{k+1}, \dots, \bU^{(N)}_{k+1}].
\end{eqnarray}
The gradient $ \nabla \bar{\Lc} (\T{T}^{\T{H}_{m_\pi}}_k)$ is equal to
$$ \nabla_{\T{H}_{m_\pi}} \Big(\frac{\gamma}{2} \|\T{H} \times_1 {\bU}^{(1)}_{k+1} \cdots \times_N {\bU}^{(N)}_{k+1} -\T{R}^{\T{G}}_{k}\|_F\Big).$$
Let $\br^g_{k}=\rmvec{(\T{R}^{\T{G}}_{k})}$ and  $\bh= \rmvec{(\T{H})}$. Then, we obtain 
\begin{eqnarray}
\|\bU_{k+1} \bh- \br^g_k\|_F^2 = \sum_{r_1\cdots r_N}  \sum_{i_1\cdots i_N} ((\bU_{k+1})_{i_1 i_2 \cdots i_N, r_1 \cdots r_N} \bh_{r_1 \cdots r_N}- (\br^g_k)_{i_1\cdots i_N})^2.
\end{eqnarray}
Hence
\begin{eqnarray}\label{eq:update:h000}
\nonumber
&& \big(\nabla \bar{\Lc} (\T{T}^{\T{H}_{m_\pi}})\big)_{r_1\cdots r_N}   \\
&=&  \gamma \sum_{r_1 \cdots r_N \in \T{H}_{m}}  \sum_{i_1 \cdots i_N} (\bU_{k+1})_{i_1 \cdots i_N, r_1 \cdots r_N} ((\bU_{k+1})_{i_1 \cdots i_N r_1 \cdots r_N} {\bh}_{r_1 \cdots r_N}- (\br^g_k)_{i_1  \cdots i_N}). 
\end{eqnarray}

Using \eqref{eq:update:h000} and \eqref{eq:finalup3}, we have
\begin{equation}\label{eq:update:h0}
\T{H}_{m_\pi, k+1}=\text{Prox}^{J_2}_{\varrho^h}\Big(\T{H}_{m_\pi,k}- \frac{1}{\varrho^h} \nabla \bar{\Lc} (\T{T}^{\T{H}_{m_\pi,k}}_{k}), \frac{\lambda_2}{\varrho^h} \Big), \qquad m=1,2, \ldots, M. 
\end{equation}
Finally, we set
\begin{equation}
\T{H}_{m_1, k+1} =  \T{H}_{m_2, k+1}= \cdots = \T{H}_{m_\pi, k+1}.
\end{equation}

\subsection*{Updating $\T{Z}$}\label{sec:up:z}
We derive an explicit formulation of the element-wise gradient w.r.t. $\T{Z}$ along with a straightforward way of handling missing data.  To do so, from \eqref{eq:finalup4}, we have
\begin{eqnarray}\label{eq:genz}
\nonumber
\T{Z}_{k+1} &=& \argmin_{\T{Z}} \Big\{F(s^h,\T{X};\T{Z}) \\
 & +& \frac{\gamma}{2} \|\T{Z}- \gamma^{-1} \T{Y}_{k}- (\T{G}_k + \T{H}_k) \times_1 \bU^{(1)}_{k+1} \cdots \times_N \bU^{(N)}_{k+1} \|^2_{F}
\Big\}.
\end{eqnarray}

Next, we provide a generalized framework for computing gradients and handling missing data that enables the use of standard optimization methods for solving \eqref{eq:genz}. Let 
\begin{align*}
\T{R} &= \T{Z}_k- \gamma^{-1} \T{Y}_{k}- (\T{G}_k + \T{H}_k) \times_1 \bU^{(1)}_{k+1} \cdots \times_N \bU^{(N)}_{k+1}, \quad \text{and} \\
\T{N} &= \nabla \bar{\Lc}(\T{T}^{\T{Z}}_k).  
\end{align*}
The gradient of the objective in \eqref{eq:genz} w.r.t. $z_{i_1 \cdots i_N} \in \Omega$ is given by 
\begin{eqnarray}%\label{eq:Y}
\T{N}_{i_1\cdots i_N} &=& \nabla_{z_{i_1\cdots i_n}} F(s^h,\T{X};\T{Z})+  \gamma \T{R}_{z_{i_1 \cdots i_N}},
\end{eqnarray}
where 
$$
F(s^h,\T{X};\T{Z}) =  - \frac{1}{\prod_{n \in [N]} I_n}\sum_{i_1=1}^{I_1}\cdots\sum_{i_N=1}^{I_N} f_{i_1\cdots i_N}(s^h , \T{X};\T{Z} ).
$$

Now, we can solve \eqref{eq:genz} via a gradient-based optimization method.  In our implementation, we use the limited-memory BFGS method  \citep{liu1989limited}.

\section{ Convergence Analysis of Linearized ADMM}
%
%In order to prove the global convergence result of our algorithm for DCOT, we need to introduce adequate and necessary material on the nonsmooth Kurdyka- Lojasiewicz property (for more details see \citet{bolte2014proximal}).
%

The next result shows that Algorithm \ref{alg:lin:admm} provides sufficient decrease of the augmented Lagrangian $\Lc(\T{T})$ in each iteration.
 
\begin{lemma} \label{lemdec}
Let Assumption~\ref{AssumptionsA} hold, and $\gamma > L$, $\varrho^{g}> L^g, \varrho^{h}>L^h$, and $\varrho^{n}> L^n$ for $n=1, 2, \dots, N$,  where $L^{n}, L^g, L^h$ and $L$ are Lipschitz constants of the gradients of Lagrangian function $\Lc$ w.r.t.  $\{\bU^{(n)}\}_{n=1}^N, \T{G}, \T{H}$ and $\T{Z}$, respectively. Then, for the sequence  $\{\T{T}_k\}_{k\geq 0}$  generated by Algorithm \ref{alg:lin:admm}, we have
\begin{eqnarray}\label{eq:suf}
\nonumber
\Lc(\T{T}^{\bU^{(n)}_k}_{k})-
\Lc(\T{T}^{\bU^{(n)}_{k+1}}_{k+1})&-& \frac{ \varrho^{(n)}-L^{(n)}}{2}\| \bU^{(n)}_{k}- \bU^{(n)}_{k+1}\|^2_{F} \geq 0, \qquad n=1, \dots, N,\\
\nonumber
\Lc(\T{T}^{\T{G}_{k}}_{k})-
\Lc(\T{T}^{\T{G}_{k+1}}_{k+1})&-& \frac{ \varrho^{g}-L^g}{2}\| \T{G}_{k}- \T{G}_{k+1}\|^2_{F} \geq 0,\\
\nonumber
\Lc(\T{T}^{\T{H}_{k}}_{k})-
\Lc(\T{H}^{\T{H}_{k+1}}_{k+1})&-& \frac{\varrho^{h}-L^h}{2}\| \T{H}_{k}- \T{H}_{k+1}\|^2_{F} \geq 0,\\
\nonumber
\Lc(\T{T}^{\T{Z}_{k}}_{k})-
\Lc(\T{T}^{\T{Z}_{k+1}}_{k+1})&-& \frac{ \gamma-L}{2}\| \T{Z}_{k}- \T{Z}_{k+1}\|^2_{F} \geq 0, \\
\Lc(\T{T}^{\T{Y}_k}_{k})-\Lc(\T{T}^{\T{Y}_{k+1}}_{k+1}) &+& \frac {{L}^2}{\gamma} \| \T{Z}_{k+1}-\T{Z}_{k} \|_F^2 \geq 0. 
\end{eqnarray}
\end{lemma}
\begin{proof}
The first optimality condition for \eqref{eq:genz} is given by
$$  
 \nabla  F(s^h,\T{X};\T{Z}_{k+1}) - \T{Y}_k - \gamma \Big(({\T{G}}_{k+1} + {\T{H}}_{k+1}\big) \times_1 {\bU}^{(1)}_{k+1} \cdots \times_N {\bU}^{(N)}_{k+1} -\T{Z}_{k+1} \Big) =0 
$$
which together with \eqref{eq:finalup5} implies that  the iterative gap of dual variable can be bounded by that of primal variable, i.e.,
\begin{eqnarray}\label{eq:yg}
\nonumber
\T{Y}_{k+1} = \nabla  F(s^h,\T{X};\T{Z}_{k+1}).
\end{eqnarray}
Now, using Assumption~\ref{AssumptionsA}, we have
\begin{eqnarray}\label{eqn:dualgap}
\nonumber
\|\T{Y}_{k+1} - \T{Y}_{k}\|_F  &\leq& \| \nabla F(s^h,\T{X};\T{Z}_{k+1}) -  \nabla F(s^h,\T{X};\T{Z}_k) \|_F \\
& \leq & L \| \T{Z}_{k+1}-\T{Z}_{k} \|_F.
\end{eqnarray}
Hence, we obtain 
\begin{equation}\label{dec:y}
\Lc(\T{T}^{\T{Y}_k}_{k})-\Lc(\T{T}^{\T{Y}_{k+1}}_{k+1}) \geq  \frac{1}{ \gamma}\| \T{Y}_{k+1}-\T{Y}_{k}\|^2_{F}  \geq - \frac {{L}^2}{\gamma} \| \T{Z}_{k+1}-\T{Z}_{k} \|_F^2. 
\end{equation}

We note that the function $\Lc(\T{T}^{\T{Z}}_{k})$ in \eqref{eq:genz} is strongly convex w.r.t. $\T{Z}$ whenever $\gamma > L$, where $L$ is a Lipschitz constant for $\nabla F(s^h,\T{X}; \T{Z})$; see, Assumption~\ref{AssumptionsA}. Thus, we have 
\begin{align*}\label{eq:opt:z}
 \Lc(\T{T}^{\T{Z}_{k+1}}_{k+1})- \Lc(\T{T}^{\T{Z}_{k}}_{k})&= F(s^h,\T{X};\T{Z}_{k+1}) - F(s^h,\T{X};\T{Z}_{k}) \\
 & + \frac{\gamma}{2} \| \big({{\T{G}}_{k+1} + {\T{H}}}_{k+1}\big) \times_1 {\bU}^{(1)}_{k+1} \cdots \times_N {\bU}^{(N)}_{k+1} -\T{Z}_{k+1} - \frac{1}{\gamma} \T{Y}_{k+1}\|^2_{F}\\
& - \frac{\gamma}{2} \| \big({{\T{G}}_{k} + {\T{H}}}_{k}\big) \times_1 {\bU}^{(1)}_{k} \cdots \times_N {\bU}^{(N)}_{k} -\T{Z}_{k} - \frac{1}{\gamma} \T{Y}_{k}\|^2_{F}\\
& {\leq}   \langle \nabla \Lc(\T{T}^{\T{Z}_{k+1}}_k),  \T{Z}_{k+1}-\T{Z}_{k} \rangle - \frac{\gamma-L}{2} \|\T{Z}_{k+1}-\T{Z}_{k} \|_F^2, \\
&= - \frac{\gamma-L}{2} \|\T{Z}_{k+1}-\T{Z}_{k} \|_F^2, 
\end{align*}
where the last equality follows from the first-order optimality condition for \eqref{eq:finalup4}.

The remainder of the proof of this lemma follows along similar lines to the proof of \citet[Lemma 1, p. 470]{bolte2014proximal}.
\end{proof}

The following lemma shows that the augmented Lagrangian has a sufficient decrease in each iteration and it is uniformly lower bounded.

\begin{lemma}\label{lem1}
Let  $\{\T{T}_k\}_{k\geq 0}$ be a sequence generated by Algorithm \ref{alg:lin:admm}, and set 
\begin{eqnarray}\label{eq:dk}
\nonumber 
\T{D}_k &=& \sum_{n=1}^{N}\|\bU^{(n)}_k-\bU^{(n)}_{k+1}\|_{F}^2
+\|\T{G}_k-\T{G}_{k+1}\|_{F}^2\\
&+& \|\T{H}_{k}-\T{H}_{k+1}\|_{F}^2 + \|\T{Z}_k-\T{Z}_{k+1}\|_{F}^2 + \|\T{Y}_k-\T{Y}_{k+1}\|_{F}^2.
\end{eqnarray}
Then, there exists a positive constant $\vartheta$ such that 
 \begin{eqnarray}\label{h5}
 \Lc( \T{T}_{k+1}) &\leq& \Lc (\T{T}_{k}) - \frac{\vartheta}{2} D_k,  \qquad \text{for all} ~~ k \geq 0,
\end{eqnarray}
and
\begin{eqnarray}\label{a34}
 D_k \leq \frac{2}{\vartheta}(\Lc(\T{T}_0)- \underline{\T{L}}).
\end{eqnarray}
Here, $\underline{\T{L}}$ is the uniform lower bound of $\Lc( \T{T}_{k})$.
\end{lemma}
\begin{proof} 
Let $\bar{\gamma} = \frac{\gamma^2 - \gamma L- 2{L}^2}{2\gamma} $ and
\begin{eqnarray}\label{hatgamma}
\nonumber
  \hat{\gamma} &=&  \max(  \varrho^h- L^h,  \varrho^g- L^g, \varrho^{(1)}- L^{(1)}, \dots,  \varrho^{(n)}- L^{(n)}),\\
   \vartheta &=&\max(\hat{\gamma}, \bar{\gamma}),
\end{eqnarray}
where $\gamma$ is a dual step-size. 

We note that the roots of the quadratic equation $\gamma^2 - \gamma L- 2{L}^2=0$ are $-L$ and $2L$. Hence, ours choice of dual step size $( \gamma > 2L)$ and regularization parameters in Algorithm \ref{alg:lin:admm} implies that $\vartheta >0 $. Now, using Lemma~\ref{lemdec}, we have
\begin{eqnarray*}%\label{a3}
 \nonumber % Remove numbering (before each equation)
\Lc(\T{T}_k)- \Lc (\T{T}_{k+1}) &\geq & \sum_{n=1}^{N} \frac{\varrho^n -L^n}{2} \|\bU_k^{(n)}-\bU_k^{(n+1)}\|^2_{F} \\
&+& \frac{\varrho^h -L^h}{2} \| \T{H}_k-\T{H}_{k+1}\|^2_{F} + \frac{\varrho^g -L^g}{2}  \|\T{G}_k-\T{G}_{k+1}\|^2_{F}\\
  &+&\frac{\gamma-L}{2} \|\T{Z}_k-\T{Z}_{k+1}\|^2_{F}
- \frac {{L}^2}{\gamma} \| \T{Z}_{k+1}-\T{Z}_{k} \| \\
&{\geq} &\frac{\hat{\gamma}}{2}( \sum_{n=1}^{N}\|\bU_k^{(n)}-\bU_{k+1}^{(n)}\|^2_{F} 
+ \| \T{H}_k-\T{H}_{k+1}\|^2_{F} + \|\T{G}_k-\T{G}_{k+1}\|^2_{F})\\
  &+&\frac{\gamma^2 -  \gamma L- 2{L}^2}{2\gamma}.\frac{1}{2}( \| \T{Z}_k-\T{Z}_{k+1}\|^2_{F} + \|\T{Z}^k-\T{Z}_{k+1}\|^2_{F})\\
%  &= &\frac{ \hat{\gamma} }{2}(\sum_{n=1}^{N}\|\bU_k^{(n)}-\bU_{k+1}^{(n)}\|^2_{F} 
%+ \| \T{H}_k-\T{H}_{k+1}\|^2_{F} + \|\T{G}_k-\T{G}_{k+1}\|^2_{F})\\
%  &+& \frac{\gamma^2- 2{L}^2}{2\gamma(1+ 2{L}^2)}\|\T{Z}^k-\T{Z}_{k+1}\|^2_{F}   \\
%  &+& \frac{{L}^2\gamma^2- 2{L}^4}{2\gamma(1+ 2{L}^2)} \|\T{Z}^k-\T{Z}_{k+1}\|^2_{F} \\
%&\geq &\frac{  \hat{\gamma} }{2}( \sum_{n=1}^{N}\|\bU_k^{(n)}-\bU_{k+1}^{(n)}\|^2_{F} 
%+ \| \T{H}_k-\T{H}_{k+1}\|^2_{F} + \|\T{G}_k-\T{G}_{k+1}\|^2_{F})\\
%  &+& \frac{\gamma^2 (\varrho^z -L)- 8{\varrho^z}^2}{2\gamma(1+ 8{\varrho^z}^2)} \Big( \|\T{Z}^k-\T{Z}_{k+1}\|^2_{F} +\|\T{Y}^k-\T{Y}_{k+1}\|^2_{F} \Big) \\
%    \nonumber
&\geq&\frac{\vartheta}{2}\Big( \sum_{n=1}^{N}\|\bU_k^{(n)}-\bU_k^{(n+1)}\|^2_{F} 
+ \| \T{H}_k-\T{H}_{k+1}\|^2_{F} + \|\T{G}_k-\T{G}_{k+1}\|^2_{F}
\\
 &+& \|\T{Z}_k-\T{Z}_{k+1}\|^2_{F}+\|\T{Y}_k-\T{Y}_{k+1}\|^2_{F}\Big),
\end{eqnarray*}
where the last inequality follows from \eqref{hatgamma}. 

To show \eqref{a34}, let $\tilde{\T{Z}}= (\tilde{\T{G}} +\tilde{\T{H}}) \times_1 \tilde{\bU}^{(1)} \cdots \times_N \tilde{\bU}^{(N)}$. Using Assumption \ref{AssumptionsA}, we have  
\begin{eqnarray}\label{a505}
\nonumber
F( s^h, \T{X}; \tilde{\T{Z}}_{k+1})  & \leq & F( s^h, \T{X}; \T{Z}_{k})  +  \langle \nabla F(s^h, \T{X};\T{Z}_{k+1}) , \tilde{\T{Z}}_{k+1} -\T{Z}_{k+1} \rangle \\
\nonumber
&+&  \frac{L}{2} \|\tilde{\T{Z}}_{k+1}- \T{Z}_{k+1}\|_F^2, \\
&= &F( s^h, \T{X};\T{Z}_{k}) + \langle \T{Y}, \tilde{\T{Z}}_{k+1} -\T{Z}_{k+1} \rangle +  \frac{L}{2} \|\tilde{\T{Z}}_{k+1}- \T{Z}_{k+1}\|_F^2. 
\end{eqnarray} 

By Assumption \ref{assu:smooth}, penalty functions $\{J_1(.), J_2(.),  J_{3,1}(.),   J_{3,2}(.), \cdots,  J_{3,N}(.)\}$ and the loss function $F$ are lower bounded. Now, since $\gamma > L$,  we get 
\begin{eqnarray} \label{a33}
\nonumber
\Lc (\T{T}_{k+1}) &\geq & F( s^h, \T{X};\tilde{\T{Z}}_{k+1})  +  \lambda_1 J_1(\T{G}_{k+1})+ \lambda_2 J_{2}{(\T{G}_{k+1})}\\
&+&\lambda_3 J_3{(\bU_{k+1})} +  \frac{\gamma-L}{2}\|\tilde{\T{Z}}_{k+1}-\T{Z}_{k+1}\|^2_{F},\\
\nonumber
&\geq & \underline{F}+  \underline{J_1} + \underline{J_2}+ \underline{J_3}\geq\underline{\T{L}},
\end{eqnarray}
where where $\underline{\T{L}}$ is the uniform lower bound of $\Lc( \T{T}_{k})$.

Finally, using \eqref{h5}, we obtain \eqref{a34}.
\end{proof}

Next, we give a formal definition of the limit point set. Let the sequence $\{\T{T}_k\}_{k\geq 0}$ be a sequence generated by the 
Algorithm \ref{alg:lin:admm} from a starting point $\T{T}_0$. The set of all limit points is denoted by $\Upsilon(\T{T}_0)$, i.e.,
\begin{equation}\label{set:lim}
\Upsilon(\T{T}_0) = \left\{ \bar{\T{T}}: \exists
\text{ an infinite sequence } \{\T{T}_{k_s}\}_{s \geq 0}
\text{ such that }\T{T}_{k_s} \rightarrow \bar{\T{T}} \text{ as }s\rightarrow\infty \right\}.
\end{equation}
We now show that the set of accumulations points of the sequence $\{\T{T}_k\}_{k\geq 0}$ generated by Algorithm \ref{alg:lin:admm} is nonempty and it is a subset of the critical points of $\Lc$.

\begin{lemma}\label{lem2}
Let $\{\T{T}_k\}_{k\geq 0}$ be a sequence generated by Algorithm \ref{alg:lin:admm}. Then, 
\\
(i) $\Upsilon(\T{T}_0)$ is a non-empty set, and any point in $\Upsilon(\T{T}_0)$ is a critical point of $\Lc(\T{T})$;\\
(ii) $\Upsilon(\T{T}_0)$ is a compact and connected set;  \\
(iii) The function $\Lc(\T{T})$ is finite and constant on  $\Upsilon(\T{T}_0)$.
\end{lemma}
\begin{proof}
(i). It follows from \eqref{a34} that the sequence $\{\T{T}_k\}_{k\geq 0}$ is bounded which implies that $\Upsilon(\T{T}_0)$ is non-empty due to the Bolzano-Weierstrass Theorem. Consequently, there exists a sub-sequence $\{\T{T}_{k_s}\}_{s\geq 0}$, such that 
\begin{equation}\label{lim:seeq}
\T{T}_{k_s}\rightarrow \T{T}_*, \qquad \text{as} \quad s\rightarrow \infty.
\end{equation}
Since $J_{3,n}$ is lower semi-continuous, \eqref{lim:seeq} yields
\begin{equation} \label{eq:inf}
\liminf_{s\rightarrow\infty} J_{3,n} (\bU_{k_s}^{(n)}) \geq  J_{3,n} (\bU_{*}^{(n)}).
\end{equation}
Further, from the iterative step \eqref{eq:finalup1}, we have 
\begin{eqnarray*}
 \bU^{(n)}_{k+1} = \argmin_{\bU^{(n)}} \quad \lambda_{3,n} J_{3,n}(\bU^{(n)})+ \langle \nabla_{\bU^{(n)}} \bar{\Lc}(\T{T}_k^{\bU^{(n)}_k}), \bU^{(n)}-\bU^{(n)}_k \rangle + \frac{ \varrho^n}{2 }\| \bU^{(n)}-\bU^{(n)}_k\|_F^2.
\end{eqnarray*}
Thus, letting $\bU^{(n)}= \bU^{(n)}_{*}$ in the above, we get
\begin{eqnarray}\label{eq:lin}
\nonumber
 && \lambda_{3,n} J_{3,n}(\bU^{(n)}_{k+1})+ \langle \nabla_{\bU^{(n)}} \bar{\Lc}(\T{T}_k^{\bU^{(n)}_k}), \bU^{(n)}_{k+1}-\bU^{(n)}_k \rangle + \frac{ \varrho^n}{2 }\| \bU^{(n)}_{k+1}-\bU^{(n)}_k\|_F^2, \\
 &\leq &  \lambda_{3,n} J_{3,n}(\bU^{(n)}_*)+ \langle \nabla_{\bU^{(n)}} \bar{\Lc}(\T{T}_k^{\bU^{(n)}_k}), \bU^{(n)}_{*}-\bU^{(n)}_k \rangle + \frac{ \varrho^n}{2 }\| \bU^{(n)}_{*}-\bU^{(n)}_k\|_F^2, 
\end{eqnarray}
Choosing $k = k_s-1$ in the above inequality and letting $s$ goes to $\infty$, we obtain 
\begin{eqnarray}\label{eq:sup}
\limsup_{s\rightarrow\infty}  J_{3,n}(\bU^{(n)}_{k_s}) \leq   J_{3,n}(\bU^{(n)}_*) .
\end{eqnarray}
Here, we have used the fact that $ \nabla_{\bU^{(n)}} \bar{\Lc}$ is a gradient Lipchitz continuous function w.r.t. $\bU^{(n)}$, the sequence  $\bU^{(n)}_{k}$ is bounded and that the distance between two successive iterates tends to zero; see, \eqref{a34}. Now, we combine \eqref{eq:inf} and \eqref{eq:sup} to obtain
\begin{eqnarray}\label{eq:limu}
\lim_{s\rightarrow\infty}  J_{3,n}(\bU^{(n)}_{k_s}) =  J_{3,n}(\bU^{(n)}_*)~~ \text{for all}~~ n \in N.
\end{eqnarray}
Arguing similarly with other variables,  we obtain 
\begin{subequations}
\begin{eqnarray}\label{eq:limfull}
\lim_{s\rightarrow\infty}  J_{1}(\T{G}_{k_s}) &=&  J_{1}(\T{G}_*), \label{eq:limfull1} \\
\lim_{s\rightarrow\infty}  J_{2}(\T{H}_{k_s}) & =&  J_{2}(\T{H}_*),  \label{eq:limfull2} \\
\lim_{s\rightarrow\infty}  F( s^h, \T{X}; \T{Z}_{k_s})  &=&  F( s^h, \T{X}; \T{Z}_{*}), \label{eq:limfull3}\\
\lim_{s\rightarrow\infty}  \bar{\Lc} (\T{T}_{k_s}) &=&  \bar{\Lc} (\T{T}^*), \label{eq:limfull4}
\end{eqnarray} 
\end{subequations}
where \eqref{eq:limfull1} and \eqref{eq:limfull2} follow since $ J_{1}$ and $ J_{2}$ are lower semi-continuous; \eqref{eq:limfull3} and \eqref{eq:limfull4} are obtained from the continuity of functions $F$ and $ \bar{\Lc}$. Thus, $\lim_{s\rightarrow\infty} \Lc ( \T{T}_{k_s})=\Lc ( \T{T}_*)$. 

Next, we show that $\T{T}_*$ is a critical point of $\Lc (.)$. By the first-order optimality condition for the augmented Lagrangian function in \eqref{eq:aug}, we have
\begin{eqnarray} \label{a55}
\nonumber
&& \partial J_{3,n}(\bU^{(n)}_{k+1}) + \nabla_{\bU^{(n)}} \bar{\Lc}(\T{T}^{\bU^{(n)}_{k+1}}_{k+1}) \in \partial_{\bU^{(n)}} \Lc(\bU^{(n)}_{k+1}),\qquad n=1,\dots, N,\\
\nonumber
&&\partial J_{2}(\T{H}_{k+1}) + \nabla_{\T{H}} \bar{\Lc}(\T{T}^{\T{H}_{k+1}}_{k+1}) \in \partial_{\T{H}} \Lc(\T{H}_{k+1}),\\
\nonumber
&&\partial J_{1}(\T{G}_{k+1}) + \nabla_{\T{G}} \bar{\Lc}(\T{T}^{\T{G}_{k+1}}_{k+1}) \in \partial_{\T{G}} \Lc(\T{G}_{k+1}),\\
\nonumber
&&  \nabla_{\T{Z}} F( s^h, \T{X}; \T{Z}_{k+1}) + \nabla_{\T{Z}} \bar{\Lc}(\T{T}^{\T{Z}_{k+1}}_{k+1}) = \nabla_{\T{Z}} \Lc(\T{Z}_{k+1}),\\
&& \gamma  \big(({\T{G}}_{k+1} + {\T{H}}_{k+1}\big) \times_1 {\bU}^{(1)}_{k+1} \cdots \times_N {\bU}^{(N)}_{k+1} -\T{Z}_{k+1} \big) = -  \nabla_{\T{Y}} \Lc(\T{Y}^{k+1}).
\end{eqnarray}
%Combine \eqref{eqn:dualz} with \eqref{eq:finalup5} to obtain
%\begin{eqnarray}
% \big(({\T{G}}_{k+1} + {\T{H}}_{k+1}\big) \times_1 {\bU}^{(1)}_{k+1} \cdots \times_N {\bU}^{(N)}_{k+1} -\T{Z}_{k+1} \big)  
%\rightarrow 0.
%\end{eqnarray} 
Similarly, by the first-order optimality condition for  subproblems \eqref{eq:finalup1}--\eqref{eq:finalup4}, we have
\begin{eqnarray} \label{a56}
\nonumber
&&   \partial J_{3,n}(\bU^{(n)}_{k+1}) + \nabla_{\bU^{(n)}} \bar{\Lc}(\T{T}^{\bU^{(n)}_k}_k)  + \rho^{(n)} (\bU_k^{(n)}-\bU_{k+1}^{(n)}) =0,  \qquad n=1, \ldots, N,\\
\nonumber
&& \partial J_{2}(\T{H}_{k+1}) + \nabla_{\T{H}} \bar{\Lc}(\T{T}^{\T{H}_{k}}_{k}) + \rho^{\T{H}} (\T{H}_k-\T{H}_{k+1}) =0,\\
\nonumber
&& \partial J_{1}(\T{G}_{k+1}) + \nabla_{\T{G}} \bar{\Lc}(\T{T}^{\T{G}_{k}}_{k}) + \rho^{\T{G}} (\T{G}_k-\T{G}_{k+1}) =0, \\
&& \nabla_{\T{Z}} F( s^h, \T{X}; \T{Z}_{k+1})= 0.
\end{eqnarray}

Combine \eqref{a55} with \eqref{a56} to obtain

\begin{equation}\label{a59}
(\xi_{k+1}^1, \dots, \xi_{k+1}^N, \xi_{k+1}^{\T{G}},\xi_{k+1}^{\T{H}}, \xi_{k+1}^{\T{Z}}, \xi_{k+1}^{\T{Y}})\in \partial \Lc (\T{T}_{k+1}),
\end{equation}
where
\begin{eqnarray} \label{a57}
\nonumber
&& \xi^{k+1}_{\bU^{(n)}} :=  \nabla_{\bU^{(n)}} \bar{\Lc}(\T{T}^{\bU^{(n)}_{k+1}}_{k+1}) - \nabla_{\bU^{(n)}} \bar{\Lc}(\T{T}^{\bU^{(n)}_k}_k)  - \rho^{(n)} (\bU_k^{(n)}-\bU_{k+1}^{(n)}),  \qquad n=1, \ldots, N, \\
\nonumber
&& \xi^{k+1}_{\T{H}} := \nabla_{\T{H}} \bar{\Lc}(\T{T}^{\T{H}_{k+1}}_{k+1}) - \nabla_{\T{H}} \bar{\Lc}(\T{T}^{\T{H}_{k}}_{k}) - \rho^{\T{H}} (\T{H}_k-\T{H}_{k+1}),\\
\nonumber
 && \xi^{k+1}_{\T{G}} := \nabla_{\T{G}} \bar{\Lc}(\T{T}^{\T{G}_{k+1}}_{k+1}) - \nabla_{\T{G}} \bar{\Lc}(\T{T}^{\T{G}_{k}}_{k}) - \rho^{\T{G}} (\T{G}_k-\T{G}_{k+1}),\\
\nonumber
&& \xi^{k+1}_{\T{Z}} := \nabla_{\T{Z}} \bar{\Lc}(\T{T}^{\T{Z}_{k+1}}_{k+1}) = \T{Y}_k-\T{Y}_{k+1},\\
&&  \xi^{k+1}_{\T{Y}} := \frac{1}{\gamma}  (\T{Y}_k-\T{Y}_{k+1}).
\end{eqnarray}
Note that the function $\bar{\Lc} (\T{T})$ defined in \eqref{eq:quadprob}  is a gradient Lipchitz continuous function w.r.t. $\bU^{(1)} , \dots, \bU^{(N)},\T{G},\T{H}$. Thus, 
\begin{eqnarray} \label{aa57}
\nonumber
\|\nabla_{\bU^{(n)}} \bar{\Lc}(\T{T}^{\bU^{(n)}_{k+1}}_{k+1}) - \nabla_{\bU^{(n)}} \bar{\Lc}(\T{T}^{\bU^{(n)}_k}_k) \| & \leq &  \rho^{(n)} \|\bU_k^{(n)}-\bU_{k+1}^{(n)}\|,  \qquad n=1, \ldots, N, \\
\nonumber
 \| \nabla_{\T{H}} \bar{\Lc}(\T{T}^{\T{H}_{k+1}}_{k+1}) - \nabla_{\T{H}} \bar{\Lc}(\T{T}^{\T{H}_{k}}_{k}) \| & \leq &   \rho^{\T{H}} \| \T{H}_k-\T{H}_{k+1} \|,\\
\| \nabla_{\T{G}} \bar{\Lc}(\T{T}^{\T{G}_{k+1}}_{k+1}) - \nabla_{\T{G}} \bar{\Lc}(\T{T}^{\T{G}_{k}}_{k}) \| & \leq &  \rho^{\T{G}} \| \T{G}_k-\T{G}_{k+1}\|,
\end{eqnarray}  
Using \eqref{aa57}, \eqref{a34} and \eqref{a57}, we obtain
\begin{eqnarray} \label{a66}
\lim_{k\rightarrow\infty} \big(\|\xi_{k+1}^1\|, \dots, \|\xi_{k+1}^N\|, \|\xi_{k+1}^{\T{G}}\|,\|\xi_{k+1}^{\T{H}}\|, \|\xi_{k+1}^{\T{Z}}\|, \|\xi_{k+1}^{\T{Y}}\|\big) =\big(0,\dots, 0\big).
\end{eqnarray}
Now, from \eqref{a59} and \eqref{a66}, we conclude that $(0,\dots,0)\in \partial \Lc(\T{T}_*)$ due to the closedness property of $\partial \Lc$. Therefore,  $\T{T}_*$ is a critical point of $\Lc (.)$. This completes the proof of (i).
\\
(ii). The proof follows from \citet[Lemma 5 and Remark 5]{bolte2014proximal}.
\\
(iii). Let  $\T{L}_* = \lim\limits_{k\rightarrow\infty} \Lc(\T{T}_k)$. Choose $\T{T}_*\in \Upsilon(\T{T}_0)$. There exists a subsequence $\T{T}_{k_s}$ converging to $\T{T}_{*}$ as $s$ goes to infinity. Since we have proven that $
\lim\limits_{s\rightarrow\infty}\Lc(\T{T}_{k_s}) = \Lc(\T{T}_{*})$, and $\Lc(\T{T}_k)$ is a non-increasing sequence, we conclude that $\Lc(\T{T}_*) = \T{L}_*$, hence the restriction of $\Lc(\T{T})$ to $\Upsilon(\T{T}_{0})$ equals $\T{L}_*$.
\end{proof}
\subsection*{Proof of Theorem \ref{thm:main}}
\begin{proof}
The augmented Lagrangian function $\Lc(\T{T})$ is a Kurdyka-Lojasiewicz function. Further, by Lemma \ref{lem2}, $\Lc(\T{T})$ is constant on $\Upsilon(\T{T}_0)$ and the set $\Upsilon(\T{T}_0)$ defined in \eqref{set:lim} is compact. Putting all these together, the proof of this theorem follows along similar lines to the proof of \citet[Theorem~1]{bolte2014proximal}, with function $\psi(x)$ replaced by $\Lc(\T{T})$.
\end{proof}
\section{Consistency of DCOT}

In this section, we provide the proof of Theorem~\ref{l2convergence} and Corollary~\ref{rate}. 

\subsection*{Proof of Theorem~\ref{l2convergence}}

\begin{proof}
We bound empirical processes induced by $\rho(.,.)$ defiend in \eqref{K} by a chaining argument as in \citet{wong1995probability,shen1998method}. We define a partition of $\Gamma(\T{Z})$ (the parameter space of $\T{Z}$) as follows: 
 
\begin{eqnarray}\label{eq:paramspace}
A(\zeta) &=&\Big\{\T{Z}\in \Gamma(\T{Z}): \zeta \leq \rho(\T{Z}^*, \T{Z}) \leq 2\zeta \Big\}.
\end{eqnarray}
From the definition of $\rho(.,.)$ in \eqref{K},  for any ${\T{Z} \in A(\zeta)}$, we have
\begin{equation}\label{eq:bT3}
\zeta^2 \leq  \frac{1}{\prod_{n \in [N]} I_n}\sum_{i_1=1}^{I_1}\cdots\sum_{i_N=1}^{I_N}  (z_{i_1 \cdots i_N}^*- z_{i_1 \cdots i_N})^2  \leq 4\zeta^2.  
\end{equation}
Since $\widehat{\T{Z}}$ is a minimizer of \eqref{eq:loss:general}, we obtain
\begin{eqnarray}\label{eq:outermeasure}
\nonumber
&&P\Big(\rho(\widehat{\T{Z}},\T{Z}^*)) \geq  \eta \Big) \\
&&\leq P^*\Big( \sup_{ \T{Z} \in \bigcup_{\zeta} A(\zeta)} (F(s^h, \T{X}; \T{Z}^*)- F(s^h, \T{X}; \T{Z}) +\lambda (J(\T{Z}^*)- J(\T{Z})) 
\geq 0\Big),
\end{eqnarray}
where $P^*$ denotes the outer probability measure  \citet{billingsley2013convergence}. 

From the definition of the objective $F$ in \eqref{eq:loss:smooth}, we get
\begin{align}\label{eq:lossdiff}
\nonumber
& \quad F(s^h, \T{X}; \T{Z}^*) -F(s^h, \T{X}; \T{Z}) \\ 
&= \frac{1}{\prod_{n \in [N]} I_n}\sum_{i_1=1}^{I_1}\cdots\sum_{i_N=1}^{I_N} \underbrace{ f(s^h , \T{X}; z_{i_1 \cdots i_N}^*)-  f(s^h , \T{X}; z_{i_1 \cdots i_N})}_{f^{\Delta}_{i_1 \cdots i_N}}. 
\end{align}
Now, using the definition of $\ell_2$--smoothing loss in \eqref{eq:quad:loss}, we obtain an upper bound for the loss difference $f^{\Delta}_{i_1 \cdots i_N}$ as follows: 
\begin{align}\label{eq:bthm}
\nonumber
f^{\Delta}_{i_1 \cdots i_N}&= \sum_{(j_1,\cdots, j_N) \in \Omega}  s_{i_1\cdots i_N, j_1\cdots j_N}^h (z_{i_1 \cdots i_N}- z_{i_1 \cdots i_N}^*)\cdot (2 x_{j_1\cdots j_N}- z_{i_1 \cdots i_N}^* -z_{i_1 \cdots i_N})\\
\nonumber
& = \sum_{(j_1,\cdots ,j_N ) \in \Omega}  2 s_{i_1\cdots i_N, j_1\cdots j_N}^h (z_{i_1 \cdots i_N}- z_{i_1 \cdots i_N}^*)(z_{j_1\cdots j_N}^*- z_{i_1\cdots i_N}^*) \\
\nonumber
&+  \sum_{(j_1,\cdots ,j_N ) \in \Omega } 2 s_{i_1\cdots i_N, j_1\cdots j_N}^h (z_{i_1 \cdots i_N}-z_{i_1 \cdots i_N}^*)\varepsilon_{j_1\cdots j_N} - (z_{i_1 \cdots i_N}^*- z_{i_1 \cdots i_N})^2 \\
\nonumber
& \leq   \underbrace{a_1\sqrt{R_{\max}}|z_{i_1 \cdots i_N}^*- z_{i_1 \cdots i_N}|
 \sum_{(j_1,\cdots, j_N)\in \Omega}2 s_{i_1\cdots i_N, j_1\cdots j_N}^h \sum_{n=1}^{N} \big(d(\by_{i_n},\by_{j_n})\big)^\alpha}_{T^{(1)}_{i_1\cdots i_N}}\\
 & +  \underbrace{  \sum_{(j_1,\cdots,j_N)\in \Omega} 2s_{i_1\cdots i_N, j_1\cdots j_N}^h(z_{i_1 \cdots i_N}- z_{i_1 \cdots i_N}^*)\varepsilon_{j_1\cdots j_N}}_{T^{(2)}_{i_1\cdots i_N}} -(z_{i_1 \cdots i_N}^*- z_{i_1 \cdots i_N})^2,
\end{align}
where the second equality uses our assumption that $x_{j_1\cdots j_N}=z^*_{j_1\cdots j_N}+ \epsilon_{j_1\cdots j_N}$ and the second inequality follows form Assumption \ref{assu:smooth}. 

By substituting~\eqref{eq:bthm} into \eqref{eq:lossdiff} and using \eqref{eq:bT3}, we obtain  
\begin{align}\label{eq:sub}
\nonumber
& \sup_{A(\zeta)} F(s^h, \T{X}; \T{Z}^*)- F(s^h, \T{X}; \T{Z})\\
&= \sup_{A(\zeta)} \Big( \frac{1}{\prod_{n \in [N]} I_n}\sum_{i_1=1}^{I_1}\cdots\sum_{i_N=1}^{I_N} T^{(2)}_{i_1\cdots i_N} + T^{(1)}_{i_1\cdots i_N}\Big) - \zeta^2.
\end{align}
Since $\sum_{i} a_i b_i \leq (\sum_{i} a_i^2)^{\frac{1}{2}} (\sum_{i} b_i^2)^{\frac{1}{2}} $, it follows from \eqref{eq:bthm} that 
\begin{eqnarray*}
&& \frac{1}{\prod_{n \in [N]} I_n}\sum_{i_1=1}^{I_1}\cdots\sum_{i_N=1}^{I_N} T^{(1)}_{i_1\cdots i_N} \\
& \leq &   a_1 \sqrt{R_{\max}} \Big(\frac{1}{\prod_{n \in [N]} I_n}\sum_{i_1=1}^{I_1}\cdots\sum_{i_N=1}^{I_N} (z_{i_1 \cdots i_N}- z_{i_1 \cdots i_N}^*)^2 \Big)^{1/2}\\
\nonumber
& \cdot &  \Big( \frac{1}{\prod_{n \in [N]} I_n}\sum_{i_1=1}^{I_1}\cdots\sum_{i_N=1}^{I_N} \big(\sum_{(j_1,\cdots ,j_N ) \in \Omega} 2 s_{i_1\cdots i_N, j_1\cdots j_N}^h \sum_{n=1}^{N} \big(d(\by_{i_n},\by_{j_n})\big)^\alpha  \big)^2\Big)^{1/2}.
\end{eqnarray*}
Now, using Assumption~\ref{assu:smooth} and \eqref{eq:bT3}, we obtain
\begin{eqnarray} \label{eq:bT1}
 \sup_{A(\zeta)} \frac{1}{\prod_{n \in [N]} I_n}\sum_{i_1=1}^{I_1}\cdots\sum_{i_N=1}^{I_N} T^{(1)}_{i_1\cdots i_N}   \leq   2\zeta a_1 \phi_1 \sqrt{R_{\max}}. 
\end{eqnarray}

Using the fact that $\sum_{i} a_i b_i \leq (\sum_{i} a_i^2)^{\frac{1}{2}} (\sum_{i} b_i^2)^{\frac{1}{2}} $, for $T^{(2)}_{i_1\cdots i_N}$ defined in \eqref{eq:bthm}, we have 
\begin{eqnarray}\label{eq:bTT}
\nonumber
 && \frac{1}{\prod_{n \in [N]} I_n}\sum_{i_1=1}^{I_1}\cdots\sum_{i_N=1}^{I_N} T^{(2)}_{i_1\cdots i_N} \\
 \nonumber
 &\leq&  \frac{1}{\prod_{n \in [N]} I_n}\sum_{i_1=1}^{I_1}\cdots\sum_{i_N=1}^{I_N} 
\Big( \sum_{(j_1,\cdots,j_N ) \in \Omega}
 2s_{i_1\cdots i_N, j_1\cdots j_N}^h (z_{i_1 \cdots i_N}- z_{i_1 \cdots i_N}^*)\varepsilon_{j_1\cdots j_N} \Big)\\
\nonumber
  & \leq & \Big( \frac{1}{\prod_{n \in [N]} I_n}\sum_{i_1=1}^{I_1}\cdots\sum_{i_N=1}^{I_N} (z_{i_1 \cdots i_N}-z_{i_1 \cdots i_N}^*)^2 \Big)^{1/2}\\
  & \cdot& \Big( \frac{1}{\prod_{n \in [N]} I_n}\sum_{i_1=1}^{I_1}\cdots\sum_{i_N=1}^{I_N}( \sum_{(j_1,\cdots,j_N)\in \Omega} 2s_{i_1\cdots i_N, j_1\cdots i_N}^h \varepsilon_{j_1\cdots j_N})^2 \Big)^{1/2}. 
\end{eqnarray} 

We combine \eqref{eq:bTT} with \eqref{eq:bT3} to obtain  
\begin{eqnarray}\label{eq:bT2}
\nonumber
&& \quad \sup_{A(\zeta)} \frac{1}{\prod_{n \in [N]} I_n}\sum_{i_1=1}^{I_1}\cdots\sum_{i_N=1}^{I_N} T^{(2)}_{i_1\cdots i_N}  \\
 && \leq  \zeta  \Big( \frac{1}{\prod_{n \in [N]} I_n}\sum_{i_1=1}^{I_1}\cdots\sum_{i_N=1}^{I_N}( \sum_{(j_1,\cdots,j_N)\in \Omega} 2s_{i_1\cdots i_N, j_1\cdots i_N}^h \varepsilon_{j_1\cdots j_N})^2 \Big)^{1/2}.  
\end{eqnarray}

Substituting  \eqref{eq:bT1} and \eqref{eq:bT2} into \eqref{eq:sub} yields 
\begin{align}\label{eq:sub22}
\nonumber
& \sup_{A(\zeta)} \Big(F(s^h, \T{X}; \T{Z}^*) - F(s^h,\T{X}; {\T{Z}})\Big) 
=  2\zeta a_1 \phi_1 \sqrt{R_{\max}} \\
&+ \zeta  \Big( \frac{1}{\prod_{n \in [N]} I_n}\sum_{i_1=1}^{I_1}\cdots\sum_{i_N=1}^{I_N}( \sum_{(j_1,\cdots,j_N)\in \Omega} 2s_{i_1\cdots i_N, j_1\cdots i_N}^h \varepsilon_{j_1\cdots j_N})^2 \Big)^{1/2} -  \zeta^2.
\end{align}
Now, let  
\begin{equation}\label{eq:qc}
Q_\zeta:= \frac{\zeta^2 - 2\zeta a_1 \sqrt{R_{\max}} \phi_1 - \lambda  J(\T{Z}^*)}{\zeta}.
\end{equation}
It follows from \eqref{eq:sub22} that 
\begin{eqnarray}\label{eq:bT4}
\nonumber 
&& 
 P^* \Big(\sup_{A(\zeta)} (F(s^h, \T{X}; \T{Z}^*) - F(s^h,\T{X}; {\T{Z}}))\geq   \lambda(J({\T{Z}}) -  J(\T{Z}^*) ) \Big) \\
 \nonumber 
 & \leq &
  P^* \Big(\sup_{A(\zeta)} (F(s^h, \T{X}; \T{Z}^*) - F(s^h,\T{X}; {\T{Z}}))\geq  -  J(\T{Z}^*) ) \Big) \\
 \nonumber 
& \leq &
 P^* \Big( \frac{1}{\prod_{n \in [N]} I_n} \sum_{i_1=1}^{I_1}\cdots\sum_{i_N=1}^{I_N} (\sum_{(j_1, \cdots, j_N) \in \Omega} 2 s_{i_1\cdots i_N, j_1\cdots i_N}^h \varepsilon_{j_1 \cdots j_N})^2)^{1/2} \geq Q_\zeta\Big) \\
 \nonumber 
& \leq & P^* \Big(\max_{i_1,\cdots,i_N} |\sum_{(j_1, \cdots, j_N) \in \Omega} 2s_{i_1\cdots i_N, j_1\cdots i_N}^h \varepsilon_{j_1 \cdots j_N} | \geq Q_\zeta \Big) \\
\nonumber 
&\leq &\sum_{i_1=1}^{I_1}\cdots\sum_{i_N=1}^{I_N} 2\exp\big(-\frac{Q_\zeta^2}{2\sigma^2  \max_{i_1,\cdots, i_N} \sum_{(j_1, \cdots, j_N) \in \Omega} (s_{i_1\cdots i_N, j_1\cdots i_N}^h)^2}  
\big)\\ 
&\leq & 2 \prod_{n \in [N]} I_n \exp\big(-\frac{Q_\zeta^2 }{2 \sigma^2 \phi_2}),\end{eqnarray}  
where the third to the last inequalities follow from the Chernoff inequality of a weighted
sub-Gaussian distribution \citet{chung2006concentration} and Assumption~\ref{assu:subgaus}. 

Let $\zeta= 2^{t-1} \eta $. By our assumptions, we have $\eta \geq 2^4 a_1 \sqrt{R_{\max}} \phi_1$ and $\lambda J(\T{Z}^*) \leq 2^{-2}\eta^2$. Substituting these bounds into \eqref{eq:qc} and using \eqref{eq:bT4}, we obtain
\begin{eqnarray*}
 && \sum_{t=1}^{\infty} P^* \Big(\sup_{A(2^{t-1} \eta)} (F(s^h, \T{X}; \T{Z}^*) - F(s^h, \T{X}; {\T{Z}}))\geq \lambda J({\T{Z}}) -\lambda J(\T{Z}^*) \Big) \\
 &\leq&  4 \prod_{n \in [N]} I_n \sum_{t=1}^{\infty} \exp (\frac{-2^{2t-8} \eta^2}{2 \phi_2\sigma^2})\\
&\leq & 4  \prod_{n \in [N]} I_n  \frac{\exp(- a_2 \frac{\eta^2}{\phi_2 \sigma^2})}{1-\exp(-a_2 \frac{\eta^2}{\sigma^2 \phi_2 })},
\end{eqnarray*} 
where $a_2>0$ is a constant.

Thus, there exists a positive constant $a_3$ such that
\begin{eqnarray}\label{eq:bT5}
\nonumber 
P(\rho(\widehat{\T{Z}},\T{Z}^*) \geq \eta ) &\leq&  a_3 \exp(-a_2 \frac{\eta^2}{\phi_2\sigma^2}+ \log {\prod_{n \in [N]} I_n}).
\end{eqnarray}  
\end{proof}

\subsection*{Proof of Corollary~\ref{rate}}
\begin{proof}
Assume $\Omega$ is sufficiently large. It follows form the law of large numbers that 
\begin{eqnarray}\label{eq:last:cor1}
\sum_{(j_1, \cdots, j_N) \in \Omega}  K_h \left( \sum_{n=1}^{N}  \|\by_{j_n} -\by_{i_n}\| \right)  c_{i_1 \cdots i_N,j_1\cdots j_N}  \geq  \frac{|\Omega|}{2} E  \left( K_h \left( \sum_{n=1}^{N}  \|\by -\by_{i_n}\| \right) c_{i_1 \cdots i_N} \Big| \Delta=1 \right). 
\end{eqnarray}
On the other hand,  by letting $u=  \sum_{n=1}^{N}  \|\by -\by_{i_n}\| $, we obtain
\begin{eqnarray}\label{eq:last:cor2}
 \nonumber
E  \left( K_h \left( u \right) c \Big| \Delta=1 \right)  
  \nonumber
& \geq & P\left(c=1,\Delta=1 \right) E \left( K_h \left(u\right) \Big|c=1, \Delta=1 \right)  \\
 \nonumber
& \geq & a_4 E  \left(K_h \left( u  \right) \Big|c=1, \Delta=1  \right)  \\
 \nonumber
& \geq & a_5 \int_{}^{} K_h(u) f_{U \big |c=1, \Delta=1} udu  \\
 \nonumber
& \geq & a_6  h \int_{}^{} K_h(u) du  
\\
& \geq & a_7 |\Omega| h. 
\end{eqnarray}
Here, $ f_{U \big |c=1, \Delta=1}$ is the conditional density for $U$ defined in \eqref{eq:uu} and the inequalities follow from Assumption~\ref{assu:kernels}. 

Similarly, for some positive constants $a_8$, we have 
\begin{eqnarray}\label{eq:last:cor3}
 \nonumber
 & & \sum_{(j_1, \cdots, j_N) \in \Omega}  K_h \left( \sum_{n=1}^{N}  \|\by_{j_n} -\by_{i_n}\| \right) \left( \sum_{n=1}^{N}  \|\by_{j_n} -\by_{i_n}\| \right)^{\alpha} c_{i_1 \cdots i_N,j_1\cdots j_N}  \\
  \nonumber
 & \geq & 2 |\Omega| E \left( K_h(U_{i_1 \cdots i_N}) U_{i_1 \cdots i_N}^{\alpha} \Big| c_{i_1 \cdots i_N}=1, \Delta=1) \right) \\
  \nonumber
  & \geq & 2 |\Omega| \int_{}^{} K_h(u) u^{\alpha} f_{U_{i_1 \cdots i_N}|c_{i_1 \cdots i_N}=1, \Delta=1} udu \\ 
  & \geq & a_{8} |\Omega| h^{\alpha+1}, 
\end{eqnarray}
where the last inequality uses Assumption~\ref{assu:kernels}.

Further, 
\begin{eqnarray}\label{eq:last:cor4}
\nonumber
 & & \sum_{(j_1, \cdots, j_N) \in \Omega}  K_h^2 (\|y_{i_1, \cdots, i_N} - y_{j_1, \cdots, j_N} \|) c_{i_1 \cdots i_N,j_1\cdots j_N}  \\
 \nonumber
 & \geq & 2 |\Omega| E \left( K_h(U_{i_1 \cdots i_N}) U_{i_1 \cdots i_N} \big| c_{i_1 \cdots i_N}=1, \Delta=1) \right) \\
 & \leq &  a_{9} |\Omega| h. 
\end{eqnarray}
We combine the inequalities \eqref{eq:last:cor1}--\eqref{eq:last:cor4} to get
\begin{eqnarray*}
\max_{i_1\cdots i_N} \sum_{(j_1,\cdots ,j_N ) \in \Omega} (s^h_{i_1\cdots i_N, j_1\cdots j_N})^2  &\leq &   (|\Omega| h)^{-1}, \quad \text{and}\\
\max_{i_1\cdots i_N}   \sum_{(j_1,\cdots ,j_N ) \in \Omega} s^h_{i_1\cdots i_N, j_1j_2\cdots j_N}  (\sum_{n=1}^{N} d (\by_{i_n},\by_{j_n}))^\alpha  &\leq &  h^\alpha.
\end{eqnarray*}
Thus, $\phi_1 =h^\alpha$  and $\phi_2 = (|\Omega| h)^{-1}$, then the desired result immediately follows.
\end{proof}
\vspace{-.5cm}
\bibliography{ref_dcot}
\end{document}